\documentclass[letterpaper]{article}
\usepackage[margin=1in,dvips]{geometry}
\usepackage{graphicx,psfrag,amsmath,amsthm,amssymb}
\usepackage{natbib}
\usepackage{url,color,booktabs}

\newcommand{\A}{\ensuremath{\mathbf{A}}}

\newcommand{\F}{\ensuremath{\mathbf{F}}}

\newcommand{\W}{\ensuremath{\mathbf{W}}}
\newcommand{\X}{\ensuremath{\mathbf{X}}}
\newcommand{\Y}{\ensuremath{\mathbf{Y}}}
\newcommand{\Z}{\ensuremath{\mathbf{Z}}}

\newcommand{\f}{\ensuremath{\mathbf{f}}}

\newcommand{\h}{\ensuremath{\mathbf{h}}}

\newcommand{\w}{\ensuremath{\mathbf{w}}}
\newcommand{\x}{\ensuremath{\mathbf{x}}}
\newcommand{\y}{\ensuremath{\mathbf{y}}}
\newcommand{\z}{\ensuremath{\mathbf{z}}}

\newcommand{\blambda}{\ensuremath{\boldsymbol{\lambda}}}

\newcommand{\bbN}{\ensuremath{\mathbb{N}}}
\newcommand{\bbR}{\ensuremath{\mathbb{R}}}

\newcommand{\calD}{\ensuremath{\mathcal{D}}}

\newcommand{\calI}{\ensuremath{\mathcal{I}}}

\newcommand{\calO}{\ensuremath{\mathcal{O}}}

\newcommand{\norm}[1]{\left\lVert#1\right\rVert}
\newcommand{\ceil}[1]{\lceil#1\rceil}

\newcommand{\caja}[4][1]{{%
    \renewcommand{\arraystretch}{#1}%
    \begin{tabular}[#2]{@{}#3@{}}%
      #4%
    \end{tabular}%
    }}

{%
\begin{list}{#1}{
\vspace{-\topsep}
\vspace{-\partopsep}
\setlength{\itemindent}{0cm}
\setlength{\rightmargin}{0cm}
\setlength{\listparindent}{0cm}
\settowidth{\labelwidth}{#1}
\setlength{\leftmargin}{\labelwidth}
\addtolength{\leftmargin}{\labelsep}
\setlength{\itemsep}{0cm}
}%
}%
{%
\end{list}
\vspace{-\topsep}
\vspace{-\partopsep}
}

{\begin{enumerate}%
}%
{\end{enumerate}}

\DeclareMathOperator*{\argmin}{arg\,min}

\theoremstyle{plain}
\newtheorem{thm}{Theorem}[section]
\newtheorem{lemma}[thm]{Lemma}
\newtheorem*{lemma*}{Lemma}

\newtheorem*{prop*}{Proposition}

\theoremstyle{definition}

\newtheorem*{defn*}{Definition}

\newtheorem*{exmp*}{Example}

\newtheorem*{conj*}{Conjecture}

\theoremstyle{remark}
\newtheorem{rmk}[thm]{Remark}
\newtheorem*{rmk*}{Remark}

\graphicspath{{grf/}}

\bibpunct[, ]{(}{)}{;}{a}{,}{,} 

\title{ParMAC: distributed optimisation of nested functions, \\ with application to learning binary autoencoders}
\author{
  Miguel \'A.\ Carreira-Perpi\~n\'an\hspace{5ex} Mehdi Alizadeh \\
  Electrical Engineering and Computer Science, University of California, Merced \\
  {\url{http://eecs.ucmerced.edu}}
}
\date{May 30, 2016}

\begin{document}

\maketitle

\begin{abstract}
  
  Many powerful machine learning models are based on the composition of multiple processing layers, such as deep nets, which gives rise to nonconvex objective functions. A general, recent approach to optimise such ``nested'' functions is the \emph{method of auxiliary coordinates (MAC)} \citep{CarreirWang14a}. MAC introduces an auxiliary coordinate for each data point in order to decouple the nested model into independent submodels. This decomposes the optimisation into steps that alternate between training single layers and updating the coordinates. It has the advantage that it reuses existing single-layer algorithms, introduces parallelism, and does not need to use chain-rule gradients, so it works with nondifferentiable layers. With large-scale problems, or when distributing the computation is necessary for faster training, the dataset may not fit in a single machine. It is then essential to limit the amount of communication between machines so it does not obliterate the benefit of parallelism. We describe a general way to achieve this, \emph{ParMAC}. ParMAC works on a cluster of processing machines with a circular topology and alternates two steps until convergence: one step trains the submodels in parallel using stochastic updates, and the other trains the coordinates in parallel. Only submodel parameters, no data or coordinates, are ever communicated between machines. ParMAC exhibits high parallelism, low communication overhead, and facilitates data shuffling, load balancing, fault tolerance and streaming data processing. We study the convergence of ParMAC and propose a theoretical model of its runtime and parallel speedup. To illustrate our general results in a specific algorithm, we develop ParMAC to learn binary autoencoders with application to fast, approximate image retrieval. We implement this using Message Passing Interface (MPI) in a distributed system and demonstrate nearly perfect speedups in a 128-processor cluster with a training set of 100 million high-dimensional points. The speedups achieved agree well with the prediction of our theoretical speedup model.

\end{abstract}

\section{Introduction: big data, parallel processing and nested models}
\label{s:intro}

Serial computing has reached a plateau and parallel, distributed architectures are becoming widely available, from machines with a few cores to cloud computing with 1000s of machines. The combination of powerful nested models with large datasets is a key ingredient to solve difficult problems in machine learning, computer vision and other areas, and it underlies recent successes in deep learning \citep{Hinton_12a,Le_12a,Dean_12a}. Unfortunately, parallel computation is not easy, and many good serial algorithms do not parallelise well. The cost of communicating (through the memory hierarchy or a network) greatly exceeds the cost of computing, both in time and energy, and will continue to do so for the foreseeable future \citep{FullerMillet11a,Graham_04a}. Thus, good parallel algorithms must minimise communication and maximise computation per machine, while creating sufficiently many subproblems (ideally independent) to benefit from as many machines as possible. The load (in runtime) on each machine should be approximately equal. Faults become more frequent as the number of machines increases, particularly if they are inexpensive machines. Machines may be heterogeneous and differ in CPU and memory; this is the case with initiatives such as SETI@home, which may become an important source of distributed computation in the future. Big data applications have additional restrictions. The size of the data means it cannot be stored on a single machine, so distributed-memory architectures are necessary. Sending data between machines is prohibitive because of the size of the data and the high communication costs. In some applications, more data is collected than can be stored, so data must be regularly discarded. In others, such as sensor networks, limited battery life and computational power imply that data must be processed locally.

In this paper, we focus on machine learning models of the form $\y = \f_{K+1}(\dots \f_2(\f_1(\x))\dots)$, i.e., consisting of a nested mapping from the input \x\ to the output \y. Such \emph{nested models} involve multiple parameterised layers of processing and include deep neural nets \citep{HintonSalakh06a}, cascades for object recognition in computer vision \citep{Serre_07a,Ranzat_07b} or for phoneme classification in speech processing \citep{GoldMorgan99a,SaonChien12a}, wrapper approaches to classification or regression \citep{KohaviJohn97a}, and various combinations of feature extraction/learning and preprocessing prior to some learning task. Nested and hierarchical models are ubiquitous in machine learning because they provide a way to construct complex models by the composition of simple layers. However, training nested models is difficult even in the serial case because \emph{function composition produces inherently nonconvex functions}, which makes gradient-based optimisation difficult and slow, and sometimes inapplicable (e.g.\ with nonsmooth or discrete layers).

Our starting point is a recently proposed technique to train nested models, the \emph{method of auxiliary coordinates (MAC)} \citep{CarreirWang12a,CarreirWang14a}. This reformulates the optimisation into an iterative procedure that alternates training submodels independently with coordinating them. It introduces significant model and data parallelism, can often train the submodels using existing algorithms, and has convergence guarantees with differentiable functions to a local stationary point, while it also applies with nondifferentiable or even discrete layers. MAC has been applied to various nested models \citep{CarreirWang14a,WangCarreir14a,CarreirRaziper15a,RaziperCarreir16a,CarreirVladym15a}. However, the original papers proposing MAC \citep{CarreirWang12a,CarreirWang14a} did not address how to run MAC on a distributed computing architecture, where communication between machines is far costlier than computation. This paper proposes \emph{ParMAC}, a parallel, distributed framework to learn nested models using MAC, implements it in Message Passing Interface (MPI) for the problem of learning binary autoencoders (BAs), and demonstrates its ability to train on large datasets and achieve large speedups on a distributed cluster. We first review related work (section~\ref{s:related}), describe MAC in general and for BAs (section~\ref{s:MAC}) and introduce the ParMAC model and some extensions of it (section~\ref{s:ParMAC}). Then, we analyse theoretically ParMAC's parallel speedup (section~\ref{s:speedup-th}) and convergence (section~\ref{s:conv}). Finally, we describe our MPI implementation of ParMAC for BAs (section~\ref{s:ParMAC-BAhash-implem}) and show experimental results (section~\ref{s:expts}). Although our MPI implementation and experiments are for a particular ParMAC algorithm (for binary autoencoders), we emphasise that our contributions (the definition of ParMAC and the theoretical analysis of its speedup and convergence) apply to ParMAC in general for any situation where MAC applies, i.e., nested functions with $K$ layers.

\section{Related work}
\label{s:related}

Distributed optimisation and large-scale machine learning have been steadily gaining interest in recent years. Most work has centred on \emph{convex} optimisation, particularly when the objective function has the form of empirical risk minimisation (data fitting term plus regulariser) \citep{Cevher_14a}. This includes many important models in machine learning, such as linear regression, LASSO, logistic regression or SVMs. Such work is typically based on stochastic gradient descent (SGD) \citep{Bottou10a}, coordinate descent (CD) \citep{Wright16a} or the alternating direction method of multipliers (ADMM) \citep{Boyd_11a}. This has resulted in several variations of parallel SGD \citep{Mcdonal_10a,Bertsek11a,Zinkev_10a,Gemull_11a,Niu_11a}, parallel CD \citep{Bradley_11a,RichtarTakac13a,LiuWright15a} and parallel ADMM \citep{Boyd_11a,Ouyang_13a,ZhangKwok14a}. 

It is instructive to consider the parallel SGD case in some detail. Here, one typically runs SGD independently on data subsets (done by $P$ worker machines), and a parameter server regularly gathers the replica parameters from the workers, averages them and broadcasts them back to the workers. One can show \citep{Zinkev_10a} that, for a small enough step size and under some technical conditions, the distance to the minimum in objective function value satisfies an upper bound. The upper bound has a term that decreases as the number of workers $P$ increases, so that parallelisation helps, but it has another term that is independent of $P$, so that past a certain point parallelisation does not help. In practice, the speedups over serial SGD are generally modest. Also, the theoretical guarantees of parallel SGD are restricted to \emph{shallow} models, as opposed to \emph{deep} or \emph{nested} models, because the composition of functions is nearly always nonconvex. Indeed, parallel SGD can diverge with nonconvex models. The intuitive reason for this is that, with local minima, the average of two workers can have a larger objective value than each of the individual workers, and indeed the average of two minima need not be a minimum. In practice, parallel SGD can give reasonable results with nonconvex models if one takes care to average replica models that are close in parameter space and thus associated with the same optimum (e.g.\ eliminating ``stale'' models and other heuristics), but this is not easy \citep{Dean_12a}.

Little work has addressed nonconvex models. Most of it has focused on deep nets \citep{Dean_12a,Le_12a}. For example, Google's DistBelief \citep{Dean_12a} uses asynchronous parallel SGD, with gradients for the full model computed with backpropagation, to achieve data parallelism (with the caveat above), and some form of model parallelism. The latter is achieved by carefully partitioning the neural net into pieces and allocating them to machines to compute gradients. This is difficult to do and requires a careful match of the neural net structure (number of layers and hidden units, connectivity, etc.\@) to the target hardware. Although this has managed to train huge nets on huge datasets by using tens of thousands of CPU cores, the speedups achieved were very modest. Other work has used similar techniques but for GPUs \citep{Chen_12e,Zhang_13a,Coates_13a,Seide_14a}.

Another recent trend is on parallel computation abstractions tailored to machine learning, such as Spark \citep{Zaharia_10a}, GraphLab \citep{Low_12a}, Petuum \citep{Xing_15a} or TensorFlow \citep{Abadi_15a}, with the goal of making cloud computing easily available to train machine learning models. Again, this is often based on shallow models trained with gradient-based convex optimisation techniques, such as parallel SGD. Some of these systems implement some form of deep neural nets.

Finally, there also exist specific approximation techniques for certain types of large-scale machine learning problems, such as spectral problems, using the Nystr{\"o}m formula or other landmark-based methods \citep{WilliamSeeger01a,Bengio_04a,DrineasMahoney05a,Talwal_13a,VladymCarreir13a,VladymCarreir16a}.

ParMAC is specifically designed for nested models, which are typically nonconvex and include deep nets and many other models, some of which have nondifferentiable layers. As we describe below, ParMAC has the advantages of being simple and relatively independent of the target hardware, while achieving high speedups.

\section{Optimising nested models using the method of auxiliary coordinates (MAC)}
\label{s:MAC}

Many machine learning architectures share a fundamental design principle: \emph{mathematically, they construct a (deeply) nested mapping from inputs to outputs}, of the form $\f(\x;\W) = \f_{K+1}(\dots \f_2(\f_1(\x;\W_1);\W_2)\dots;\W_{K+1})$ with parameters \W, such as deep nets or binary autoencoders consisting of multiple processing layers. Such problems are traditionally optimised using methods based on gradients computed using the chain rule. However, such gradients may sometimes be inconvenient to use, or may not exist (e.g.\ if some of the layers are nondifferentiable). Also, they are hard to parallelise, because of the inherent sequentiality in the chain rule.

The \emph{method of auxiliary coordinates (MAC)} \citep{CarreirWang12a,CarreirWang14a} is designed to optimise nested models without using chain-rule gradients while introducing parallelism. It solves an equivalent but in appearance very different problem to the nested one, which affords embarrassing parallelisation. The idea is to break nested functional relationships judiciously by introducing new variables (the \emph{auxiliary coordinates}) as equality constraints. These are then solved by optimising a penalised function using alternating optimisation over the original parameters (which we call the \W\ step) and over the coordinates (which we call the \Z\ step). The result is a \emph{coordination-minimisation (CM) algorithm}: the minimisation (\W) step updates the parameters by splitting the nested model into independent submodels and training them using existing algorithms, and the coordination (\Z) step ensures that corresponding inputs and outputs of submodels eventually match.

MAC algorithms have been developed for several nested models so far: deep nets \citep{CarreirWang14a}, low-dimensional SVMs \citep{WangCarreir14a}, binary autoencoders \citep{CarreirRaziper15a}, affinity-based loss functions for binary hashing \citep{RaziperCarreir16a} and parametric nonlinear embeddings \citep{CarreirVladym15a}. In this paper we focus mostly on the particular case of binary autoencoders. These define a nonconvex nondifferentiable problem, yet its MAC algorithm is simple and effective. It allows us to demonstrate, in an actual implementation in a distributed system, the fundamental properties of ParMAC: how MAC introduces parallelism; how ParMAC keeps the communication between machines low; the use of stochastic optimisation in the \W\ step; and the tradeoff between the different amount of parallelism in the \W\ and \Z\ steps. It also allows us to test how good our theoretical model of the speedup is in experiments. We first give the detailed MAC algorithm for binary autoencoders, and then generalise it to $K>1$ hidden layers.

\subsection{Optimising binary autoencoders using MAC}
\label{s:BAhash}

A \emph{binary autoencoder (BA)} is a usual autoencoder but with a binary code layer. It consists of an \emph{encoder} $\h(\x)$ that maps a real vector $\x\in\bbR^D$ onto a \emph{binary} code vector with $L<D$ bits, $\z\in\{0,1\}^L$, and a linear \emph{decoder} $\f(\z)$ which maps \z\ back to $\bbR^D$ in an effort to reconstruct \x. We will call \h\ a \emph{binary hash function} (see later). Let us write $\h(\x) = s(\A\x)$ (\A\ includes a bias by having an extra dimension $x_0=1$ for each \x) where $\A\in\bbR^{L\times (D+1)}$ and $s(t)$ is a step function applied elementwise, i.e., $s(t) = 1$ if $t\ge 0$ and $s(t) = 0$ otherwise. Given a dataset of $D$-dimensional patterns $\X = (\x_1,\dots,\x_N)$, our objective function, which involves the nested model $\y = \f(\h(\x))$, is the usual least-squares reconstruction error:
\begin{equation}
  \label{e:BA-nested}
  E_{\text{BA}}(\h,\f) = \sum^N_{n=1}{ \norm{\x_n - \f(\h(\x_n))}^2 }.
\end{equation}
Optimising this nonconvex, nonsmooth function is NP-complete. Where the gradients do exist with respect to \A\ they are zero, so optimisation of \h\ using chain-rule gradients does not apply. We introduce as auxiliary coordinates the outputs of \h, i.e., the codes for each of the $N$ input patterns, and obtain the following equality-constrained problem:
\begin{equation}
  \label{e:BA-MAC-constrained}
  \min_{\h,\f,\Z}{ \sum^N_{n=1}{ \norm{\x_n - \f(\z_n)}^2 } } \quad \text{s.t.} \quad \z_n = \h(\x_n),\ \z_n\in\{0,1\}^L,\ n=1,\dots,N.
\end{equation}
Note the codes are binary. We now apply the quadratic-penalty method (it is also possible to apply the augmented Lagrangian method; \citealp{NocedalWright06a}) and minimise the following objective function while progressively increasing $\mu$, so the constraints are eventually satisfied:
\begin{equation}
  \label{e:BA-MAC-QP}
  E_Q(\h,\f,\Z;\mu) = \sum^N_{n=1}{ \left( \norm{\x_n - \f(\z_n)}^2 + \mu \norm{\z_n - \h(\x_n)}^2 \right) } \text{ s.t.\ } \z_n\in\{0,1\}^L,\ n=1,\dots,N.
\end{equation}
Finally, we apply alternating optimisation over \Z\ and $\W = (\h,\f)$. This results in the following two steps:
\begin{itemize}
\item Over \Z\ for fixed $(\h,\f)$, this is a binary optimisation on $NL$ variables, but it separates into $N$ independent optimisations each on only $L$ variables, with the form of a binary proximal operator (where we omit the index $n$): $\min_{\z}{ \norm{\x - \f(\z)}^2 + \mu \norm{\z - \h(\x)}^2 }$ s.t.\ $\z\in\{0,1\}^L$. After some transformations, this problem can be solved exactly for small $L$ by enumeration or approximately for larger $L$ by alternating optimisation over bits, initialised by solving the relaxed problem to $[0,1]$ and truncating its solution (see \citealp{CarreirRaziper15a} for details).
\item Over $\W = (\h,\f)$ for fixed \Z, we obtain $L+D$ independent problems: for each of the $L$ single-bit hash functions (which try to predict \Z\ optimally from \X), each solvable by fitting a linear SVM; and for each of the $D$ linear decoders in \f\ (which try to reconstruct \X\ optimally from \Z), each a linear least-squares problem. With linear \h\ and \f\ this simply involves fitting $L$ SVMs to $(\X,\Z)$ and $D$ linear regressors to $(\Z,\X)$.
\end{itemize}
The user must choose a schedule for the penalty parameter $\mu$ (sequence of values $0 < \mu_1 < \dots < \infty$). This should increase slowly enough that the binary codes can change considerably and explore better solutions before the constraints are satisfied and the algorithm stops. With BAs, MAC stops for a finite value of $\mu$ \citep{CarreirRaziper15a}. This occurs whenever \Z\ does not change compared to the previous \Z\ step, which gives a practical stopping criterion. Also, in order to generalise well to unseen data, we stop iterating for a $\mu$ value not when we (sufficiently) optimise $E_Q(\h,\f,\Z;\mu)$, but when the precision of the hash function in a validation set decreases. This is a form of early stopping that guarantees that we improve (or leave unchanged) the initial \Z, and besides is faster. We also have to initialise \Z. This can be done by running PCA and binarising its result, for example. Fig.~\ref{f:BA-alg} gives the MAC algorithm for BAs.

\begin{figure}[t]
  \centering
  \setlength{\fboxsep}{1ex}
  \framebox{%
    \begin{minipage}[c]{0.70\linewidth}
      \begin{tabbing}
        n \= n \= n \= n \= n \= \kill
        \underline{\textbf{input}} $\X_{D \times N} = (\x_1,\dots,\x_N)$, $L \in \bbN$ \\
        Initialise $\Z_{L \times N} = (\z_1,\dots,\z_N) \in \{0,1\}^{LN}$ \\
        \underline{\textbf{for}} $\mu = 0 < \mu_1 < \dots < \mu_{\infty}$ \+ \\
        \underline{\textbf{for}} $l = 1,\dots,L$ \` {\small\textsf{\W\ step: \h}} \+ \\
        $h_l \leftarrow$ fit SVM to $(\X,\Z_{\cdot l})$ \- \\
        $\f \leftarrow$ least-squares fit to $(\Z,\X)$ \` {\small\textsf{\W\ step: \f}} \\
        \underline{\textbf{for}} $n = 1,\dots,N$ \` {\small\textsf{\Z\ step}} \+ \\
        $\z_n \leftarrow \argmin_{\z_n\in\{0,1\}^L}{\norm{\x_n-\f(\z_n)}^2 + \mu \norm{\z_n-\h(\x_n)}^2}$ \- \\
        \underline{\textbf{if}} no change in \Z\ and $\Z = \h(\X)$ \underline{\textbf{then}} stop \- \\
        \underline{\textbf{return}} \h, $\Z = \h(\X)$
      \end{tabbing}
    \end{minipage}
  }
  \caption{Binary autoencoder MAC algorithm.}
  \label{f:BA-alg}
\end{figure}

The BA was proposed as a way to learn good binary hash functions for fast, approximate information retrieval \citep{CarreirRaziper15a}. Binary hashing \citep{GraumanFergus13a} has emerged in recent years as an effective way to do fast, approximate nearest-neighbour searches in image databases. The real-valued, high-dimensional image vectors are mapped onto a binary space with $L$ bits and the search is performed there using Hamming distances at a vastly faster speed and smaller memory (e.g.\ $N=10^9$ points with $D=500$ take 2 TB, but only 8 GB using $L=64$ bits, which easily fits in RAM). As shown by \citet{CarreirRaziper15a}, training BAs with MAC beats approximate optimisation approaches such as relaxing the codes or the step function in the encoder, and yields state-of-the-art binary hash functions \h\ in unsupervised problems, improving over established approaches such as iterative quantisation (ITQ) \citep{Gong_13a}.

In this paper, we focus on linear hash functions because these are, by far, the most used type of hash functions in the literature of binary hashing, due to the fact that computing the binary codes for a test image must be fast at run time. We also provide an experiment with nonlinear hash functions (RBF network).

\subsection{MAC with $K$ layers}
\label{s:MAC:Klayers}

We now consider the more general case of MAC with $K$ hidden layers \citep{CarreirWang12a,CarreirWang14a}, inputs \x\ and outputs \y\ (for a BA, $\x = \y$). It helps to think of the case of a deep net and we will use it as a running example, but the ideas apply beyond deep nets. Consider a regression problem of mapping inputs \x\ to outputs \y\ (both high-dimensional) with a deep net $\f(\x)$ given a dataset of $N$ pairs $(\x_n,\y_n)$. We minimise the least-squares error (other loss functions are possible):
\begin{equation}
  \label{e:nested}
  E(\W) = \frac{1}{2} \sum^N_{n=1}{\norm{\y_n - \f(\x_n;\W)}^2} \qquad \f(\x;\W) = \f_{K+1}(\dots \f_2(\f_1(\x;\W_1);\W_2)\dots;\W_{K+1})
\end{equation}
where each layer function has the form $\f_k(\x;\W_k) = \sigma(\W_k\x)$, i.e., a linear mapping followed by a squashing nonlinearity ($\sigma(t)$ applies a scalar function, such as the sigmoid $1/(1+e^{-t})$, elementwise to a vector argument, with output in $[0,1]$). We introduce one auxiliary variable per data point and per hidden unit and define the following equality-constrained optimisation problem:
\begin{equation}
  \label{e:mac}
  \frac{1}{2} \sum^N_{n=1}{\norm{\y_n - \f_{K+1}(\z_{K,n};\W_{K+1})}^2} \text{ s.t.\ }
  \renewcommand{\arraystretch}{0.5}
  \left\{
  \begin{array}{@{}l@{}}
    \z_{K,n} = \f_K(\z_{K-1,n};\W_K) \\ \dots \\ \z_{1,n} = \f_1(\x_n;\W_1)
  \end{array}
  \right\} n=1,\dots,N.
\end{equation}
Each $\z_{k,n}$ can be seen as the coordinates of $\x_n$ in an intermediate feature space, or as the hidden unit activations for $\x_n$. Intuitively, by eliminating \Z\ we see this is equivalent to the nested problem~\eqref{e:nested}; we can prove under very general assumptions that both problems have exactly the same minimisers \citep{CarreirWang12a}. Applying the quadratic-penalty method, we optimise the following function:
\begin{equation}
  \label{e:mac-quadpen}
  E_Q(\W,\Z;\mu) = \frac{1}{2} \sum^N_{n=1}{\norm{\y_n - \f_{K+1}(\z_{K,n};\W_{K+1})}^2} + \frac{\mu}{2} \sum^N_{n=1}{\sum^K_{k=1}{\norm{\z_{k,n} - \f_k(\z_{k-1,n};\W_k)}^2}}
\end{equation}
over $(\W,\Z)$ and drive $\mu \rightarrow \infty$. This defines a continuous path $(\W^*(\mu),\Z^*(\mu))$ which, under mild assumptions \citep{CarreirWang12a}, converges to a minimum of the constrained problem~\eqref{e:mac}, and thus to a minimum of the original problem~\eqref{e:nested}. In practice, we follow this path loosely. The quadratic-penalty objective function can be seen as breaking the functional dependences in the nested mapping \f\ and unfolding it over layers. Every squared term involves only a shallow mapping; all variables $(\W,\Z)$ are equally scaled, which improves the conditioning of the problem; and the derivatives required are simpler: we require no backpropagated gradients over \W, and sometimes no gradients at all. We now apply alternating optimisation of the quadratic-penalty objective over \Z\ and \W:
\begin{description}
\item[\W\ step (submodels)] Minimising over \W\ for fixed \Z\ results in a separate minimisation over the weights of each hidden unit---each a single-layer, single-unit \emph{submodel} that can be solved with existing algorithms (logistic regression).
\item[\Z\ step (coordinates)] Minimising over \Z\ for fixed \W\ separates over the coordinates $\z_n$ for each data point $n=1,\dots,N$ and can be solved using the derivatives with respect to \z\ of the single-layer functions $\f_1,\dots,\f_{K+1}$ (omitting the subindex $n$): $\smash{\min_{\z}{ \norm{\y - \f_{K+1}(\z_K)}^2 + \mu \sum^K_{k=1}{\norm{\z_k - \f_k(\z_{k-1})}^2} }}$.
\end{description}
Thus, the \W\ step results in many independent, single-layer single-unit submodels that can be trained with existing algorithms, without extra programming cost. The \Z\ step is new and has the form of a ``generalised'' proximal operator \citep{Rockaf76b,CombetPesquet11a}. MAC reduces a complex, highly-coupled problem---training a deep net---to a sequence of simple, uncoupled problems (the \W\ step) which are coordinated through the auxiliary variables (the \Z\ step). For a large net with a large dataset, this affords an enormous potential for parallel computation.

\section{ParMAC: a parallel, distributed computation model for MAC}
\label{s:ParMAC}

We now turn to the contribution of this paper, the distributed implementation of MAC algorithms. As we have seen, a specific MAC algorithm depends on the model and objective function and on how the auxiliary coordinates are introduced. We can achieve steps that are closed-form, convex, nonconvex, binary, or others. However, the following always hold: (1) In the \Z\ step, \emph{the $N$ subproblems for $\z_1,\dots,\z_N$ are independent, one per data point}. Each $\z_n$ step depends on all or part of the current model. (2) In the \W\ step, there are $M$ \emph{independent submodels}, where $M$ depends on the problem. For example, $M$ is the number of hidden units in a deep net, or the number of hash functions and linear decoders in a BA. Each submodel depends on all the data and coordinates (usually, a given submodel depends, for each $n$, on only a portion of the vector $(\x_n,\y_n,\z_n)$). We now show how to turn this into a distributed, low-communication \emph{ParMAC} algorithm, give an MPI implementation of ParMAC for BAs, and discuss the convergence of ParMAC. Throughout the paper, unless otherwise indicated, we will use the term ``machine'' to mean a single-CPU processing unit with its own local memory and disk, which can communicate with other machines in a cluster through a network or shared memory.

\subsection{Description of ParMAC}

The basic idea in ParMAC is as follows. With large datasets in distributed systems, it is imperative to minimise data movement over the network because the communication time generally far exceeds the computation time in modern architectures. In MAC we have 3 types of data: the original training data $(\X,\Y)$, the auxiliary coordinates \Z, and the model parameters (the submodels). Usually, the latter type is far smaller. \emph{In ParMAC, we never communicate training or coordinate data; each machine keeps a disjoint portion of $(\X,\Y,\Z)$ corresponding to a subset of the points. Only model parameters are communicated, during the \W\ step, following a circular topology, which implicitly implements a stochastic optimisation}. The model parameters are the hash functions \h\ and the decoder \f\ for BAs, and the weight vector $\w_h$ of each hidden unit $h$ for deep nets. Let us see this in detail (refer to fig.~\ref{f:ParMAC}).

\begin{figure}[t]
  \scriptsize
  \psfrag{data}[][][1][90]{\normalsize Data}
  \psfrag{model}[][][1][90]{\normalsize Model}
  \psfrag{P1}[t][t]{\normalsize Machine 1}
  \psfrag{P2}[t][t]{\normalsize Machine 2}
  \psfrag{P3}[t][t]{\normalsize Machine 3}
  \psfrag{P4}[t][t]{\normalsize Machine 4}
  \psfrag{wm}[][]{\normalsize $\w_h$}
  \psfrag{w01}[r][Br]{$1$}
  \psfrag{w02}[r][Br]{$2$}
  \psfrag{w03}[r][Br]{$3$}
  \psfrag{w04}[r][Br]{$4$}
  \psfrag{w05}[r][Br]{$5$}
  \psfrag{w06}[r][Br]{$6$}
  \psfrag{w07}[r][Br]{$7$}
  \psfrag{w08}[r][Br]{$8$}
  \psfrag{w09}[r][Br]{$9$}
  \psfrag{w10}[r][Br]{$10$}
  \psfrag{w11}[r][Br]{$11$}
  \psfrag{w12}[r][Br]{$12$}
  \psfrag{w13}[r][Br]{$13$}
  \psfrag{w14}[r][Br]{$14$}
  \psfrag{w15}[r][Br]{$15$}
  \psfrag{w16}[r][Br]{$16$}
  \psfrag{w17}[r][Br]{$17$}
  \psfrag{w18}[r][Br]{$18$}
  \psfrag{w19}[r][Br]{$19$}
  \psfrag{w20}[r][Br]{$20$}
  \psfrag{w21}[r][Br]{$21$}
  \psfrag{w22}[r][Br]{$22$}
  \psfrag{w23}[r][Br]{$23$}
  \psfrag{w24}[r][Br]{$24$}
  \psfrag{w25}[r][Br]{$25$}
  \psfrag{w26}[r][Br]{$26$}
  \psfrag{w27}[r][Br]{$27$}
  \psfrag{w28}[r][Br]{$28$}
  \psfrag{w29}[r][Br]{$29$}
  \psfrag{w30}[r][Br]{$30$}
  \psfrag{w31}[r][Br]{$31$}
  \psfrag{w32}[r][Br]{$32$}
  \psfrag{w33}[r][Br]{$33$}
  \psfrag{w34}[r][Br]{$34$}
  \psfrag{w35}[r][Br]{$35$}
  \psfrag{w36}[r][Br]{$36$}
  \psfrag{w37}[r][Br]{$37$}
  \psfrag{w38}[r][Br]{$38$}
  \psfrag{w39}[r][Br]{$39$}
  \psfrag{w40}[r][Br]{$40$}
  \psfrag{w41}[r][Br]{$41$}
  \psfrag{w42}[r][Br]{$42$}
  \psfrag{w43}[r][Br]{$43$}
  \psfrag{w44}[r][Br]{$44$}
  \psfrag{w45}[r][Br]{$45$}
  \psfrag{w46}[r][Br]{$46$}
  \psfrag{w47}[r][Br]{$47$}
  \psfrag{w48}[r][Br]{$48$}
  \psfrag{xn}[B][B]{\normalsize $\x_n$}
  \psfrag{yn}[B][B]{\normalsize $\y_n$}
  \psfrag{zn}[B][B]{\normalsize $\z_n$}
  \psfrag{x01}[r][Br]{$1$}
  \psfrag{x02}[r][Br]{$2$}
  \psfrag{x03}[r][Br]{$3$}
  \psfrag{x04}[r][Br]{$4$}
  \psfrag{x10}[r][Br]{$10$}
  \psfrag{x11}[r][Br]{$11$}
  \psfrag{x12}[r][Br]{$12$}
  \psfrag{x13}[r][Br]{$13$}
  \psfrag{x14}[r][Br]{$14$}
  \psfrag{x20}[r][Br]{$20$}
  \psfrag{x21}[r][Br]{$21$}
  \psfrag{x22}[r][Br]{$22$}
  \psfrag{x23}[r][Br]{$23$}
  \psfrag{x24}[r][Br]{$24$}
  \psfrag{x30}[r][Br]{$30$}
  \psfrag{x31}[r][Br]{$31$}
  \psfrag{x32}[r][Br]{$32$}
  \psfrag{x33}[r][Br]{$33$}
  \psfrag{x34}[r][Br]{$34$}
  \psfrag{x40}[r][Br]{$40$}
  \includegraphics[width=\linewidth]{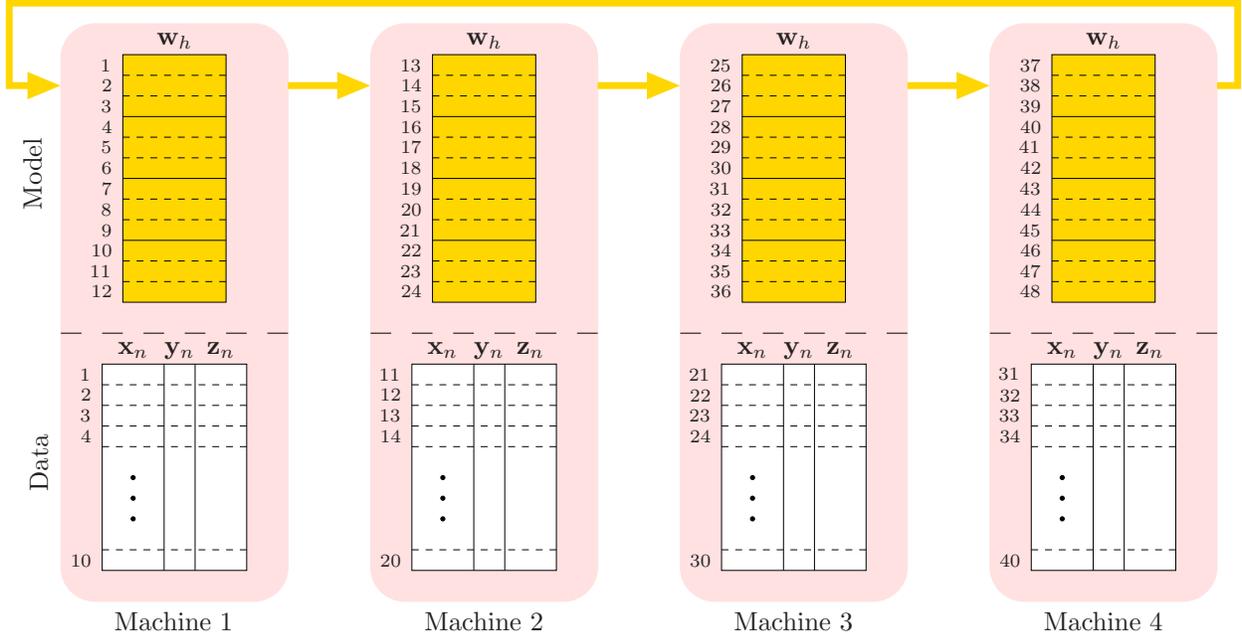}
  \caption{ParMAC model with $P=4$ machines, $M=12$ submodels and $N=40$ data points. ``$\w_h$'' represents the submodels (hash functions and decoders for BAs, hidden unit weight vectors for deep nets). Submodels $h$, $h+M$, $h+2M$ and $h+3M$ are copies of submodel $h$, but only one of them is the most currently updated. At the end of the \W\ step all copies are identical.}
  \label{f:ParMAC}
\end{figure}

Assume for simplicity we have $P$ identical processing machines, each with their own memory and CPU, which are connected through a network. The machines are connected in a circular (ring) unidirectional topology, i.e., machine $1$ $\rightarrow$ machine $2$ $\rightarrow \cdots \rightarrow$ machine $P$ $\rightarrow$ machine $1$, where ``machine $p$ $\rightarrow$ machine $q$'' means machine $p$ can send data directly to machine $q$ (and we say machine $q$ is the successor of machine $p$). Call $\calD = \{(\x_n,\y_n,\z_n)\mathpunct{:}\ n \in \{1,\dots,N\}\}$ the entire dataset and corresponding coordinates. Each machine $p$ will store a subset $\calD_p = \{(\x_n,\y_n,\z_n)\mathpunct{:}\ n \in \calI_p\}$ such that the subsets are disjoint and their union is the entire data, i.e., the index sets satisfy $\calI_p \cap \calI_q = \varnothing$ if $p \neq q$ and $\cup^P_{p=1}{\calI_p} = \{1,\dots,N\}$.

The \Z\ step is very simple. Before the \Z\ step starts%
\footnote{Also, the machines need not start all at the same time in the \Z\ step. A machine can start the \Z\ step on its data as soon as it has received all the updated submodels in the \W\ step. Likewise, as soon as a machine finishes its \Z\ step, it can start the \W\ step immediately, without waiting for all other machines to finish their \Z\ step. However, in our implementation we consider the \W\ and \Z\ steps as barriers, so that all machines start the \W\ or \Z\ step at the same time.},
each machine will contain all the (just updated) submodels. This means that in the \Z\ step each machine $p$ processes its auxiliary coordinates $\{\z_n\mathpunct{:}\ n\in\calI_p\}$ independently of all other machines, i.e., no communication occurs.

The \W\ step is more subtle. At the beginning of the \W\ step, each machine will contain all the submodels and its portion of the data and (just updated) coordinates. Each submodel must have access to the entire data and coordinates in order to update itself and, since the data cannot leave its home machine, the submodel must go to the data (this contrasts with the intuitive notion of the model sitting in a computer while data arrive and are processed). We achieve this in the circular topology as follows. We assume synchronous processing for simplicity, but in practice one would implement this asynchronously. Assume arithmetic modulo $P$ and an imaginary clock whose period equals the time that any one machine takes to process its portion $M/P$ of submodels. At each clock tick, the $P$ machines update each a different portion $M/P$ of the submodels. For example, in fig.~\ref{f:ParMAC}, at clock tick 1 machine $1$ updates submodels $1$--$3$ using its data $\calD_{\calI_1}$ (where $\calI_1 = \{1,\dots,10\}$); machine $2$ updates submodels $4$--$6$; machine $3$ updates submodels $7$--$9$; and machine $4$ updates submodels $10$--$12$. This happens in parallel. Then each machine sends the submodels updated to its successor, also in parallel. In the next tick, each machine updates the submodels it just received, i.e., machine $1$ updates $10$--$12$, machine $2$ updates submodels $1$--$3$, machine $3$ updates submodels $4$--$6$; and machine $4$ updates submodels $7$--$9$ (and each machine always uses its data portion, which never changes). This is repeated until each submodel has visited each machine and thus has been updated with the entire dataset \calD. This happens after $P$ ticks, and we call this an \emph{epoch}. This process may be repeated for $e$ epochs in $eP$ ticks. At this time, each machine contains $M/P$ submodels that are finished (i.e., updated $e$ times over the entire dataset), and the remaining $M(1-1/P)$ submodels it contains are not finished, indeed the finished versions of those submodels reside in other machines. Finally, before starting with the \Z\ step, each machine must contain all the (just updated) submodels (i.e., the parameters for the entire nested model). We achieve this%
\footnote{In MPI, this can be directly achieved with \texttt{MPI\_Alltoall} broadcasting, which scatters/gathers data from all members to all members of a group (a complete exchange). However, in this paper we implement it using the circular topology mechanism described.}
by running a final round of communication without computation, i.e., each machine sends its just updated submodels to its successor. Thus, after one clock tick, machine $p$ sends $M/P$ final submodels to machine $p+1$ and receives $M/P$ submodels from machine $p-1$. After $P-1$ clock ticks, each machine has received the remaining $M(1-1/P)$ submodels that were finished by other machines, hence each machine contains a (redundant) copy of all the current submodels. Fig.~\ref{f:ParMAC-anim} illustrates the sequence of operations during one epoch for the example of fig.~\ref{f:ParMAC}.

\begin{figure}[p]
  \psfrag{P1}{}
  \psfrag{P2}{}
  \psfrag{P3}{}
  \psfrag{P4}{}
  \psfrag{wm}{}
  \psfrag{w01}{}
  \psfrag{w02}{}
  \psfrag{w03}{}
  \psfrag{w04}{}
  \psfrag{w05}{}
  \psfrag{w06}{}
  \psfrag{w07}{}
  \psfrag{w08}{}
  \psfrag{w09}{}
  \psfrag{w10}{}
  \psfrag{w11}{}
  \psfrag{w12}{}
  \psfrag{w13}{}
  \psfrag{w14}{}
  \psfrag{w15}{}
  \psfrag{w16}{}
  \psfrag{w17}{}
  \psfrag{w18}{}
  \psfrag{w19}{}
  \psfrag{w20}{}
  \psfrag{w21}{}
  \psfrag{w22}{}
  \psfrag{w23}{}
  \psfrag{w24}{}
  \psfrag{w25}{}
  \psfrag{w26}{}
  \psfrag{w27}{}
  \psfrag{w28}{}
  \psfrag{w29}{}
  \psfrag{w30}{}
  \psfrag{w31}{}
  \psfrag{w32}{}
  \psfrag{w33}{}
  \psfrag{w34}{}
  \psfrag{w35}{}
  \psfrag{w36}{}
  \psfrag{w37}{}
  \psfrag{w38}{}
  \psfrag{w39}{}
  \psfrag{w40}{}
  \psfrag{w41}{}
  \psfrag{w42}{}
  \psfrag{w43}{}
  \psfrag{w44}{}
  \psfrag{w45}{}
  \psfrag{w46}{}
  \psfrag{w47}{}
  \psfrag{w48}{}
  \begin{tabular}{@{}c@{}c@{}c@{}c@{}c@{}}
    \psfrag{model}[][]{\caja{c}{c}{tick \\  1}}
    \includegraphics[width=\linewidth,height=0.25\linewidth]{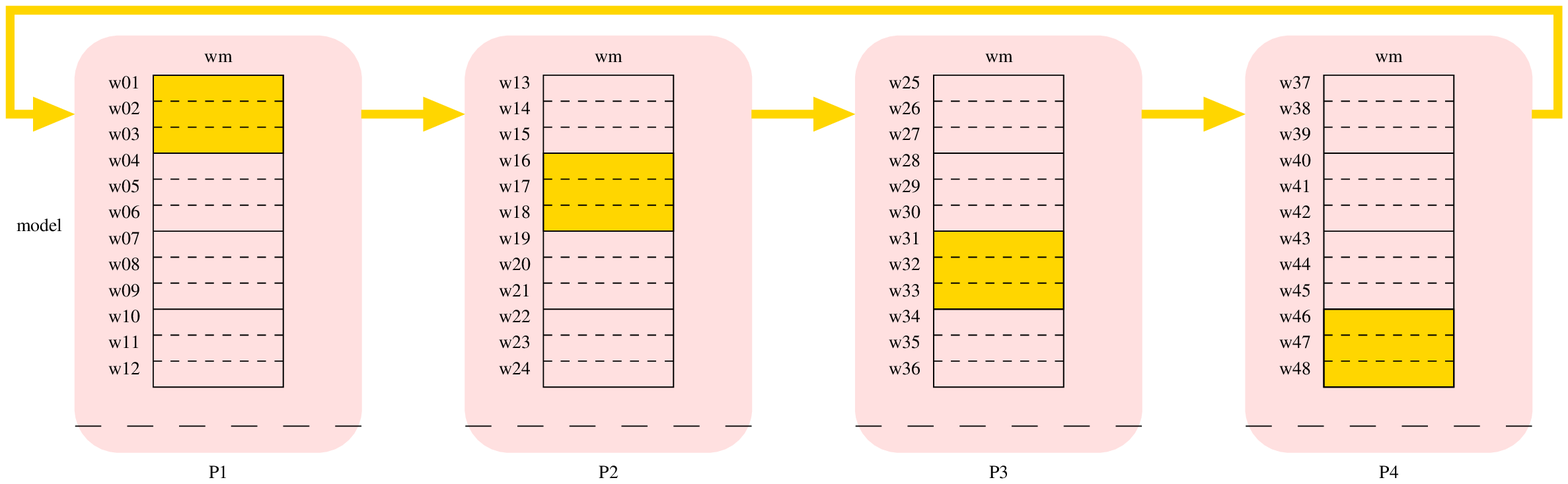} \\
    \psfrag{model}[][]{\caja{c}{c}{tick \\  2}}
    \includegraphics[width=\linewidth,height=0.25\linewidth]{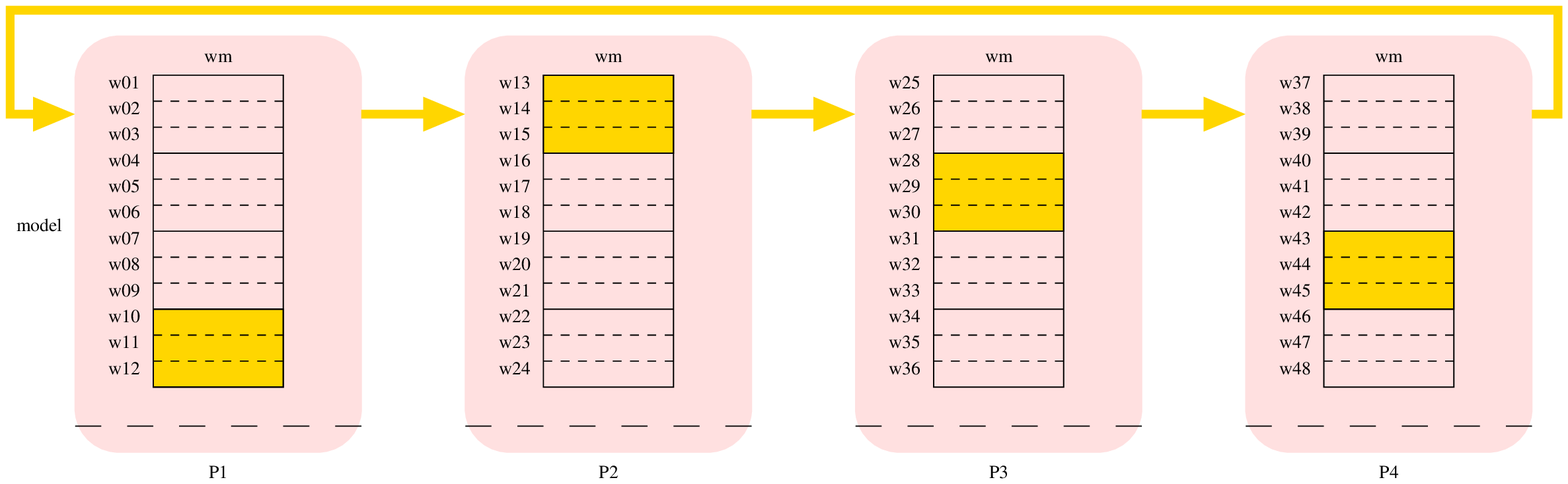} \\
    \psfrag{model}[][]{\caja{c}{c}{tick \\  3}}
    \includegraphics[width=\linewidth,height=0.25\linewidth]{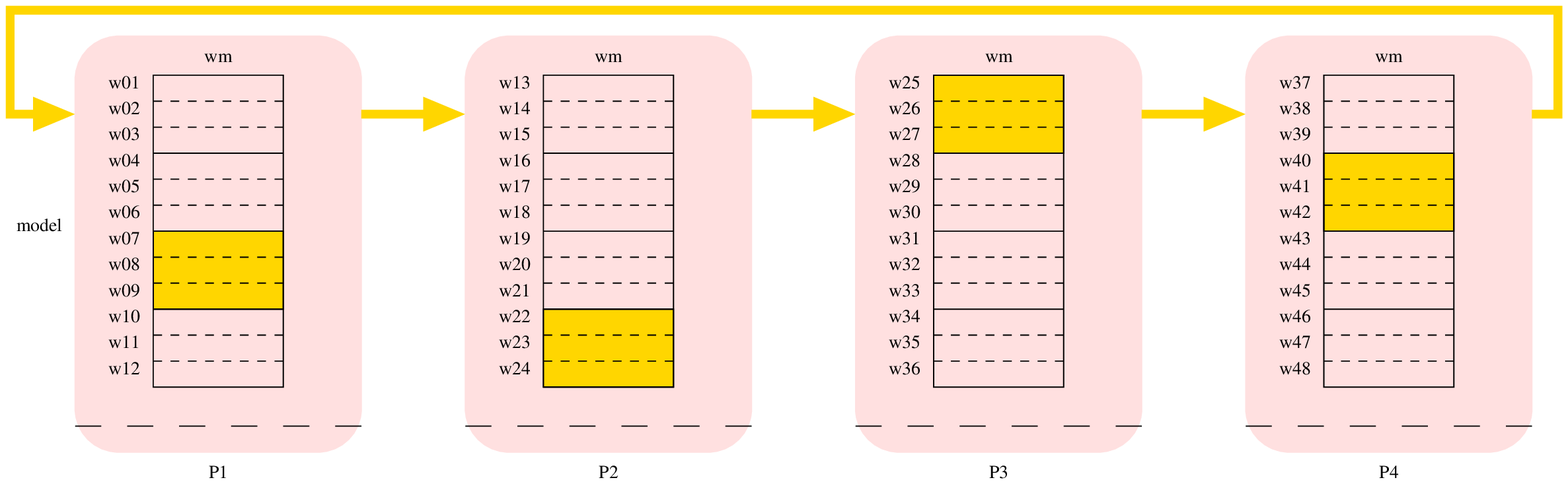} \\
    \psfrag{model}[][]{\caja{c}{c}{tick \\  4}}
    \includegraphics[width=\linewidth,height=0.25\linewidth]{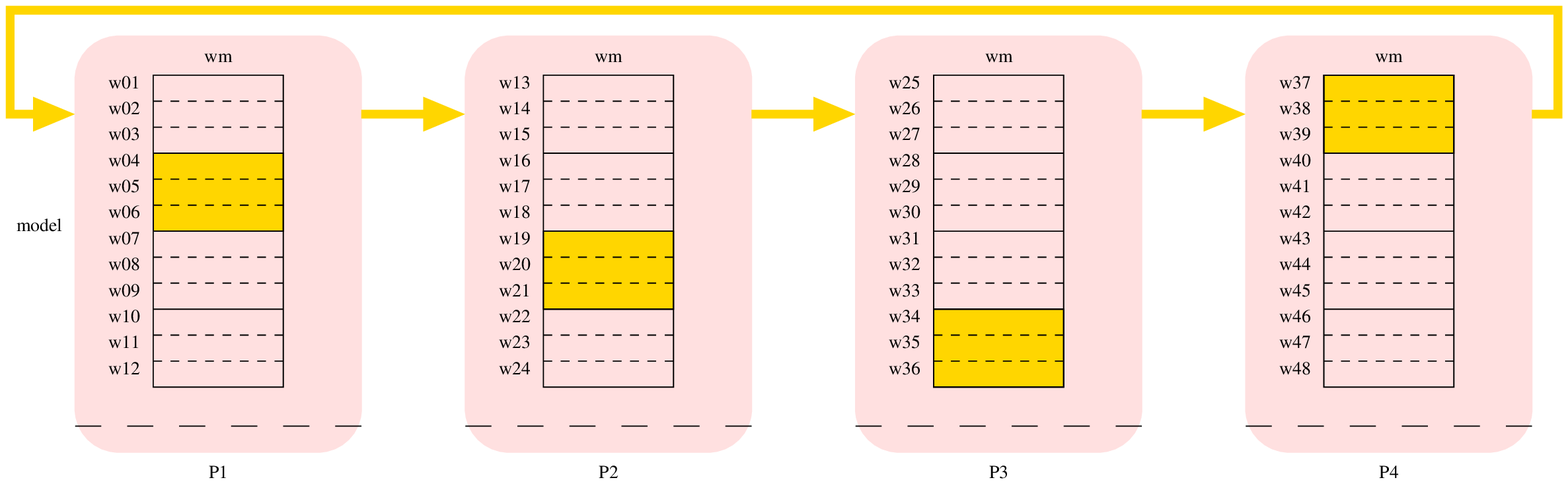} \\
    \psfrag{P1}[t][t]{\normalsize Machine 1}
    \psfrag{P2}[t][t]{\normalsize Machine 2}
    \psfrag{P3}[t][t]{\normalsize Machine 3}
    \psfrag{P4}[t][t]{\normalsize Machine 4}
    \psfrag{model}[][]{\caja{c}{c}{tick \\  5}}
    \includegraphics[width=\linewidth,height=0.25\linewidth]{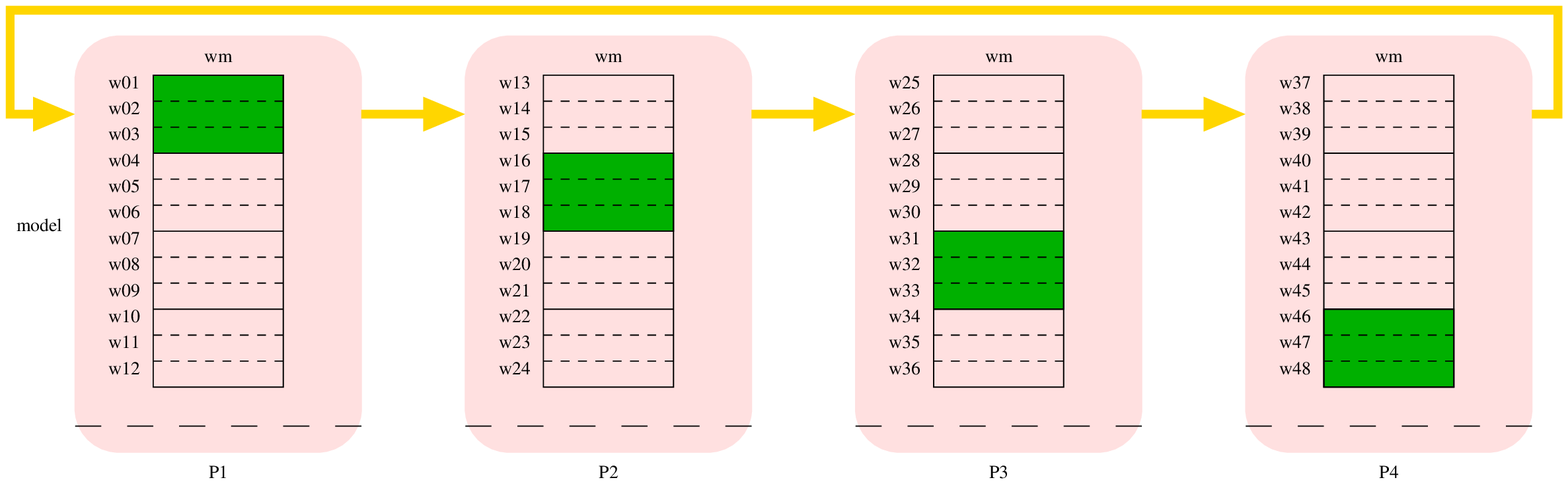}
  \end{tabular}
  \caption{Illustration of one epoch of the synchronous version of ParMAC's \W\ step for the example of fig.~\ref{f:ParMAC} with $P=4$ machines and $M=12$ submodels (we only show the ``model'' part of the figure). Each row corresponds to one clock tick, within which each machine computes on its portion of $M/P$ submodels (coloured gold) and sends them to its successor. The last tick ($= P+1$) is the start of the next epoch, at which point all submodels (coloured green) have been updated over the entire dataset.}
  \label{f:ParMAC-anim}
\end{figure}

In practice, we use an asynchronous implementation. Each machine keeps a queue of submodels to be processed, and repeatedly performs the following operations: extract a submodel from the queue, process it (except in epoch $e+1$) and send it to the machine's successor (which will insert it in its queue). If the queue is empty, the machine waits until it is nonempty. The queue of each machine is initialised with the portion of submodels associated with that machine. Each submodel carries a counter that is initially $1$ and increases every time it visits a machine. When it reaches $Pe$ then the submodel is in the last machine and the last epoch. When it reaches $P(e+1)-1$, it has undergone $e$ epochs of processing and all machines have a copy of it, so it has finished the \W\ step.

Since each submodel is updated as soon as it visits a machine, rather than computing the exact gradient once it has visited all machines and then take a step, the \W\ step is really carrying out \emph{stochastic steps for each submodel}. For example, if the update is done by a gradient step, we are actually implementing stochastic gradient descent (SGD) where the minibatches are of size $N/P$ (or smaller, if we subdivide a machine's data portion into minibatches, which should be typically the case in practice). From this point of view, we can regard the \W\ step as doing SGD on each submodel in parallel by having each submodel visit the minibatches in each machine.

In summary, using $P$ machines, ParMAC iterates as follows:
\begin{description}
\item[\W\ step] The submodels (hash functions and decoders for BAs) visit each machine. This implies we train them with stochastic gradient descent, where one ``epoch'' for a submodel corresponds to that submodel having visited all $P$ machines. All submodels are communicated in parallel, asynchronously with respect to each other, in a circular topology. With $e$ epochs, the entire model parameters are communicated $e+1$ times. The last round of communication is needed to ensure each machine has the most updated version of the model for the \Z\ step.
\item[\Z\ step] Identical to MAC, each data point's coordinates $\z_n$ are optimised independently, in parallel over machines (since each machine contains $\x_n$, $\y_n$, $\z_n$, and all the model parameters). No communication occurs at all.
\end{description}

\subsection{A \W\ step with only two rounds of communication}

As described, and as implemented in our experiments, running $e$ epochs in the \W\ step requires $e$ rounds of communication (plus a final round). However, we can run $e$ epochs with only 1 round of communication \emph{by having a submodel do $e$ consecutive passes within each machine's data}. In the example of fig.~\ref{f:ParMAC}, running $e=2$ epochs for submodel $\w_1$ means the following: instead of visiting the data as 1,\dots,10, 11,\dots,20, 21,\dots,30, 31,\dots,40, 1,\dots,10, 11,\dots,20, 21,\dots,30, 31,\dots,40, it visits the data as 1,\dots,10, 1,\dots,10, 11,\dots,20, 11,\dots,20, 21,\dots,30, 21,\dots,30, 31,\dots,40, 31,\dots,40. (We can also have intermediate schemes such as doing between 1 and $e$ within-machine passes.) This reduces the amount of shuffling, but should not be a problem if the data are randomly distributed over machines (and we can still do within-machine shuffling). This effectively reduces the total communication in the \W\ step to 2 rounds regardless of the number of epochs $e$ (with the second round needed to ensure each machine has the most updated submodels).

\subsection{Extensions of ParMAC}
\label{s:ParMAC:extensions}

In addition, the ParMAC model offers good potential for data shuffling, load balancing, streaming and fault tolerance, which make it attractive for big data. We describe these next.

\paragraph{Data shuffling}
\label{s:shuffling}

It is well known that shuffling (randomly reordering) the dataset prior to each epoch improves the SGD convergence speed. With distributed systems, this can sometimes be a problem and require data movement across machines. Shuffling is easy in ParMAC. Within a machine, we can simply access the local data (minibatches) in random order at each epoch. Across machines, we can simply reorganise the circular topology randomly (while still circular) at the beginning of each new epoch (by generating a random permutation and resetting the successor's address of each machine). We could even have each submodel follow a different, random circular topology. However, we do not implement this because it is unlikely to help (since the submodels are independent) and can unbalance the load over machines.

\paragraph{Load balancing}
\label{s:load-balancing}

This is simple because the work in both the \W\ and \Z\ steps is proportional to the number of data points $N$. Indeed, in the \W\ step each submodel must visit every data point once per epoch. So, even if the submodels differ in size, the training of any submodel is proportional to $N$. In the \Z\ step, each data point is a separate problem dependent on the current model (which is the same for all points), thus all $N$ problems are formally identical in complexity. Hence, in the assumption that the machines are identical and that each data point incurs the same runtime, load balancing is trivial: the $N$ points are allocated in equal portions of $N/P$ to each machine. If the processing power of machine $p$ is proportional to $\alpha_p > 0$ (where $\alpha_p$ could represent the clock frequency of machine $p$, say), then we allocate to machine $p$ a subset of the $N$ points proportional to $\alpha_p$, i.e., machine $p$ gets $N \alpha_p / (\alpha_1 + \dots + \alpha_P)$ data points. This is done once and for all at loading time.

In practice, we can expect some degradation of the parallel speedup even with identical machines and submodels of the same type. This is because machines do vary for various reasons, e.g.\ the runtime can be affected by differences in ventilation across machines located in different areas of a data centre, or because machines are running other user processes in addition to the ParMAC optimisation. Another type of degradation can happen if the submodels differ significantly in runtime (e.g.\ because there are different types of submodels): the runtime of the \W\ step will be driven by the slow submodels, which become a bottleneck. As discussed in section~\ref{s:speedup-th:practical}, we can group the $M$ submodels into a smaller number $M' < M$ of approximately equal-size aggregate submodels, for the purpose of estimating the speedup in theory. This need not be the fastest way to schedule the jobs, and in practice we still process the individual submodels asynchronously.

\paragraph{Streaming}
\label{s:streaming}

Streaming refers to the ability to discard old data and to add new data from training over time. This is useful in online learning, or to allow the data to be refreshed, but also may be necessary when a machine collects more data than it can store. The circular topology allows us to add or remove machines on the fly easily, and this can be used to implement streaming.

We consider two forms of streaming: (1) new data are added within a machine (e.g.\ as this machine collects new data), and likewise old data are discarded within a machine. And (2) new data are added by adding a new machine to the topology, and old data are discarded by removing an existing machine from the topology. Both forms are easily achieved in ParMAC. The first form, within-machine, is trivial: a machine can always add or remove data without any change to the system, because the data for each note is private and never interacts with other machines other than by updating submodels. Adding or discarding data is done at the beginning of the \Z\ step. Discarding data simply means removing the corresponding $\{(\x_n,\y_n,\z_n)\}$ from that machine. Adding data means inserting $\{(\x_n,\y_n)\}$ in that machine and, if necessary, creating within that machine coordinate values $\{\z_n\}$ (e.g.\ by applying the nested model to $\x_n$). We never upload or send any \z\ values over the network.

The second form, creating a new machine or removing an existing one, is barely more complicated, assuming some support from the parallel processing library. We describe it conceptually. Imagine we currently have $P$ machines. We can add a new machine, with its own preloaded data $\{(\x_n,\y_n)\}$, as follows. Adding it to the circular topology simply requires connecting it between any two machines (done by setting the address of their successor): before we have ``machine $p$ $\rightarrow$ machine $p+1$'', afterwards we have ``machine $p$ $\rightarrow$ new machine $\rightarrow$ machine $p+1$''. We add it in the \W\ step, making sure it receives a copy of the final model that has just been finished. The easiest way to do this is by inserting it in the topology at the end of the \W\ step, when each machine is simply sending along a copy of the final submodels. In the \Z\ step, we proceed as usual, but with $P+1$ machines. Removing a machine is easier. To remove machine $p$, we do so in the \Z\ step, by reconnecting ``machine $p-1$ $\rightarrow$ machine $p+1$'' and returning machine $p$ to the cluster. That is all. In the subsequent \W\ step, all machines contain the full model, and the submodels will visit the data in each machine, thus not visiting the data in the removed machine.

\paragraph{Fault tolerance}
\label{s:fault}

This situation is similar to discarding a machine in streaming, except that the fault can occur at any time and is not intended. We can handle it with a little extra bookkeeping, and again assuming some support from the parallel processing library. Imagine a fault occurs at machine $p$ and we need to remove it. If it happens during the \Z\ step, all we need to do is discard the faulty machine and reconnect the circular topology. If it happens during the \W\ step, we also discard and reconnect, but in addition we need to rescue the submodels that were being updated in $p$, which we lose. To do this, we revert to the previously updated copy of them, which resides in the predecessor of $p$ in the circular topology (if no predecessor, we are at the beginning of the \W\ step and we can use any copy in any machine). As for the remaining submodels being updated in other machines, some will have already been updated in $p$ (which require no action) and some will not have been updated in $p$ yet (which should not visit $p$ anymore). We can keep track of this information by tagging each submodel with a list of the machines it has not yet visited. At the beginning of the \W\ step the list of each submodel contains $\{1,\dots,P\}$, i.e., all machines. When this list is empty, for a submodel, then that submodel is finished and needs no further updates. 

Essentially, the robustness of ParMAC to faults comes from its in-built redundance. In the \Z\ (and \W) step, we can do without the data points in one machine because a good model can still be learned from the remaining data points in the other machines. In the \W\ step, we can revert to older copies of the lost submodels residing in other machines.

The asynchronous implementation of ParMAC we described earlier relied on tagging each submodel with a counter in order to know whether it needs processing and communicating. A more general mechanism to run ParMAC asynchronously is to tag each submodel with a list (per epoch) of machines it has to visit. All a machine $p$ needs to do upon receiving a submodel is check its list: if $p$ is not in the list, then the submodel has already visited machine $p$ and been updated with its data, so machine $p$ simply sends it along to its successor without updating it again. If $p$ is in the list, then machine $p$ updates the submodel, removes $p$ from its list, and sends it along to its successor. This works even if we use a different communication topology for each submodel at each epoch.

\section{A theoretical model of the parallel speedup for ParMAC}
\label{s:speedup-th}

In this section we give a theoretical model to estimate the computation and communication times and the parallel speedup in ParMAC. Specifically, eq.~\eqref{e:speedup} gives the speedup $S(P)$ as a function of the number of machines $P$ and other parameters, which seems to agree well with our experiments (section~\ref{s:expts:speedup}). In practice, this model can be used to estimate the optimal number of machines $P$ to use, or to explore the effect on the speedup of different parameter settings (e.g.\ the number of submodels $M$). Throughout the rest of the paper, we will call ``speedup'' $S(P)$ the ratio of the runtime using a single machine (i.e., the serial code) vs using $P > 1$ machines (the parallel code), and ``perfect speedup'' when $S(P) = P$. Our theoretical model applies to the general ParMAC case of $K$ layers, whether differentiable or not; it only assumes that the resulting submodels after introducing auxiliary coordinates are of the same ``size,'' i.e., have the same computation and communication time (this assumption can be relaxed, as we discuss at the end of the section).

We can obtain a quick, rough understanding of the speedup appealing to (a generalisation of) Amdahl's law \citep{GoedecHoisie01a}. ParMAC iterates the \W\ and \Z\ steps as follows (where $M$ is the number of submodels and $N$ the number of data points):
\begin{center}
  \psfrag{W}[][B]{\caja{c}{c}{\W\ step: \\ $M$ problems}}
  \psfrag{Z}[][B]{\caja{c}{c}{\Z\ step: \\ $N \gg M$ problems}}
  \includegraphics[width=0.7\linewidth]{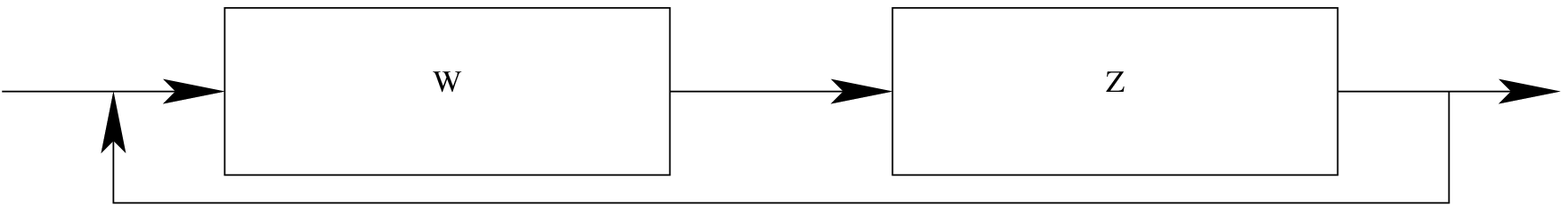}
\end{center}
Roughly speaking, the \W\ step has $M$ independent problems so its speedup would be $\min(M,P)$, while the \Z\ step has $N$ independent problems so its speedup would be $\min(N,P) = P$ (because in practice $N \ge P$). So the overall speedup would be between $M$ and $P$ depending on the relative runtimes of the \W\ and \Z\ steps. This suggests we would expect a nearly perfect speedup $S \approx P$ with $P \le M$ and diminishing returns for $P > M$. This simplified picture ignores important factors such as the ratio of computation vs communication (which our model will make more precise), but it does capture the basic, qualitative behaviour of the speedup.

\subsection{The theoretical model of the speedup}

Let us now develop a more precise, quantitative model. Consider a ParMAC algorithm, operating synchronously, such that there are $M$ independent submodels of the same size in the \W\ step, on a dataset with $N$ training points, distributed over $P$ identical machines (each with $N/P$ points). The ParMAC algorithm runs a certain number of iterations, each consisting of a \W\ and a \Z\ step, so if we ignore small overheads (setup and termination), we can estimate the total runtime as proportional to the number of iterations. Hence, we consider a theoretical model of the runtime of one iteration of the ParMAC algorithm, given the following parameters:
\begin{itemize}
\item $P$: number of machines.
\item $N$: number of training points. \\
  We assume $N > P$ is divisible by $P$. This is not a problem because $N \gg P$ in practice (otherwise, there would be no reason to distribute the optimisation).
\item $M$: number of submodels in the \W\ step. \\
  This may be smaller than, equal to or greater than $P$.
\item $e$: number of epochs in the \W\ step.
\item $t^{\W}_r$: computation time per submodel and data point in the \W\ step. \\
  This is the time to process (within the current epoch) one data point by a submodel, i.e., the time do an SGD update to a weight vector, per data point (if we use minibatches, then this is the time to process one minibatch divided by the size of the minibatch).
\item $t^{\W}_c$: communication time per submodel in the \W\ step. \\
  This is the time to send one submodel from one machine to another, including overheads such as buffering, partitioning into messages or waiting time. We assume communication does not overlap with computation, i.e., a machine can either compute or communicate at a given time but not both. Also, communication involves time spent both by the sender and the receiver; we interpret $t^{\W}_c$ as the time spent by a given machine in first receiving a submodel and then sending it.
\item $t^{\Z}_r$: computation time per data point in the \Z\ step. \\
  This is the time to finish one data point entirely, using whatever optimisation algorithm performs the \Z\ step.
\end{itemize}
$P$, $N$, $M$ and $e$ are integers greater or equal than 1, and $t^{\W}_r$, $t^{\W}_c$ and $t^{\Z}_r$ are real values greater than 0. This model assumes that $t^{\W}_r$, $t^{\W}_c$ and $t^{\Z}_r$ are constant and equal for every submodel or and data point. In reality, even if the submodels are of the same mathematical form and dimension (e.g.\ each submodel is a weight vector of a linear SVM of dimension $D$), the actual times may vary somewhat due to many factors. However, as we will show in section~\ref{s:expts:speedup}, the model does agree quite well with the experimentally measured speedups.

Let us compute the runtimes in the \W\ and \Z\ step under these model assumptions. The runtime in the \Z\ step equals the time for any one machine to process its $N/P$ points on all $M$ submodels, i.e.,
\begin{equation}
  \label{e:runtime-Z}
  \textstyle T^{\Z}(P) = M \frac{N}{P} t^{\Z}_r
\end{equation}
since all machines start and end at the same time and do the same amount of computation, without communication. To compute the runtime in the \W\ step, we again consider the synchronous procedure of section~\ref{s:ParMAC}. At each tick of an imaginary clock, each machine processes its portion $M/P$ of submodels and sends it to its successor. After $P$ ticks, this concludes one epoch. This is repeated for $e$ epochs, followed by a final round of communication of all the submodels. If $M$ is not divisible by $P$, say $M = Q P + R$ with $Q,R \in \bbN$ and $0 < R < P$, we can apply this procedure pretending there are $P-R$ fictitious submodels%
\footnote{This means that our estimated runtime is an upper bound, because when $M$ is not divisible by $P$, there may be a better way to organise the computation in the \W\ step that reduces the time when any machine is idle. In practice this is irrelevant because we implement the computation asynchronously. Each machine keeps a queue of incoming submodels it needs to process, from which it repeatedly takes one submodel, processes it and sends it to the machine's successor.}
(on which machines do useless work). Then, the runtime in each tick is $\ceil{M/P} \frac{N}{P} t^{\W}_r$ (time for any one machine to process its $N/P$ points on its portion $\ceil{M/P}$ of submodels) plus $\ceil{M/P} t^{\W}_c$ (time for any one machine to send its portion of submodels). The total runtime of the \W\ step is then $Pe$ times this plus the time of the final round of computation:
\begin{equation}
  \label{e:runtime-W}
  \textstyle T^{\W}(P) = \ceil{M/P} \left( t^{\W}_r \frac{N}{P} + t^{\W}_c \right) P e + \ceil{M/P} t^{\W}_c P.
\end{equation}
(The final round actually requires $P-1$ ticks, but we take it as $P$ ticks to simplify the equation a bit.) Finally, the total runtime $T^{\W}(P) + T^{\Z}(P)$ of one ParMAC iteration (\W\ and \Z\ step) with $P$ machines is:
\begin{align}
  \label{e:runtime}
  T(P) &= \textstyle M \frac{N}{P} t^{\Z}_r + P \ceil{M/P} \left( e \left( t^{\W}_r \frac{N}{P} + t^{\W}_c \right) + t^{\W}_c \right),\ P > 1 \\
  T(1) &= M N t^{\Z}_r + M N e t^{\W}_r
\end{align}
where for $P=1$ machine we have no communication ($t^{\W}_c = 0$). Hence, the parallel speedup is
\begin{equation}
  \label{e:speedup1}
  S(P) = \frac{T(1)}{T(P)} = \frac{M/P}{\ceil{M/P}} \frac{e t^{\W}_r + t^{\Z}_r}{\frac{1}{P} \left(e t^{\W}_r + \frac{M/P}{\ceil{M/P}} t^{\Z}_r \right) + \frac{1}{N} (e+1) t^{\W}_c} = \frac{\frac{1}{\ceil{M/P}} M (e t^{\W}_r + t^{\Z}_r) P}{\frac{1}{N} (e+1) t^{\W}_c P^2 + e t^{\W}_r P + \frac{1}{\ceil{M/P}} M t^{\Z}_r}
\end{equation}
which can be written more conveniently as
\begin{equation}
  \label{e:speedup}
  S(P) = \frac{\rho \frac{1}{\ceil{M/P}} M P}{\frac{1}{N} P^2 + \rho_2 P + \rho_1 \frac{1}{\ceil{M/P}} M}
\end{equation}
by defining the following constants:
\begin{equation}
  \label{e:ratio-comp-comm}
  \rho_1 = \frac{t^{\Z}_r}{(e+1) t^{\W}_c} \qquad \rho_2 = \frac{e t^{\W}_r}{(e+1) t^{\W}_c} \qquad \rho = \rho_1 + \rho_2 = \frac{e t^{\W}_r + t^{\Z}_r}{(e+1) t^{\W}_c}.
\end{equation}
These constants can be understood as ratios of computation vs communication, independent of the training set size, number of submodels and number of machines. These ratios depend on the actual computation within the \W\ and \Z\ step, and on the performance of the distributed system (computation power of each machine, communication speed over the network or shared memory, efficiency of the parallel processing library that handles the communication between machines). The value of these ratios can vary considerably in practice, but it will typically be quite smaller than 1 (say, $\rho \in [10^{-4},1]$), because communication is much slower than computation in current computer architectures.

\subsection{Analysis of the speedup model}

We can characterise the speedup $S(P)$ of eq.~\eqref{e:speedup} in the following three cases:
\begin{itemize}
\item \emph{If $M \ge P$ and $M$ is divisible by $P$}, then we can write the speedup as follows:
  \begin{equation}
    \label{e:speedup-divisible}
    \text{if $M$ divisible by $P$:} \quad S(P) = 1 / \left( \frac{1}{P} + \frac{1}{\rho N} \right) = P / \left( 1 + \frac{P}{\rho N} \right) \le P.
  \end{equation}
  Here, the function $S(P)$ is independent of $M$ and monotonically increasing with $P$. It would asymptote to $\lim_{P\rightarrow\infty}{S(P) = \rho N}$, but the expression is only valid up to $P=M$. From~\eqref{e:speedup-divisible} we derive the following condition for perfect speedup to occur (in the limit)%
\footnote{Note that if $t^{\W}_c = 0$ (no communication overhead) then $\rho = \infty$ and there is no upper bound in~\eqref{e:speedup-divisible-perfect}, but $P \le N$ still holds, because we have to have at least one data point per machine.}:
  \begin{equation}
    \label{e:speedup-divisible-perfect}
    S \approx P \Longleftrightarrow P \ll \rho N.
  \end{equation}
  This gives an upper bound on the number $P$ of machines to achieve an approximately perfect speedup. Although $\rho$ is quite small in practice, the value of $N$ is very large (typically millions or greater), otherwise there would be no need to distribute the data. Hence, we expect $\rho N \gg 1$, so $P$ could be quite large. In fact, the limit in how large $P$ can be does not come from this condition (which assumes $P \le M$ anyway) but from the number of submodels, as we will see next. \\
  In summary, we conclude that if $M \ge P$ and $M$ is divisible by $P$ then the speedup $S(P)$ is given by~\eqref{e:speedup-divisible}, and in practice $S(P) \approx P$ typically.
\item \emph{If $M \ge P$ and $M$ is not divisible by $P$}, then $S(P)$ is given by the full expression~\eqref{e:speedup}, which is studied in appendix~\ref{s:speedup-app}. $S(P)$ is piecewise continuous on $M$ intervals of the form
  \begin{equation}
    \label{e:speedup-intervals}
    \textstyle\big[1,\frac{M}{M-1}\big),\ \big[\frac{M}{M-1},\frac{M}{M-2}\big),\ \dots,\ \big[\frac{M}{2},M\big),\ [M,\infty).
  \end{equation}
  Within each interval $P \in \big[\frac{M}{k},\frac{M}{k-1}\big)$ for $k = 1,2,3\dots,M$ we have $\ceil{M/P} = k$ and we obtain that $S(P)$ either is monotonically increasing, or is monotonically decreasing, or achieves a single maximum at
  \begin{equation}
    \label{e:speedup-max}
    P^*_k = \sqrt{\rho_1 M N / k} \qquad S^*_k = S(P^*_k) = \frac{\rho M / k}{\rho_2 + 2 \sqrt{\rho_1 M / N k}}.
  \end{equation}
  The parallelisation ability in this case is less than if $M$ is divisible by $P$, since now some machines are idle at some times during the \W\ step.
\item \emph{If $M < P$}, then we can write the speedup as follows:
  \begin{equation}
    \label{e:speedup-largeP}
    \text{if $M < P$:} \quad S(P) = \rho / \left( \frac{\rho_1}{P} + \frac{\rho_2}{M} + \frac{P}{M N} \right) = \rho M / \left( \rho_2 + \rho_1 \frac{M}{P} + \frac{P}{N} \right)
  \end{equation}
  which corresponds to the last interval $P \in [M,\infty)$ (for $k=1$) over which $S$ is continuous. We obtain that $S(P)$ either is monotonically decreasing (if $M \ge P^*_1$), or it increases from $P = M$ up to a single maximum at $P = P^*_1$ and then decreases monotonically, with
  \begin{equation}
    \label{e:speedup-largeP:max}
    P^*_1 = \sqrt{\rho_1 M N} \qquad S^*_1 = S(P^*_1) = \frac{\rho M}{\rho_2 + 2 \sqrt{\rho_1 M / N}}.
  \end{equation}
  As $P \rightarrow \infty$ we have that $S(P) \approx \rho N M / P \rightarrow 0$ (assuming $t^{\W}_c > 0$ so $\rho < \infty$). This decrease of the speedup for large $P$ is caused by the communication overhead in the \W\ step, where $P - M$ machines are idle at each tick in the \W\ step. \\
  In the impractical case where there is no communication cost ($t^{\W}_c = 0$ so $\rho = \infty$) then $S(P)$ is actually monotonically increasing and $\lim_{P \rightarrow \infty}{S(P)} = S^*_1 = \frac{\rho}{\rho_2} M > M$, so the more machines the larger the speedup, although with diminishing returns.
\end{itemize}
Theorem~\ref{th:speedup-charact} shows that $S(P)$ at the beginning of each interval is greater than anywhere before that interval, i.e., $S(M/k) > S(P)$ $\forall P < M/k$, for $k=1,2,\dots,M$. That is, although the speedup $S(P)$ is not necessarily monotonically increasing for $P \ge 1$, it is monotonically increasing for $P \in \big\{1,\frac{M}{M-1},\frac{M}{M-2},\dots,\frac{M}{2},M\big\}$. This suggests selecting values of $P$ that make $M/P$ integer, in particular when $M$ is divisible by $P$.

\paragraph{Globally maximum speedup $S^* = \max_{P \ge 1}{S(P)}$}

This is given by (see appendix~\ref{s:speedup-app}):
\begin{itemize}
\item If $M \ge \rho_1 N$: $S^* = M / \left( 1 + \frac{M}{\rho N} \right) \le M$, achieved at $P = M$.
\item If $M < \rho_1 N$: $S^* = S^*_1 = \frac{\rho M}{\rho_2 + 2 \sqrt{\rho_1 M / N}} > M$, achieved at $P = P^*_1 = \sqrt{\rho_1 M N} > M$.
\end{itemize}
In practice, with large values of $N$, the more likely case is $S^* = S^*_1 > M$ for $P = P^*_1 > M$. In this case, the maximum speedup is achieved using more machines than submodels (even though this means some machines will be idle at some times in the \W\ step), and is bigger than $M$. Since diminishing returns occur as we approach the maximum, the practically best value of $P$ will be somewhat smaller than $P^*_1$.

\paragraph{The ``large dataset'' case}

The case where $N$ is large is practically important because the need for distributed optimisation arises mainly from this. Specifically, if we take $P \ll \rho_2 N$, the speedup becomes (see appendix~\ref{s:speedup-app}):
\begin{equation}
  \label{e:speedup-largeN}
  \text{if $M$ divisible by $P$:} \quad S(P) \approx P; \qquad \text{if $M > P$:} \quad S(P) \approx \rho / \left( \frac{\rho_1}{P} + \frac{\rho_2}{M} \right)
\end{equation}
so that the speedup is almost perfect up to $P = M$, and then it is approximately the weighted harmonic mean of $M$ and $P$ (hence, $S(P)$ is monotonically increasing and between $M$ and $P$). For $P \gg \rho_1$, we have $S(P) \approx \frac{\rho}{\rho_2} M > M$.

\paragraph{The ``dominant \Z\ step'' case}

If we take $t^{\Z}_r \gg t^{\W}_r, t^{\W}_c$ or equivalently $\rho \approx \rho_1$ very large, which means the \Z\ step dominates the runtime, then $S(P) \approx P$. This is because the \Z\ step parallelises perfectly (as long as $P < N$).

\paragraph{Transformations that keep the speedup invariant}

We can rewrite the speedup of eq.~\eqref{e:speedup} as:
\begin{equation}
  \label{e:speedup-indepN}
  S(P) = \frac{\rho' \frac{1}{\ceil{M/P}} M P}{P^2 + \rho'_2 P + \rho'_1 \frac{1}{\ceil{M/P}} M}
\end{equation}
with
\begin{equation}
  \label{e:ratio-comp-comm2}
  \rho' = \rho N = (\rho_1+\rho_2) N = \frac{N(e t^{\W}_r + t^{\Z}_r)}{(e+1) t^{\W}_c} \qquad \rho'_1 = \rho_1 N = \frac{N t^{\Z}_r}{(e+1) t^{\W}_c} \qquad \rho'_2 = \rho_2 N = \frac{N e t^{\W}_r}{(e+1) t^{\W}_c}
\end{equation}
so that $S(P)$ is independent of $N$, which has been absorbed into the communication-computation ratios. This means that $S(P)$ depends on the dataset size ($N$) and computation/communication times ($t^{\W}_r$, $t^{\Z}_r$, $t^{\W}_c$) only through $\rho'$, $\rho'_1$ and $\rho'_2$, and is therefore invariant to parameter transformations that leave these ratios unchanged. Such transformations are the following (where $\alpha > 0$):
\begin{itemize}
\item Scaling $N$, $t^{\W}_r$ and $t^{\Z}_r$ as $\alpha N$, $\frac{1}{\alpha} t^{\W}_r$ and $\frac{1}{\alpha} t^{\Z}_r$. \\
  ``Larger dataset, faster computation,'' or ``smaller dataset, slower computation.''
\item Scaling $N$ and $t^{\W}_c$ as $\alpha N$ and $\alpha t^{\W}_c$ \\
  ``Larger dataset, slower communication,'' or ``smaller dataset, faster communication.''
\item Scaling $t^{\W}_r$, $t^{\Z}_r$ and $t^{\W}_c$ as $\alpha t^{\W}_r$, $\alpha t^{\Z}_r$ and $\alpha t^{\W}_c$ \\
  ``Faster computation, faster communication,'' or ``slower computation, slower communication.''
\end{itemize}

\subsection{Discussion and examples}

\begin{figure}[b!]
  \centering
  \psfrag{processor}[t][]{number of machines $P$}
  \psfrag{speedup}[][t]{speedup $S(P)$}
  \psfrag{PleM}[b][b]{\caja[1.5]{c}{c}{$P \le M$: \\ $S(P) = \displaystyle\frac{P}{1 + \frac{P}{\rho N}} \approx P$}}
  \psfrag{MleP}[t][t]{\caja[1.5]{c}{c}{$P > M$: \\ $S(P) = \displaystyle\frac{\rho}{\frac{\rho_1}{P} + \frac{\rho_2}{M} + \frac{P}{M N}} \approx \frac{\rho}{\frac{\rho_1}{P} + \frac{\rho_2}{M}}$}}
  \psfrag{maxS}[t][t]{$S(P^*_1)$}
  \begin{tabular}[c]{@{}c@{}}
    \includegraphics[width=\linewidth]{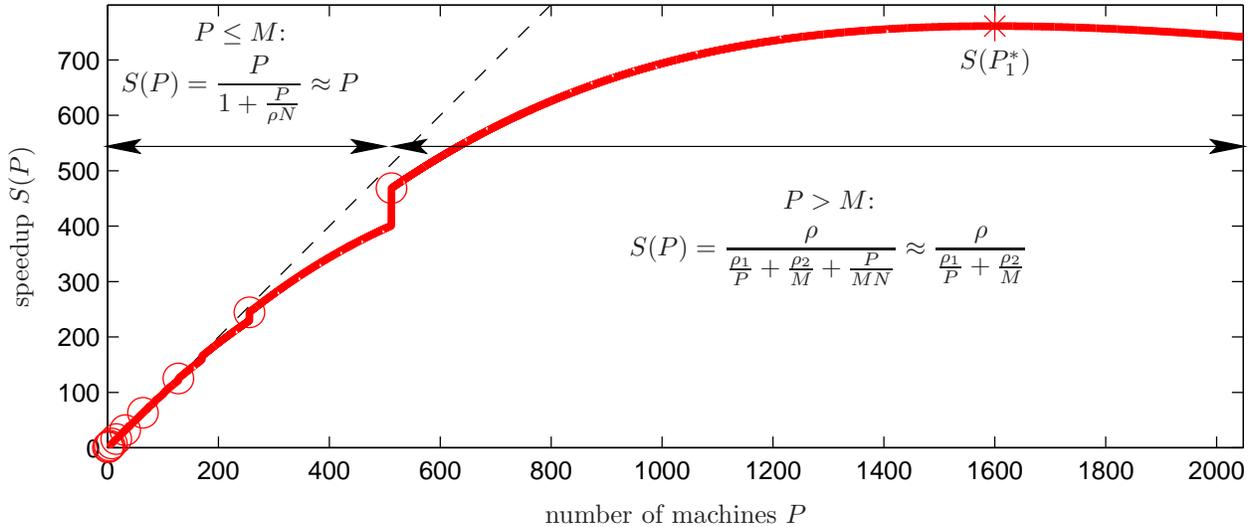}
  \end{tabular}
  \caption{Typical form of the theoretical speedup curve for realistic parameter settings, specifically $N = 10^6$ data points, $M = 512$ submodels, $e=1$ epoch in the \W\ step, and $t^{\W}_r = 1$ (this sets the units of time), $t^{\Z}_r = 5$ and $t^{\W}_c = 10^3$ (so $\rho_1 = 0.0025$, $\rho_2 = 0.0005$ and $\rho = 0.003$). Some of the discontinuities of the curve (where $\ceil{M/P}$ is discontinuous) are visible. We mark the values for $P$ such that $M$ is divisible by $P$ ($\circ$) and the maximum speedup ($\ast$), which occurs for $P = P^*_1 > M$.}
  \label{f:speedup-typical}
\end{figure}

Fig.~\ref{f:speedup-typical} plots a ``typical'' speedup curve $S(P)$, obtained with a realistic choice of parameter values. It displays the prototypical speedup shape we should expect in practice (the experimental speedups of fig.~\ref{f:speedup} confirm this). For $P \le M$ the curve is very close to the perfect speedup $S(P) = P$, slowly deviating from it as $P$ approaches $M$. For $P > M$, the curve continues to increase until it reaches its maximum at $P = P^*_1$, and decreases thereafter.

\begin{figure}[p]
  \centering
  \psfrag{processor}{}
  \begin{tabular}{@{}c@{\hspace{0.05\linewidth}}c@{\hspace{0.05\linewidth}}c@{}}
    & $e = 1$ epoch & $e = 8$ epochs \\
    \raisebox{0.16\linewidth}{\caja{c}{l}{$t^{\W}_c = 1$ \\ $t^{\Z}_r = 1$}} &
    \psfrag{speedup}[][t]{speedup $S(P)$}
    \includegraphics[height=0.30\linewidth]{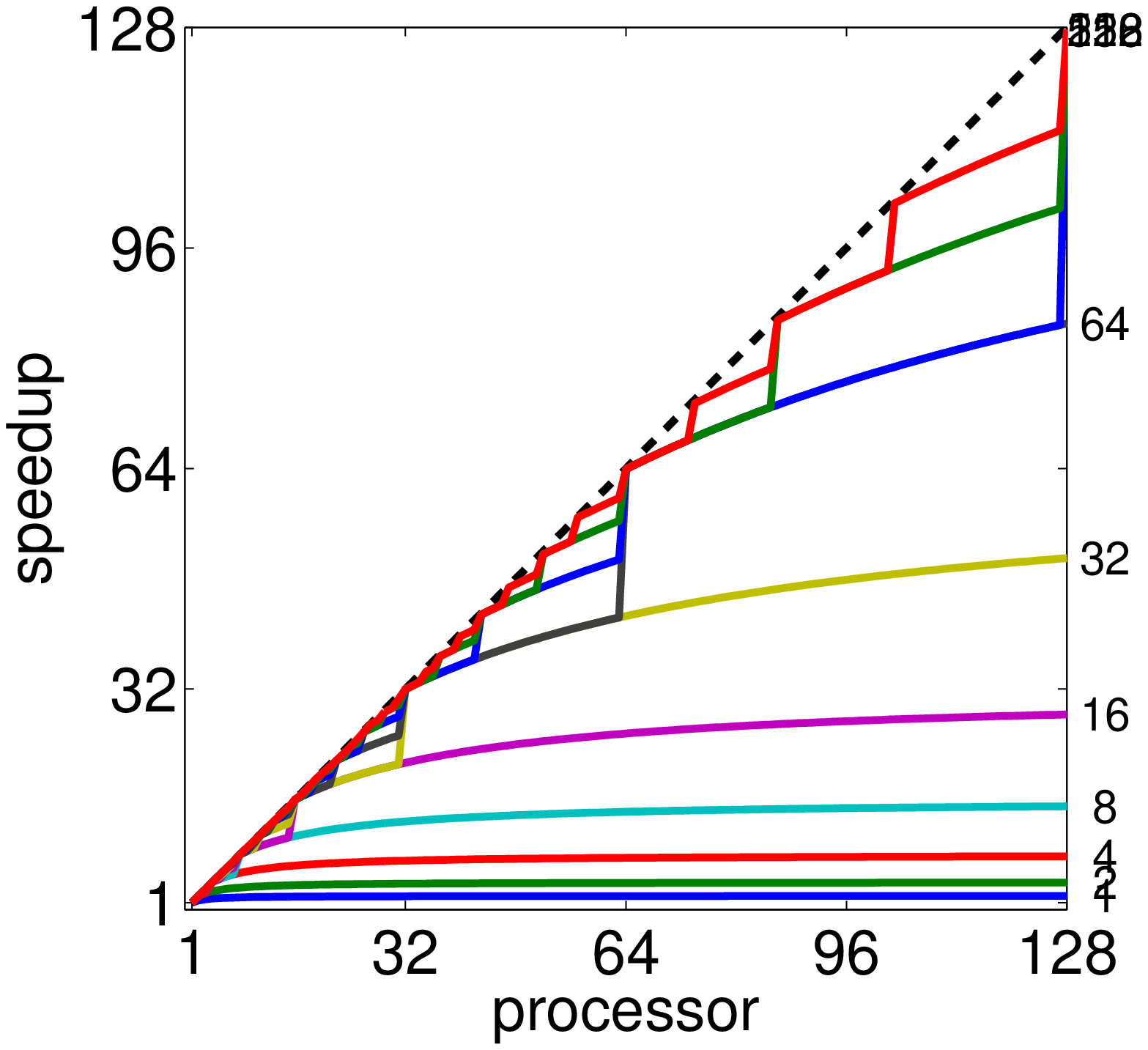} &
    \psfrag{speedup}{}
    \includegraphics[height=0.30\linewidth]{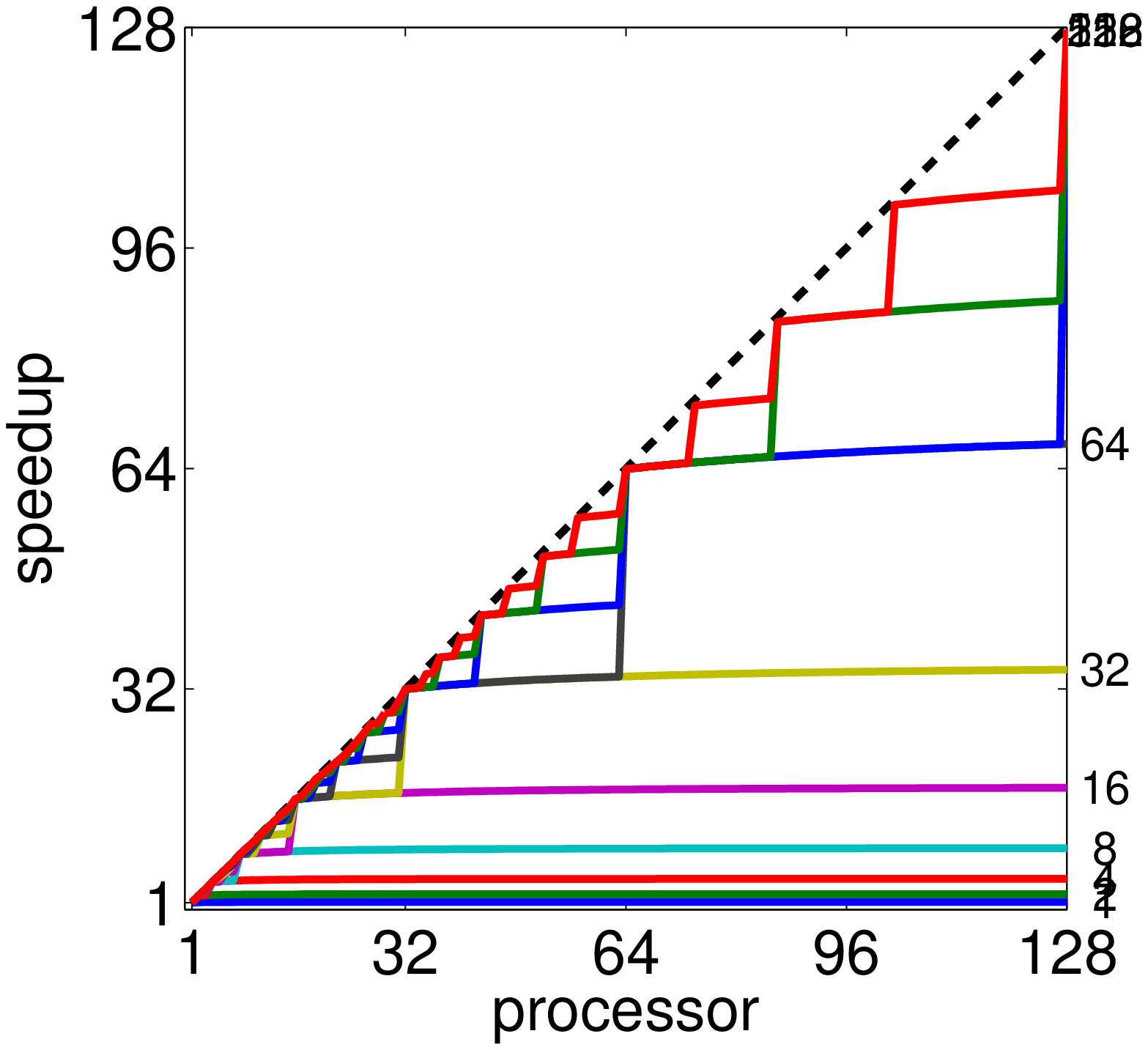} \\[-1ex]
    \raisebox{0.16\linewidth}{\caja{c}{l}{$t^{\W}_c = 1$ \\ $t^{\Z}_r = 100$}} &
    \psfrag{speedup}[][t]{speedup $S(P)$}
    \includegraphics[height=0.30\linewidth]{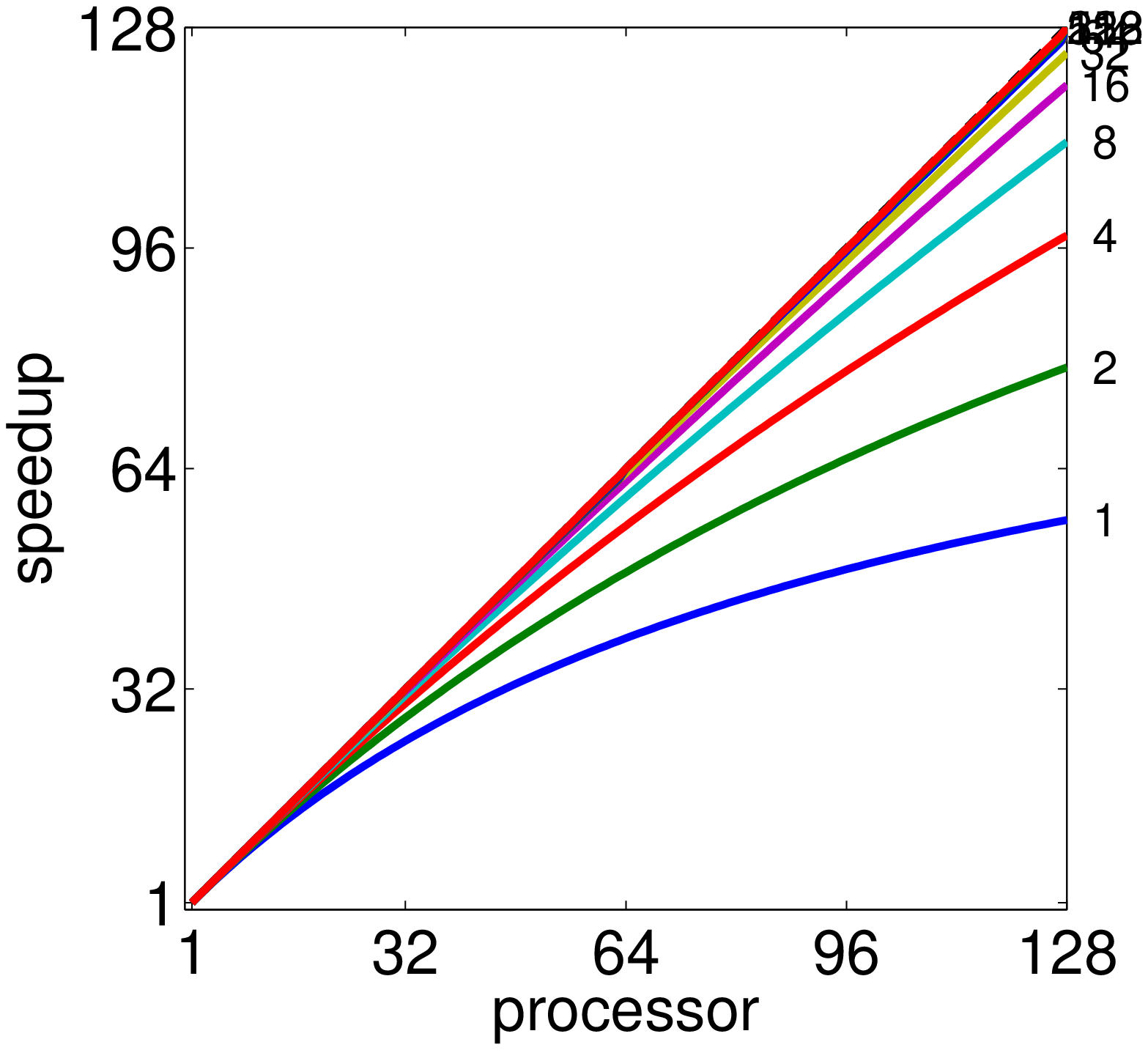} &
    \psfrag{speedup}{}
    \includegraphics[height=0.30\linewidth]{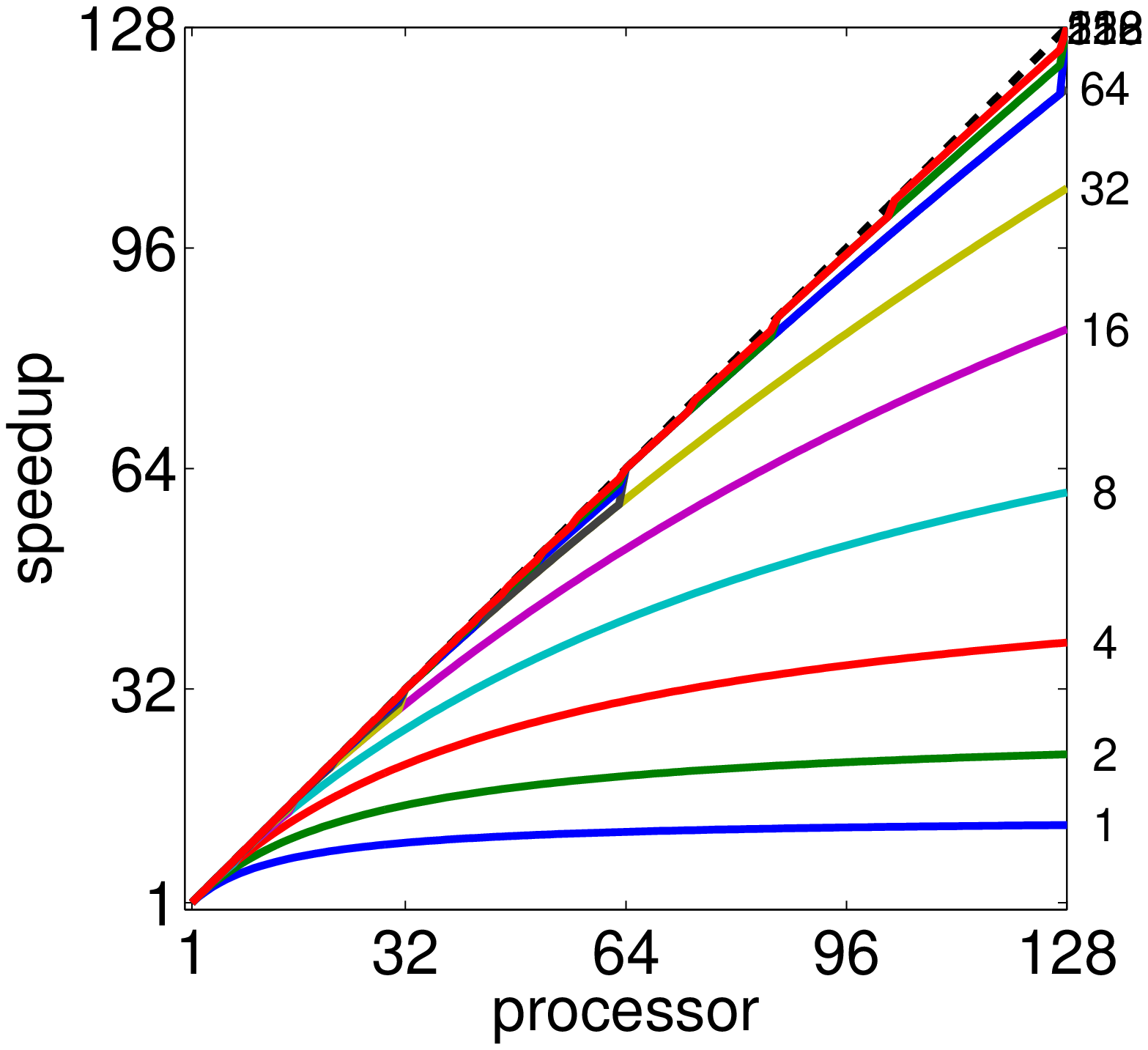} \\[-1ex]
    \raisebox{0.16\linewidth}{\caja{c}{l}{$t^{\W}_c = 100$ \\ $t^{\Z}_r = 1$}} &
    \psfrag{speedup}[][t]{speedup $S(P)$}
    \includegraphics[height=0.30\linewidth]{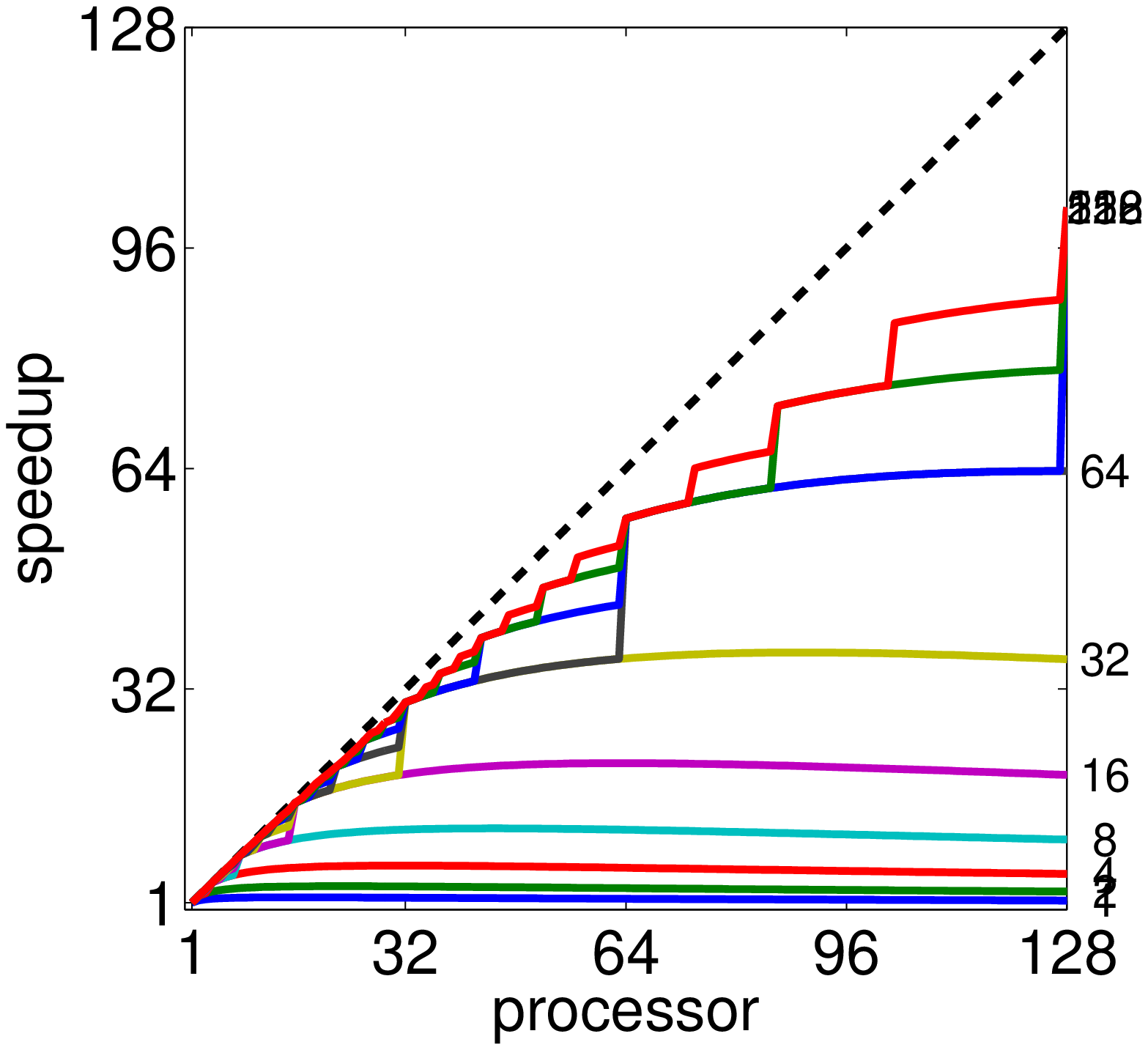} &
    \psfrag{speedup}{}
    \includegraphics[height=0.30\linewidth]{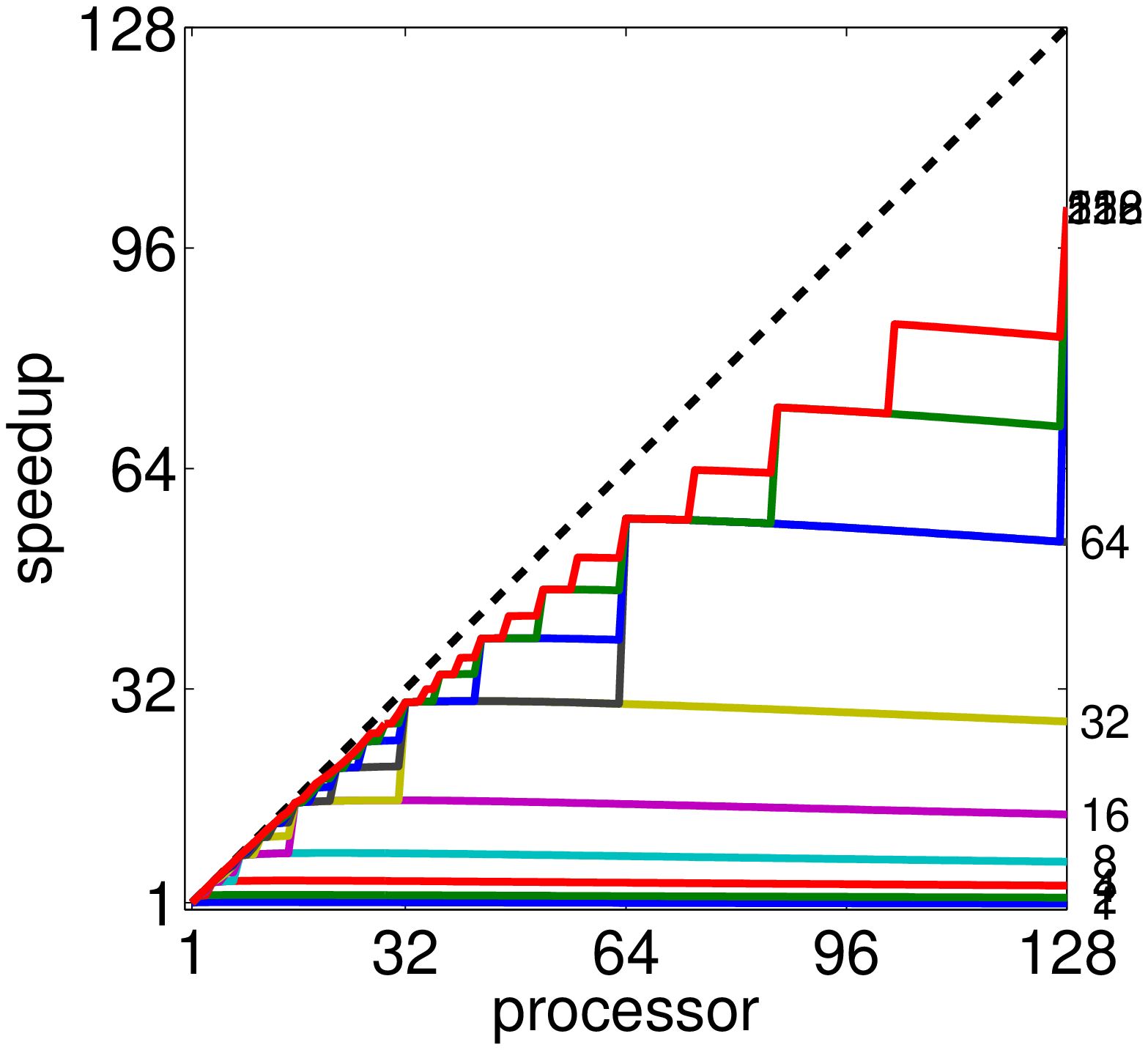} \\[-1ex]
    \raisebox{0.16\linewidth}{\caja{c}{l}{$t^{\W}_c = 1\,000$ \\ $t^{\Z}_r = 100$}} &
    \psfrag{processor}[t][]{number of machines $P$}
    \psfrag{speedup}[][t]{speedup $S(P)$}
    \includegraphics[height=0.30\linewidth]{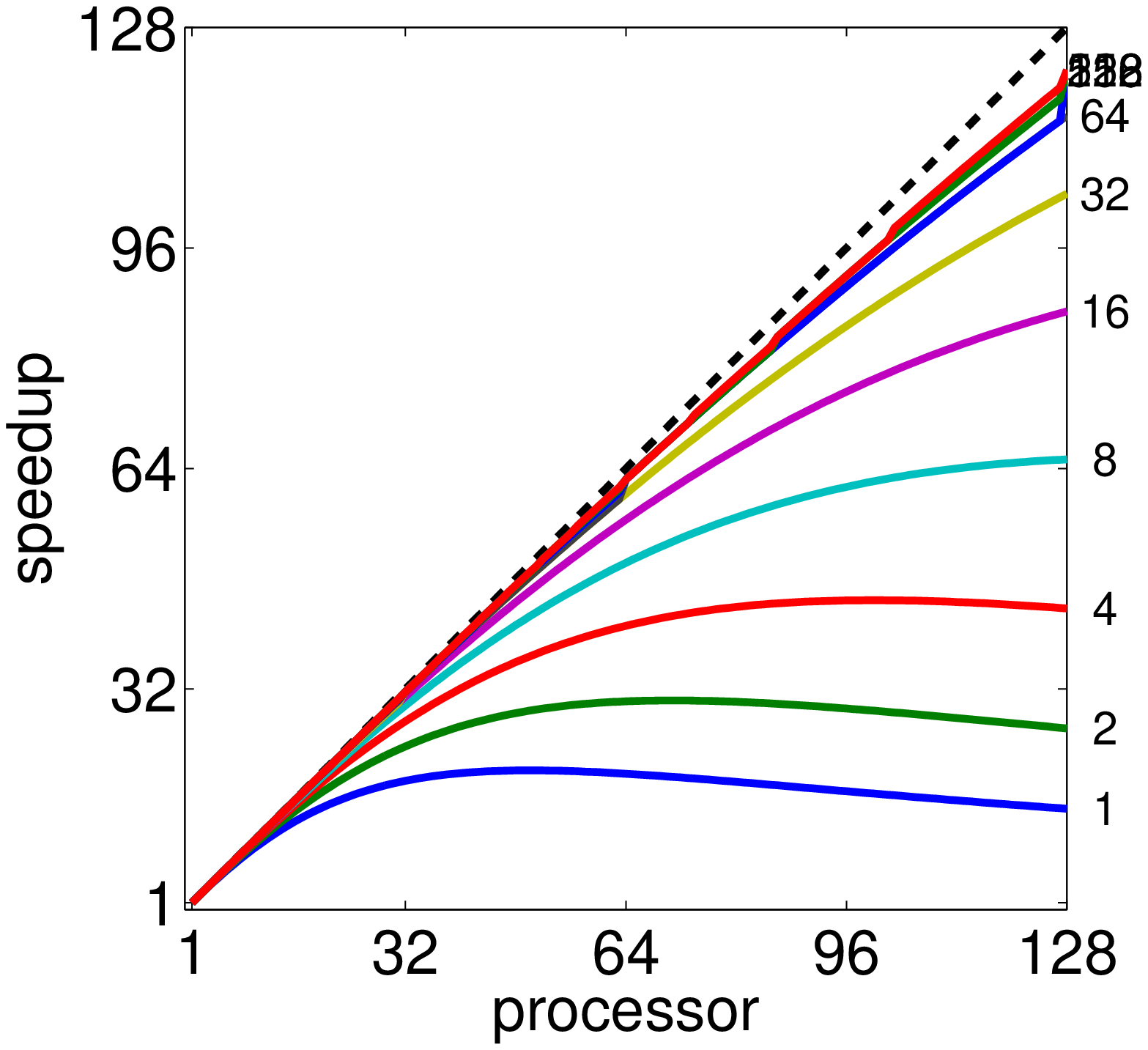} &
    \psfrag{processor}[t][]{number of machines $P$}
    \psfrag{speedup}{}
    \includegraphics[height=0.30\linewidth]{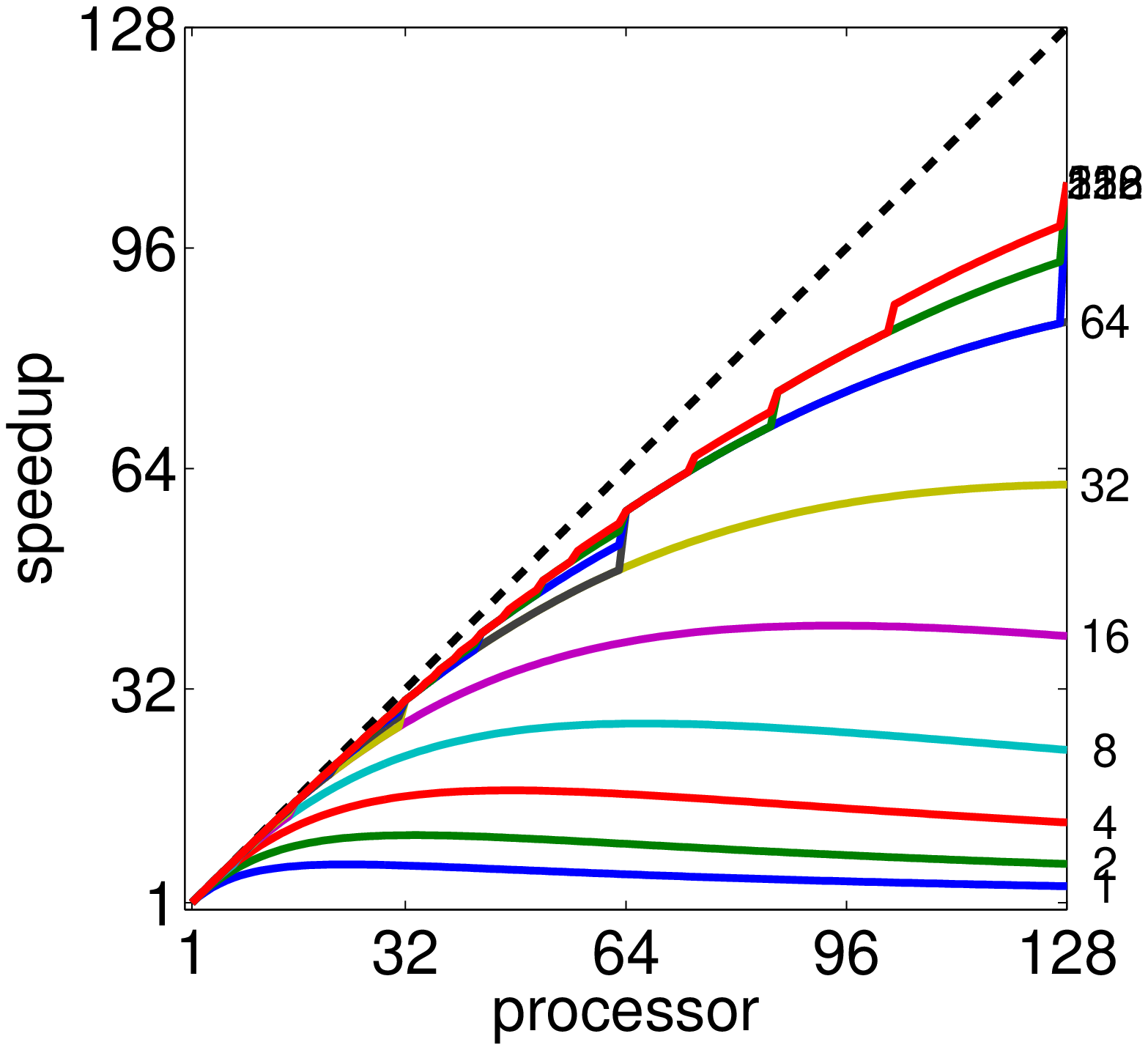}
  \end{tabular}
  \caption{Theoretical speedup $S(P)$ as a function of the number of machines $P$ for various settings of the parameters of ParMAC with a binary autoencoder. The parameters are: dataset size $N = 50\,000$ training points; number of submodels $M \in \{1,2,4,8,16,32,64,128,256,512\}$; number of epochs in the \W\ step $e \in \{1,8\}$; \W\ step computation time (per submodel and data point) $t^{\W}_r = 1$ (this sets the units of time); \W\ step communication time (per submodel) $t^{\W}_c \in \{1,100,1\,000\}$; \Z\ step computation time (per submodel and data point) $t^{\Z}_r \in \{1,100\}$. Within each plot, each curve corresponds to one value of $M$, indicated on the right end of the plot.}
  \label{f:speedup-th}
\end{figure}

Fig.~\ref{f:speedup-th} plots $S(P)$ for a wider range of parameter settings. We set the dataset size to a practically small value ($N = 50\,000$), otherwise the curves tend to look like the typical curve from fig.~\ref{f:speedup-typical}. The parameter settings are representative of different, potential practical situations (some more likely than others). We note the following observations:
\begin{itemize}
\item Again, the most important observation is that the number of submodels $M$ is the parameter with the most direct effect on the speedup: near-perfect speedups ($S \approx P$) occur if $M \ge P$, otherwise the speedups are between $M$ and $P$ (and eventually saturate if $P \gg M$).
\item When the time spent on communication is large in relative terms, the speedup is decreased. This can happen when the runtime of the \Z\ step is low (small $t^{\Z}_r$), when the communication cost is large (large $t^{\W}_c$), or with many epochs (large $e$). Indeed, since the \Z\ step is perfectly parallelisable, any decrease of the speedup should come from the \W\ step.
\item Some of the curves display noticeable discontinuities, caused by the function $(M/P)/\ceil{M/P}$, occurring at values of $P$ of the form $M/k$ for $k \in \{1,\dots,M\}$. At each such value, $S(P)$ is greater than for any smaller value of $P$, in accordance with theorem~\ref{th:speedup-charact}. This again suggests selecting values of $P$ that make $M/P$ integer, in particular when $M$ is a multiple of $P$ ($P$, $2P$, $3P$\dots). This achieves the best speedup efficiency in that machines are never idle (in our theoretical model). \\
  Also, for fixed $P$, the function $(M/P)/\ceil{M/P}$ can take the same value for different values of $M$ (e.g.\ $M=32$ and $M=64$ for $P=60$). This explains why some curves (for different $M$) partly overlap.
\item The maximum speedup is typically larger than $M$ and occurs for $P > M$. It is possible to have the maximum speedup be smaller than $M$ (in which case it occurs at $P=M$); an example appears in fig.~\ref{f:speedup-th}, row 3, column 2 (for the larger $M$ values). But this happens only when $M \ge \rho_1 N$, which requires an unusually small dataset and an impractically large number of submodels. Generally, we should expect $P > M$ to be beneficial. This is also seen experimentally in fig.~\ref{f:speedup}.
\end{itemize}

\subsection{Practical considerations}
\label{s:speedup-th:practical}

In practice, given a specific problem (with a known number of submodels $M$, epochs $e$ and dataset size $N$), our theoretical speedup curves can be used to determine optimal values for the number of machines $P$ to use. As seen in section~\ref{s:expts:speedup}, the theoretical curves agree quite well with the experimentally measured speedups. The theoretical curves do need estimates for the computation time and communication times of the \W\ and \Z\ steps. These are hard to obtain a priori; the computational complexity of the algorithm in \calO-notation ignores constant factors that affect significantly the actual times. Their estimates should be measured from test runs.

As seen from eq.~\eqref{e:speedup-indepN}, we can leave the speedup unchanged by trading off dataset size ($N$) and computation/communication times ($t^{\W}_r$, $t^{\Z}_r$, $t^{\W}_c$) in various ways, as long as one of the three following holds: the products $N t^{\W}_r$ and $N t^{\Z}_r$ remain constant; or the quotient $N/t^{\W}_c$ remains constant; or the quotients $t^{\W}_r/t^{\W}_c$ and $t^{\Z}_r/t^{\W}_c$ remain constant.

Theoretically, the most efficient operating points for $P$ are values such that $M$ is divisible by $P$, because this means no machine is ever idle. In practice with an asynchronous implementation and with $t^{\Z}_r$, $t^{\W}_r$ and $t^{\W}_c$ exhibiting some variation over submodels and data points, this is not true anymore. Still, if in a given application one is constrained to using $P \le M$ machines, choosing $P$ close to a divisor of $M$ would probably be preferable.

One assumption in our speedup model is that the $P$ machines are identical in processing power. The model does extend to the case where the machines are different, as noted in our discussion of load balancing (section~\ref{s:ParMAC:extensions}). This is because the work performed by each machine is proportional to the number of data points it contains: in the \W\ step, because every submodel runs (implicitly) SGD, and every submodel must visit each machine; in the \Z\ step, because each data point is a separate problem, and involves all submodels (which reside in each machine). Hence, we can equalise the work over machines by loading each machine with an amount of data proportional to its processing speed, independent of the number of submodels.

Another assumption in our model (in the \W\ step) is that all $M$ submodels are identical in ``size'' (computation and communication time). This is sometimes not true. For example, in the BA, we have submodels of two types: the $L$ encoders (each a binary linear SVM operating on a $D$-dimensional input) and the $D$ decoders (each a linear regressor operating on an $L$-dimensional input). Since $D > L$, the encoders are bigger than the decoders and take longer to train and communicate. We can still apply our speedup model if we ``group'' smaller submodels into a single submodel of size comparable to the larger submodels, so as to equalise as much as possible the submodel sizes (the actual implementation does not need to group submodels, of course). For the BA, under the reasonable assumption that the ratio of computation times $t^{\W}_r$ (and communication times $t^{\W}_c$) of decoder vs encoder is $L/D < 1$, we can group the $D$ decoders into $L$ groups of $D/L$ decoders each. Each group of decoders has now a computation and communication time equal to that of one encoder. This gives an effective number of independent submodels $M = 2 L$, and this is what we use when applying the model to the experimental speedups in section~\ref{s:expts:speedup}.

Finally, we emphasise that the goal of this section was to characterise the parallel speedup of ParMAC quantitatively and demonstrate the runtime gains that are achievable by using $P$ machines. In practice, other considerations are also important, such as the economic cost of using $P$ machines (which may limit the maximum $P$ available); the type of machines (obviously, we want all the computation and communication times as small as possible); the choice of optimisation algorithm in the \W\ and \Z\ steps; the fact that, because of its size, the dataset may need to be, or already is, distributed across $P$ machines; etc. It is also possible to combine ParMAC with other, orthogonal techniques. For example, if each of the submodels in the \W\ step is a convex optimisation problem (as is the case with the linear SVMs with the binary autoencoder), we could use the techniques described in section~\ref{s:related} for distributed convex optimisation to each submodel. This would effectively allow for larger speedups when $P > M$.

\section{Convergence of ParMAC}
\label{s:conv}

The only approximation that ParMAC makes to the original MAC algorithm is using SGD in the \W\ step. Since we can guarantee convergence of SGD under certain conditions, we can recover the original convergence guarantees for MAC. Let us see this in more detail. Convergence of MAC to a stationary point is given by theorem B.3 in \citet{CarreirWang12a}, which we quote here:
\begin{thm}
  \label{th:MACQP}
  Consider the constrained problem of eq.~\eqref{e:mac} and its quadratic-penalty function $E_Q(\W,\Z;\mu)$ of eq.~\eqref{e:mac-quadpen}. Given a positive increasing sequence $(\mu_k) \rightarrow \infty$, a nonnegative sequence $(\tau_k) \rightarrow 0$, and a starting point $(\W^0,\Z^0)$, suppose the quadratic-penalty method finds an approximate minimiser $(\W^k,\Z^k)$ of $E_Q(\W^k,\Z^k;\mu_k)$ that satisfies $\norm{\nabla_{\W,\Z}{E_Q(\W^k,\Z^k;\mu_k)}} \le \tau_k$ for $k=1,2,\dots$ Then, $\lim_{k\rightarrow\infty}{(\W^k,\Z^k)} = (\W^*,\Z^*)$, which is a KKT point for the problem~\eqref{e:mac}, and its Lagrange multiplier vector has elements $\blambda^*_n = \lim_{k\rightarrow\infty}{-\mu_k \, (\Z^k_n - \F(\Z^k_n,\W^k;\x_n))}$, $n=1,\dots,N$.
\end{thm}
This theorem applies to the general case of $K$ differentiable layers, where the standard Karush-Kuhn-Tucker (KKT) conditions hold \citep{NocedalWright06a}. It relies on a standard condition for penalty methods for nonconvex problems \citep{NocedalWright06a}, namely that we must be able to reduce the gradient of the penalised function $E_Q$ below an arbitrary tolerance $\tau_k \ge 0$ for each value $\mu_k$ of the penalty parameter (in MAC iterations $k=1,2,\dots$). This can be achieved by running a suitable (unconstrained) optimisation method for sufficiently many iterations. How does this change in the case of ParMAC? The \Z\ step remains unchanged with respect to MAC (the fact that the optimisation is distributed is irrelevant since the $N$ subproblems $\z_1,\dots,\z_N$ are independent). The \W\ step does change, because we are now obliged to use a distributed, stochastic training. What we need to ensure is that we can reduce the gradient of the penalised function with respect to each submodel (since they are independent subproblems in the \W\ step) below an arbitrary tolerance. This can also be guaranteed under standard conditions. In general, we can use convergence conditions from stochastic optimisation \citep{Benven_90a,KushnerYin03a,Pflug96a,Spall03a,BertsekTsitsik00a}. Essentially, these are Robbins-Monro schedules, which require the learning rate $\eta_t$ of SGD to decrease such that $\lim_{t\rightarrow\infty} \eta_t = 0$, $\sum^{\infty}_{t=1}{\eta_t} = \infty$, $\sum^{\infty}_{t=1}{\eta^2_t} < \infty$, where $t$ is the epoch number (SGD iteration, or pass over the entire dataset%
\footnote{Note that it is not necessary to assume that the points (or minibatches) are sampled at random during updates. Various results exist that guarantee convergence with deterministic errors (e.g.\ \citealp{BertsekTsitsik00a} and references therein), rather than stochastic errors. These results assume a bound on the deterministic errors (rather than a bound on the variance of the stochastic errors), and apply to general, nonconvex objective functions with standard conditions (Lipschitz continuity, Robbins-Monro schedules for the step size, etc.). They apply as a particular case to the ``incremental gradient'' method, where we cycle through the data points in a fixed sequence.}). 
We can give much tighter conditions on the convergence and the convergence rate when the subproblems in the \W\ step are convex (which is often the case, as with logistic or linear regression, linear SVMs, etc.). This is a topic that has received much attention recently (see section~\ref{s:related}), and many such conditions exist, often based on techniques such as Nesterov accelerated algorithms and stochastic average gradient \citep{Cevher_14a}. They typically bound the distance to the minimum in objective function value as $\calO(1/t^{\alpha})$ or $\calO(1/\beta^t)$ where the coefficients $\alpha > 0$, $0< \beta < 1$ and the constant factors in the \calO-notation depend on the (strong) convexity properties of the problem, Lipschitz constant, etc.

In summary, \emph{convergence of ParMAC to a stationary point is guaranteed by the same theorem as MAC, with an added SGD-type condition for the \W\ step}. This convergence guarantee is independent of the number of layers and submodels (since they are independent in the \W\ step) and the number of machines $P$ (since effectively we are doing SGD on shuffled datasets of size $N$, even if they are partitioned on portions of size $N/P$).

We can also guarantee ParMAC's convergence with only the original MAC theorem, without SGD-type conditions, while still in the distributed setting and achieving significant parallelism. This can be done by computing the gradient in the \W\ step exactly (as MAC assumes). First, each machine $p = 1,\dots,P$ computes the exact sum of per-point gradients for each submodel (by summing over its data portion), in parallel. Then, we aggregate these $P$ partial gradients into one exact gradient, for each submodel. This could be done via a parameter server, or by having each machine act as the parameter server for one submodel, and could be easily implemented with MPI functions. However, as is well known, this is far slower than using SGD.

With nondifferentiable layers, the convergence properties of MAC (and ParMAC) are not well understood. In particular, for the binary autoencoder the encoding layer is discrete and the problem is NP-complete. But, again, the only modification of ParMAC over MAC is the fact that the encoder and decoder are trained with SGD in the \W\ step, whose convergence tolerance can be achieved with SGD-type conditions. Indeed, our experiments show ParMAC gives almost identical results to MAC.

While convergence guarantees are important theoretically, in practical applications with large datasets in a distributed setting one typically runs SGD for just a few epochs, even one or less than one (i.e., we stop SGD before passing through all the data). This typically reduces the objective function to a good enough value as fast as possible, since each pass over the data is very costly. In our experiments, one to two epochs in the \W\ step make ParMAC very similar to MAC using an exact step.

\section{Implementation of ParMAC for binary autoencoders}
\label{s:ParMAC-BAhash-implem}

We have implemented ParMAC for binary autoencoders in C/C++ using the GNU Scientic Library (GSL) (\url{http://www.gnu.org/s/gsl}) and Basic Linear Algebra Subroutines (BLAS) library (\url{http://www.netlib.org}) for mathematical operations and linear algebra, and the Message Passing Interface (MPI) \citep{Gropp_99a,Gropp_99b,MPI12a} for interprocess communication.

GSL and BLAS provide a wide range of mathematical routines such as basic matrix operations, various matrix decompositions and least-squares fitting. We used the versions of GSL and BLAS that come with our Linux distribution (Ubuntu 14.04). Considerably better performance could be achieved by using LAPACK and an optimised version of BLAS (such as ATLAS, or as provided by a computer vendor for their specific architecture).

MPI is one of the most widely used frameworks for high-performance parallel computing today, and is the best option for ParMAC because of its support for distributed-memory machines and SPMD (single program, multiple data) model, its language independence, and its availability in multiple machines, from small shared-memory multiprocessor machines to hybrid clusters. In MPI, different processes cannot directly access each other's memory space, but data can be transferred by sending messages from one process to another, or collectively among multiple processes. The SPMD model, very useful in distributed machine learning, means that all processes share the same code (and executable file), and each of them can operate on different data with flow control using its individual process id. 

MPI is an industry standard for which there are many implementations, such as MPICH or OpenMPI, mostly compatible with each other. We used MPICH on our UC Merced shared-memory cluster and OpenMPI on the UCSD TSCC distributed cluster (see section~\ref{s:expts:setup}). Our ParMAC C++ code compiles and runs with both implementations. We used the highest compiler optimisation level, specifically we ran \texttt{mpicc -O3 -lgsl -lgslcblas -lm}. This calls the GNU C compiler with option \texttt{-O3}, which turns on all the available code optimisation flags. It results in a longer compilation time but more efficient code.

\begin{figure}[p]
  \small
  \begin{tabbing}
    n \= n \= n \= n \= n \= \kill
    \texttt{MPI\_Init(\&argc, \&argv);} \` // initialise the MPI execution environment \\
    \texttt{MPI\_Comm\_rank(MPI\_COMM\_WORLD, \&mpirank);} \` // get the rank of the calling MPI process \\
    \texttt{MPI\_Comm\_size(MPI\_COMM\_WORLD, \&mpisize);} \` // get the total number of MPI processes \\
    \texttt{loadsettings();} \` // load parameters ($\mu$, epochs, dataset path, etc.) \\
    \texttt{loaddatasets();} \` // load input and output datasets and initial auxiliary coordinates \\
    \texttt{initializelayers();} \` // allocate memory and initialise \f, \h\ and \Z\ steps \\
    // we use \texttt{MPI\_Bsend} to avoid managing send buffers, so we need to allocate the required \\
    // amount of buffer space into which data can be copied until it is delivered \\
    \texttt{MPI\_Pack\_size(commbuffsize, MPI\_CHAR, MPI\_COMM\_WORLD, \&mpi\_attach\_buff\_size);} \\
    // allocate enough memory so it can store the whole model \\
    \texttt{mpi\_attach\_buff = malloc(totalsubmodelcount*(mpi\_attach\_buff\_size+MPI\_BSEND\_OVERHEAD));} \\
    \texttt{MPI\_Buffer\_attach(mpi\_attach\_buff, mpi\_attach\_buff\_size);} \` // attach the allocated buffer \\
    \\
    \texttt{for (iter=1 to length($\mu$)) \{} \` // iterate over all the values of $\mu$ \+ \\
    // begin \W-step \\
    \texttt{visitedsubmodels = 0;} \\
    // each process visits all the submodels, epochs + 1 times \\
    \texttt{while (visitedsubmodels <= totalsubmodelcount*epochs) \{} \+ \\
    // \texttt{stepcounter} is a number that each submodel carries and increases by one in each step. \\
    // Once it reaches a certain value we stop sending the submodel around and it stops. \\
    // We reset \texttt{stepcounter} for all the submodels in the beginning of each \W-step. \\
    \texttt{if (stepcounter > 0) \{} \` // if this is not the first submodel to train in the iteration, we wait to receive \+ \\
    // \texttt{MPI\_Recv} blocks until the requested data is available in the application buffer in the receiving task \\
    \texttt{\textcolor{red}{MPI\_Recv}(receivebuffer, commbuffsize, MPI\_CHAR, MPI\_ANY\_SOURCE, MODEL\_MSG\_TAG,} \\
    \hspace{20mm}\texttt{MPI\_COMM\_WORLD, \&recvStatus);} \\
    \texttt{savesubmodel(receivebuffer);} \` // save the received buffer into a suitable struct \- \\
    \texttt{\}} \\
    \texttt{if (stepcounter < epochs*mpisize) \{} \` // we don't train the submodels in the last update round \+ \\
    \texttt{switch(submodeltype)} \` // train each submodel according to its type \\
    \texttt{case 'SVM': HtrainSGD();} \\
    \texttt{case 'linlayer': FtrainSGD();} \- \\
    \texttt{\}} \\
    \texttt{if (stepcounter < (ringepochs+1)*mpisize) \{} \` // we still need to send this submodel around \+ \\
    // the lookup table is created randomly and stores the path of each submodel over epochs and iterations \\
    \texttt{successor = next\_in\_lookuptable();} \` // pick the successor process from the lookup table \\
    \texttt{loadsubmodel(sendbuffer);} \` // load the submodel from its struct into the send buffer \\
    // \texttt{MPI\_Bsend} returns after the data has been copied from application buffer space to the allocated send buffer \\
    \texttt{\textcolor{red}{MPI\_Bsend}(sendbuffer, taskbufsize*sizeof(double), MPI\_CHAR, successor, MODEL\_MSG\_TAG,} \\
    \hspace{20mm}\texttt{MPI\_COMM\_WORLD);} \- \\
    \texttt{\}} \\
    \texttt{visitedsubmodels++;} \- \\
    \texttt{\}} \\
    // end \W-step \\
    \\
    // begin \Z-step \\
    \texttt{updateZ\_relaxed();} \` // initialise auxiliary coordinates based on a truncated, relaxed solution \\
    \texttt{updateZ\_alternate();} \` // update auxiliary coordinates by alternating optimisation over bits \\
    // end \Z-step \- \\
    \texttt{\}} \\
    \\
    \texttt{MPI\_Buffer\_detach(\&mpi\_attach\_buff, \&mpi\_attach\_buff\_size);} \` // detach the allocated buffer \\
    \texttt{free(mpi\_attach\_buff);} \` // free the allocated memory \\
    \texttt{MPI\_Finalize();} \` // terminate the MPI execution environment
  \end{tabbing}
  \caption{Binary autoencoder ParMAC algorithm (fragment), showing important MPI calls.}
  \label{f:BA-ParMAC-alg}
\end{figure}

The code snippet in figure~\ref{f:BA-ParMAC-alg} shows the main steps of the ParMAC algorithm for the BA. All the functions starting with \texttt{MPI\_} are API calls from the MPI library. As with all MPI programs, we start the code by initialising the MPI environment and end by finalising it. To receive data we use the synchronous%
\footnote{Note that the word ``synchronous'' here does not refer to how we process the different submodels, which as we stated earlier are not synchronised to start or end at specific clock ticks, hence are processed asynchronously with respect to each other. The word ``synchronous'' here refers to MPI's handling of an \emph{individual} receive function (see appendix~\ref{s:MPI}). This can be done either by calling \texttt{MPI\_Recv}, which will block until the data is received (synchronous blocking function), as in the pseudocode in fig.~\ref{f:BA-ParMAC-alg}; or by calling \texttt{MPI\_Irecv} (asynchronous nonblocking function) followed by a \texttt{MPI\_Wait}, which will block until the data is received, like this:
  \begin{tabbing}
    n \= n \= n \= n \= n \= \+ \kill
    \texttt{MPI\_Irecv(receivebuffer, commbuffsize, MPI\_CHAR, MPI\_ANY\_SOURCE, MODEL\_MSG\_TAG, MPI\_COMM\_WORLD, \&recvRequest);} \\
    \texttt{MPI\_Wait(\&recvRequest, \&recvStatus);}
  \end{tabbing}
  Both options are equivalent for our purpose, which is to ensure we receive the submodel before starting to train it. The \texttt{MPI\_Irecv}/\texttt{MPI\_Wait} option is slightly more flexible in that it would allow us to do some additional processing between \texttt{MPI\_IRecv} and \texttt{MPI\_Wait} and possibly achieve some performance gain.},
blocking MPI receive function \texttt{MPI\_Recv}. The process calling this blocks until the data arrives. To send data we use the buffered blocking version of the MPI send functions, \texttt{MPI\_Bsend}. This requires that we allocate enough memory and attach it to the system in advance. The process calling \texttt{MPI\_Bsend} blocks until the buffer is copied to the MPI internal memory; after that, the MPI library takes care of sending the data appropriately. The benefit of using this version of send is that the programmer can send messages without worrying about where they are buffered, so the code is simpler. Appendix~\ref{s:MPI} briefly describes important MPI functions and their arguments.

\section{Experiments}
\label{s:expts}

\subsection{Setup}
\label{s:expts:setup}

\paragraph{Computing systems}

We used two different computing systems, to which we will refer as \emph{distributed} and \emph{shared-memory}:
\begin{description}
\item[Distributed-memory] This used General Computing Nodes from the UCSD Triton Shared Computing Cluster (TSCC), available to the public for a fee. Each node contains 2 8-core Intel Xeon E5-2670 processors (16 cores in total), 64GB DRAM (4GB/core) and a 500GB hard drive. The nodes are connected through a 10GbE network. We used up to $P=128$ processors. Detailed specs are in table~\ref{t:specs} (obtained by running \texttt{dmidecode} in the actual processor) and \url{http://idi.ucsd.edu/computing}.
\item[Shared-memory] This is a 72-processor machine (36 physical cores with hyperthreading) with 256GB RAM located at UC Merced. The processors communicate through shared memory. We used this only for the large-scale experiment, and we used 64 of the 72 processors. Detailed specs are in table~\ref{t:specs} (obtained by running \texttt{dmidecode} in the actual processor).
\end{description}
In both systems, the interprocess communication is handled by MPI (OpenMPI on the TSCC cluster and MPICH on our shared-memory cluster). The shared-memory system has both faster processors and faster communication than the distributed one and this is seen in our experiments (3--4 times faster). This does not imply that shared-memory systems are necessarily superior in practice, it simply reflects characteristics of the equipment we had access to. The ParMAC speedups as a function of the number of processors are comparable in both systems.

\begin{table}
  \caption{Detailed hardware specification of the two machines used in our experiments.}
  \centering
  \begin{tabular}{@{}lll@{}}
    \toprule
    & Distributed-memory (TSCC at UCSD) & Shared-memory (cluster at UC Merced) \\
    \midrule
    CPU & Intel(R) Xeon(R) CPU E5-2670 0 & Intel(R) Xeon(R) CPU E5-2699 v3 \\
    CPU cache & 20 MB & 45 MB  \\
    CPU max frequency & 3.3 GHz & 3.6 GHz \\
    Cores/threads & 8/16 & 8/16 \\
    Memory types & DDR3 800/1066/1333/1600 & DDR4 1600/1866/2133 \\
    RAM bandwidth & 51.2 GB/s & 68 GB/s \\
    Processor connection & 10GbE & shared memory \\
    \bottomrule
  \end{tabular}
  \label{t:specs}
\end{table}

\paragraph{Datasets}

We have used 4 datasets commonly used as image retrieval benchmarks. (1) CIFAR \citep{Krizhev09a} contains $60\,000$ $32\times 32$ colour images in 10 object classes. We ignore the labels in this paper and use $N = 50\,000$ images as training set and $10\,000$ as test set. We extract $D = 320$ GIST features \citep{OlivaTorral01a} from each image. (2) SIFT-10K \citep{Jegou_11a} contains $N = 10\,000$ training high-resolution colour images and $100$ test images, each represented by $D = 128$ SIFT features. (3) SIFT-1M \citep{Jegou_11a} contains $N = 10^6$ training and $10^4$ test images. (3) SIFT-1B (\citealp{Jegou_11b}; \url{http://corpus-texmex.irisa.fr}) has three subsets: $10^9$ base vectors where the search is performed, $N = 10^8$ learning vectors used to train the model and $10^4$ query vectors.

\paragraph{Performance measures}

Regarding the quality of the BA and hash functions learnt, we report the following. (1) The binary autoencoder error $E_{\text{BA}}(\h,\f)$ which we want to minimise, eq.~\eqref{e:BA-nested}. (2) The quadratic-penalty function $E_Q(\h,\f,\Z;\mu)$ which we actually minimise for each value of $\mu$, eq.~\eqref{e:BA-MAC-QP}. (3) The retrieval precision (\%) in the test set using as true neighbours the $K$ nearest images in Euclidean distance in the original space, and as retrieved neighbours in the binary space we use the $k$ nearest images in Hamming distance. We set $(K,k) = (1\,000,100)$ for CIFAR, $(100,100)$ for SIFT-10K and $(10\,000,10\,000)$ for SIFT-1M. For SIFT-1B, as suggested by the dataset creators, we report the recall@R: the average number of queries for which the nearest neighbour is ranked within the top $R$ positions (for varying values of $R$); in case of tied distances, we place the query as top rank. All these measures are computed offline once the BA is trained.

\paragraph{Models and their parameters}

We use BAs with linear encoders (linear SVM) except with SIFT-1B, where we also use kernel SVMs. The decoder is always linear. We set $L = 16$ bits (hash functions) for CIFAR, SIFT-10K and SIFT-1M and $L = 64$ bits for SIFT-1B. The \Z\ step uses enumeration for SIFT-10K and SIFT-1M, and alternating optimisation (initialised by a truncated relaxed solution) otherwise. We initialise the binary codes from truncated PCA ran on a subset of the training set (small enough that it fits in one machine).

To train the encoder ($L$ SVMs) and decoder ($D$ linear mappings) with stochastic optimisation, we used the SGD code from \citet{BottouBousquet08a} (\url{http://leon.bottou.org/projects/sgd}), using its default parameter settings. The SGD step size is tuned automatically in each iteration by examining the first $1\,000$ datapoints.

We use a multiplicative $\mu$ schedule $\mu_i = \mu_0 a^i$ where the initial value $\mu_0$ and the factor $a > 1$ are tuned offline in a trial run using a small subset of the data. For CIFAR we use $\mu_0 = 0.005$ and $a = 1.2$ over 26 iterations ($i = 0,\dots,25$). For SIFT-10K and SIFT-1M we use $\mu_0 = 10^{-6}$ and $a = 2$ over 20 iterations. For SIFT-1B we use $\mu_0 = 10^{-4}$ and $a = 2$ over 10 iterations.

\subsection{Effect of stochastic steps in the \W\ step}
\label{s:expts:stoch}

Figures~\ref{f:sift10k-epochs} to~\ref{f:cifar-epochs-shuffle} show the effect on SIFT-10K and CIFAR of varying the number of epochs and shuffling the data as a function of the number of machines $P$ on the learning curves (errors $E_Q$ and $E_{\text{BA}}$, precision). As the number of epochs increases, the \W\ step is solved more exactly (8 epochs is practically exact in these datasets). Fewer epochs, even just one, cause only a small degradation. The reason is that, although these are relatively small datasets, they contain sufficient redundance that few epochs are sufficient to decrease the error considerably. This is also helped by the accumulated effect of epochs over MAC iterations. Running more epochs increases the runtime and lowers the parallel speedup in this particular model, because we use few bits ($L = 16$) and therefore few submodels ($M = 2L = 32$) compared to the number of machines (up to $P = 128$), so the \W\ step has less parallelism.

Fig.~\ref{f:cifar-epochs-shuffle} shows the positive effect of data shuffling in the \W\ step. To shuffle the minibatches, the successor of a machine is given by a random lookup table. Shuffling generally reduces the error (this is particularly clear in $E_Q$, which is what we actually minimise) and increases the precision with no increase in runtime.

Note that, even if we keep the circular topology fixed throughout the \W\ step (i.e., we do not randomise the topology at each epoch), there is still a small amount of shuffling. This occurs because, although all submodels process the data minibatches in the same ``direction'', submodels in different machines start at different minibatches. Were it not for this (and if no shuffling was done within-machine), ParMAC would give an identical result no matter the number of machines. However, as seen from figures~\ref{f:sift10k-epochs} to~\ref{f:cifar-epochs-shuffle}, it seems that this small intrinsic shuffling simply randomises the learning curves, but does not make them better than the learning curve for one machine.

\begin{figure}[t]
  \psfrag{iteration}{}
  \psfrag{time}{}
  \psfrag{QPerror}[][t]{\footnotesize$E_Q$}
  \psfrag{BAerror}[][t]{\footnotesize$E_{\text{BA}}$}
  \psfrag{precision}[][t]{precision}
  \begin{tabular}{@{}c@{\hspace{0\linewidth}}c@{\hspace{0.015\linewidth}}c@{\hspace{0\linewidth}}c@{}}
    \multicolumn{2}{c}{\makebox[0.45\linewidth][c]{\dotfill Number of epochs in \W\ step\dotfill}} & \multicolumn{2}{c}{\makebox[0.45\linewidth][c]{\dotfill Number of machines $P$\dotfill}} \\
    error-iteration view & error-time view & 1 epoch in \W\ step & 8 epochs in \W\ step \\[-1ex]
    \includegraphics[width=0.274\linewidth]{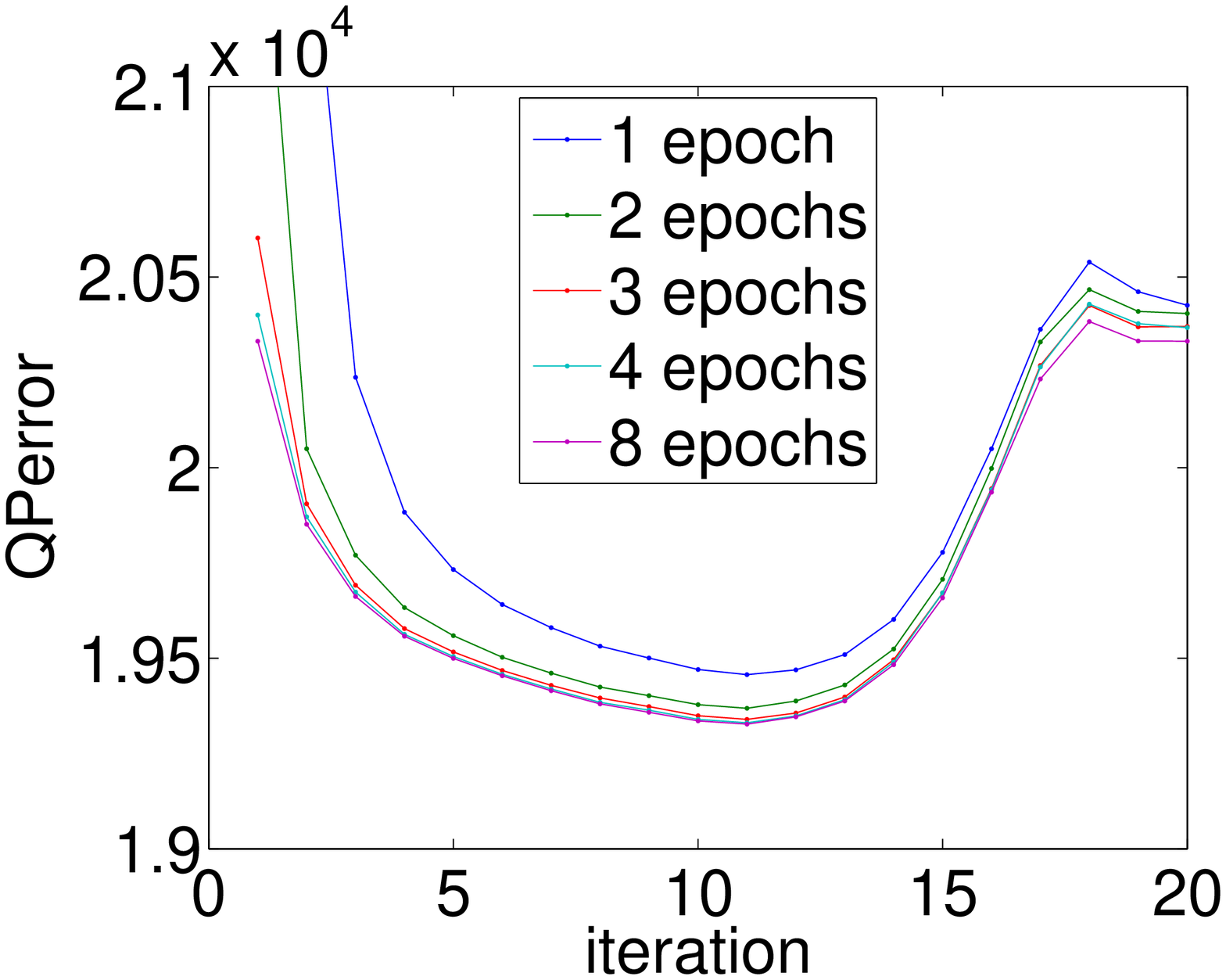} &
    \includegraphics[width=0.24\linewidth]{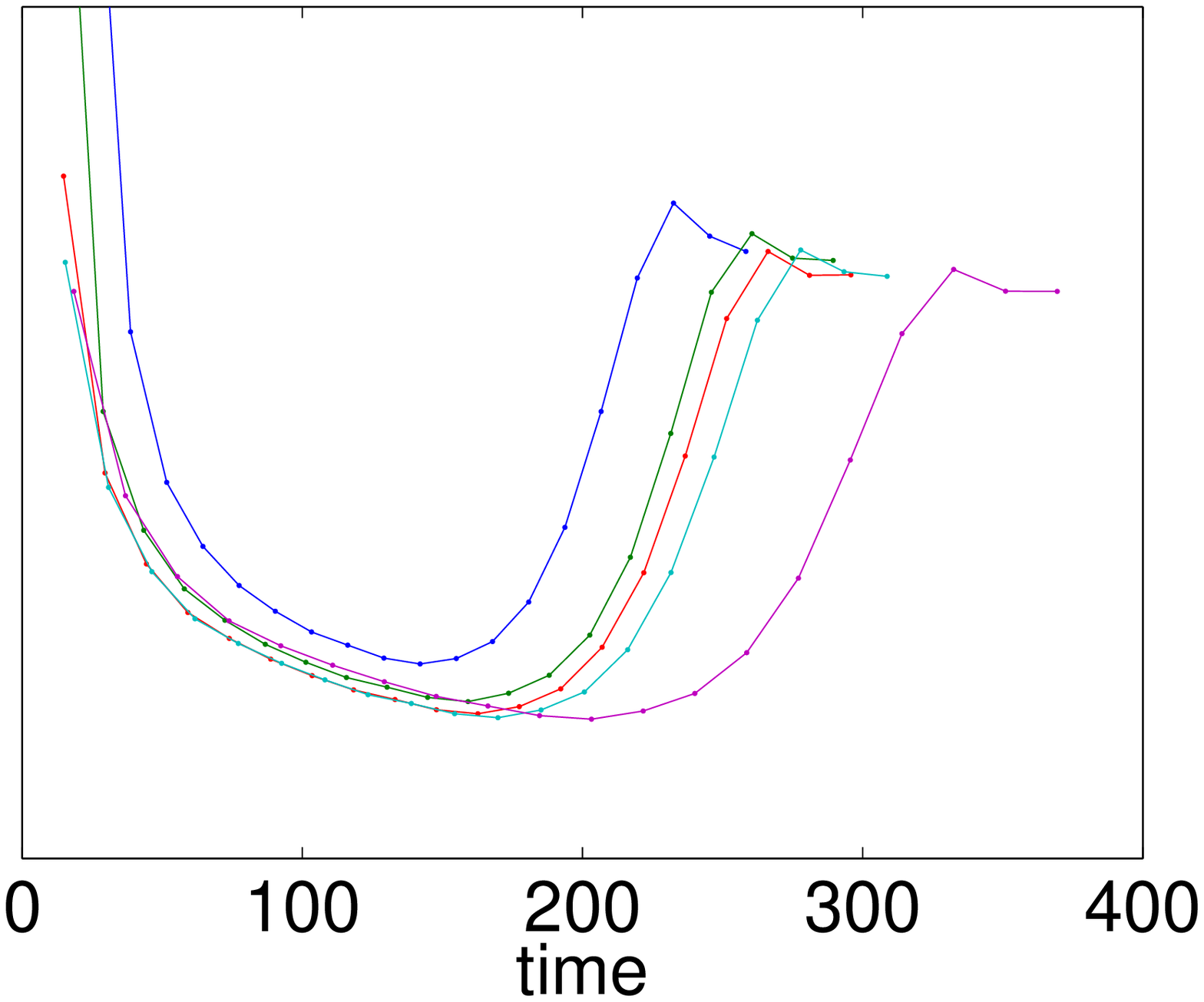} &
    \includegraphics[width=0.235\linewidth]{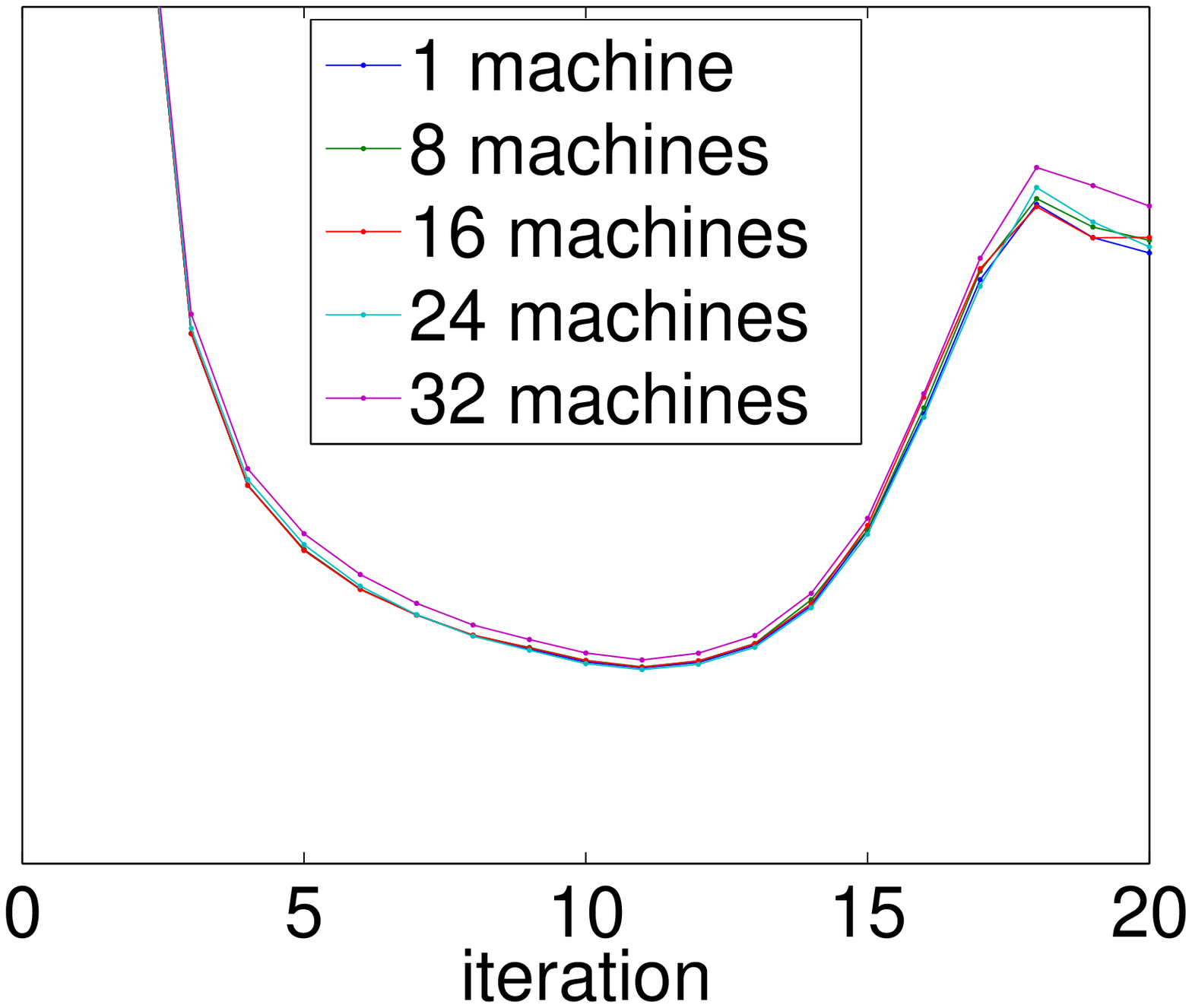} &
    \includegraphics[width=0.235\linewidth]{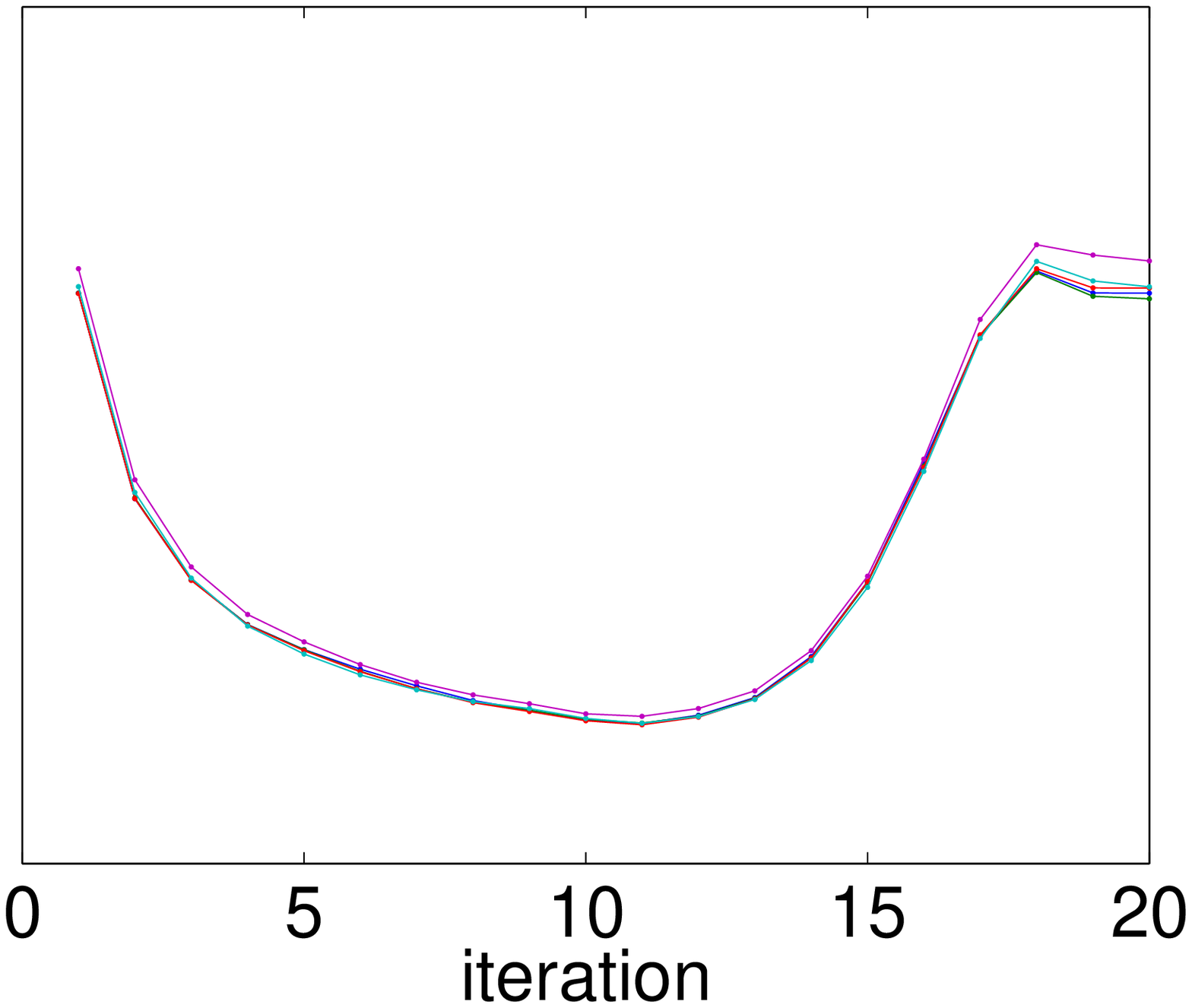} \\[-2.5ex]
    \includegraphics[width=0.274\linewidth]{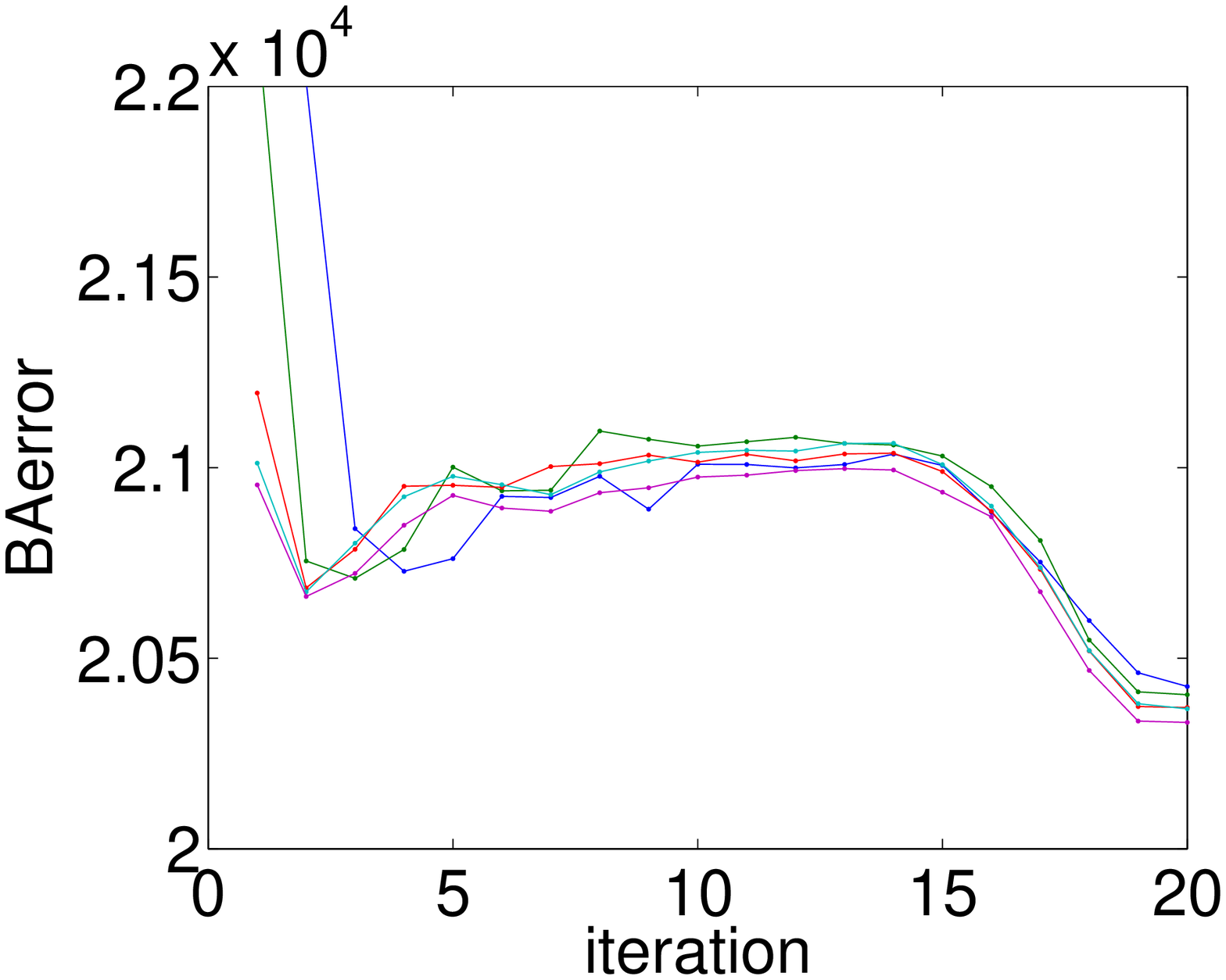} &
    \includegraphics[width=0.24\linewidth]{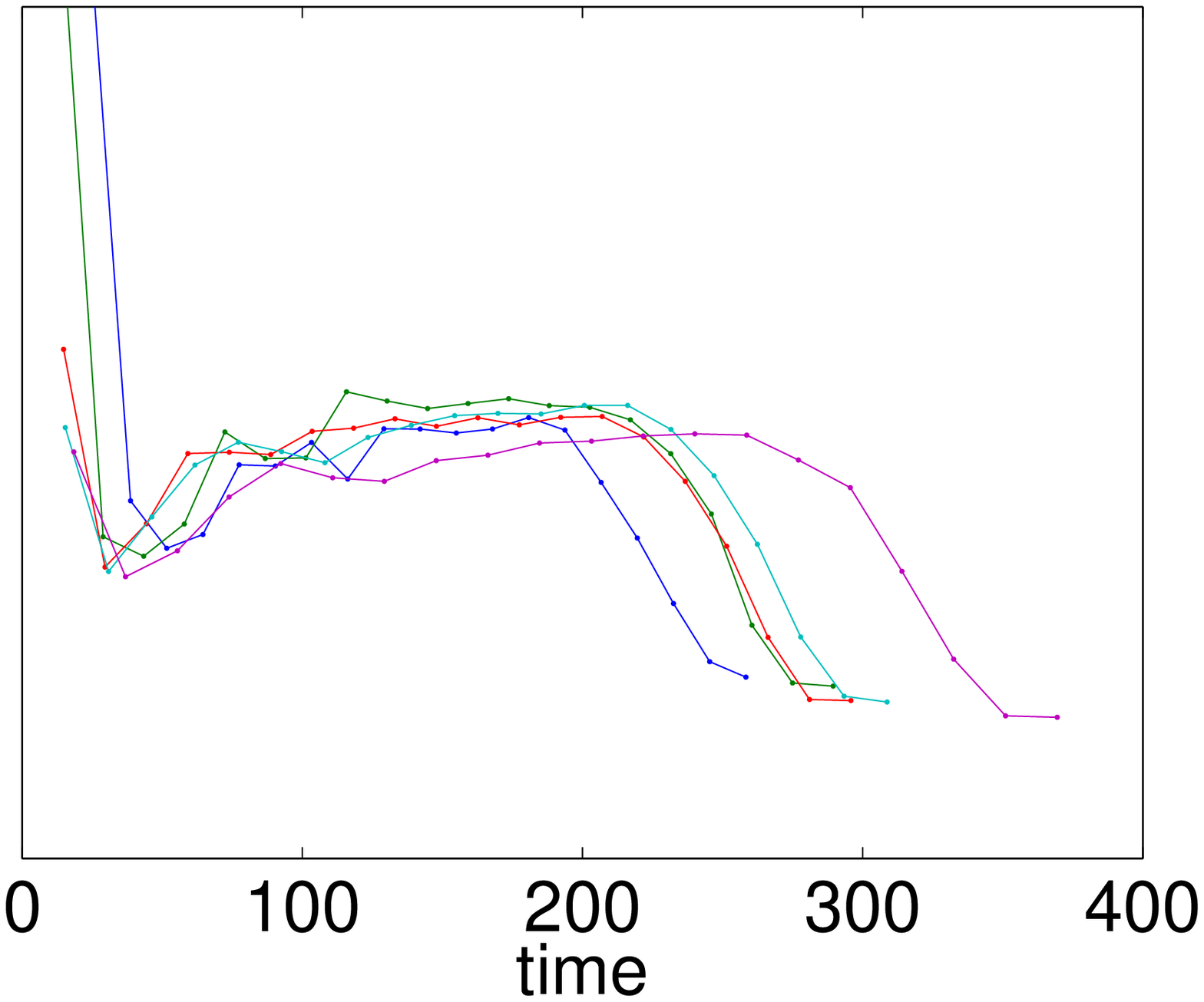} &
    \includegraphics[width=0.235\linewidth]{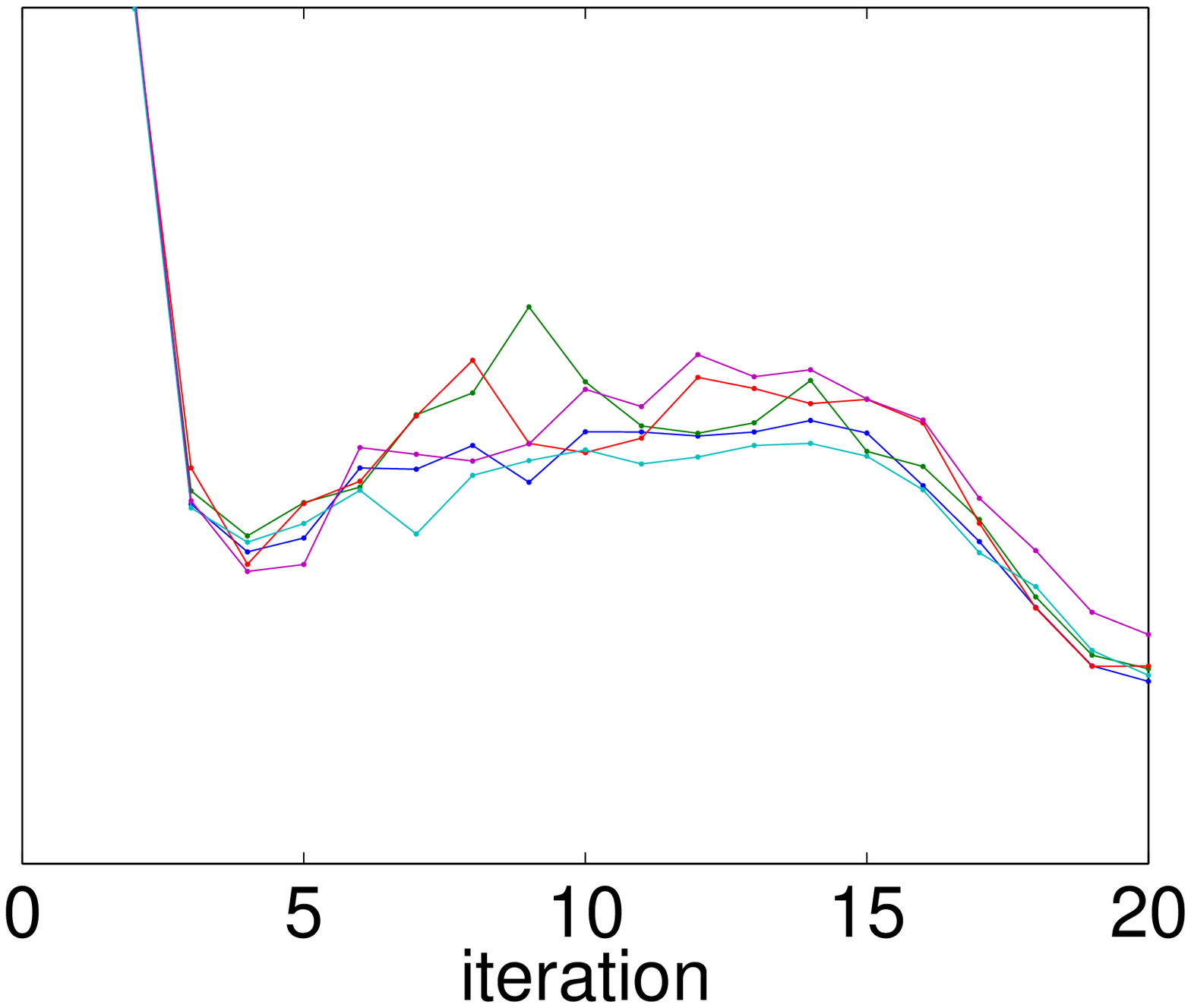} &
    \includegraphics[width=0.235\linewidth]{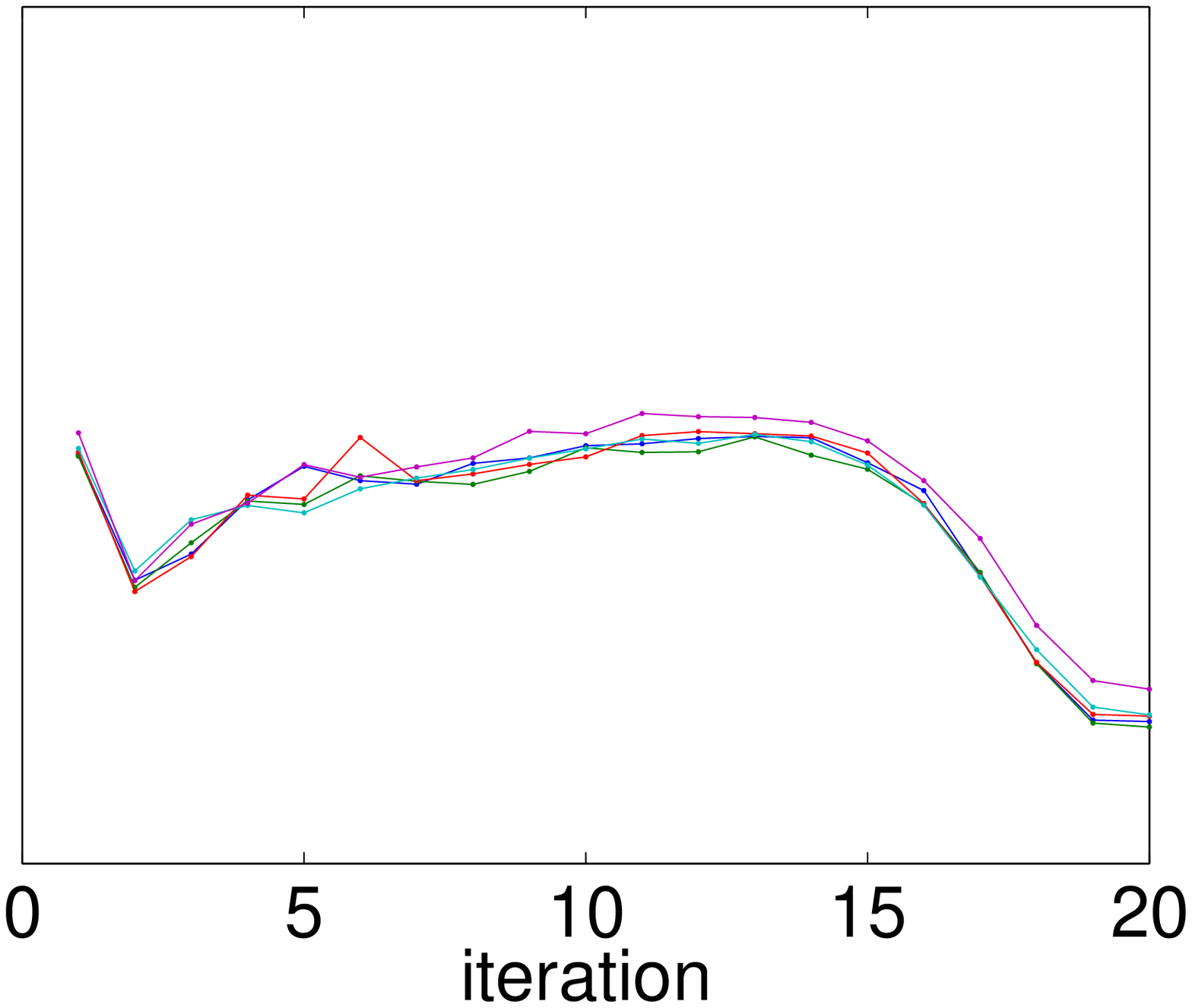} \\[-2.5ex]
    \psfrag{iteration}[t][]{iteration}
    \includegraphics[width=0.274\linewidth]{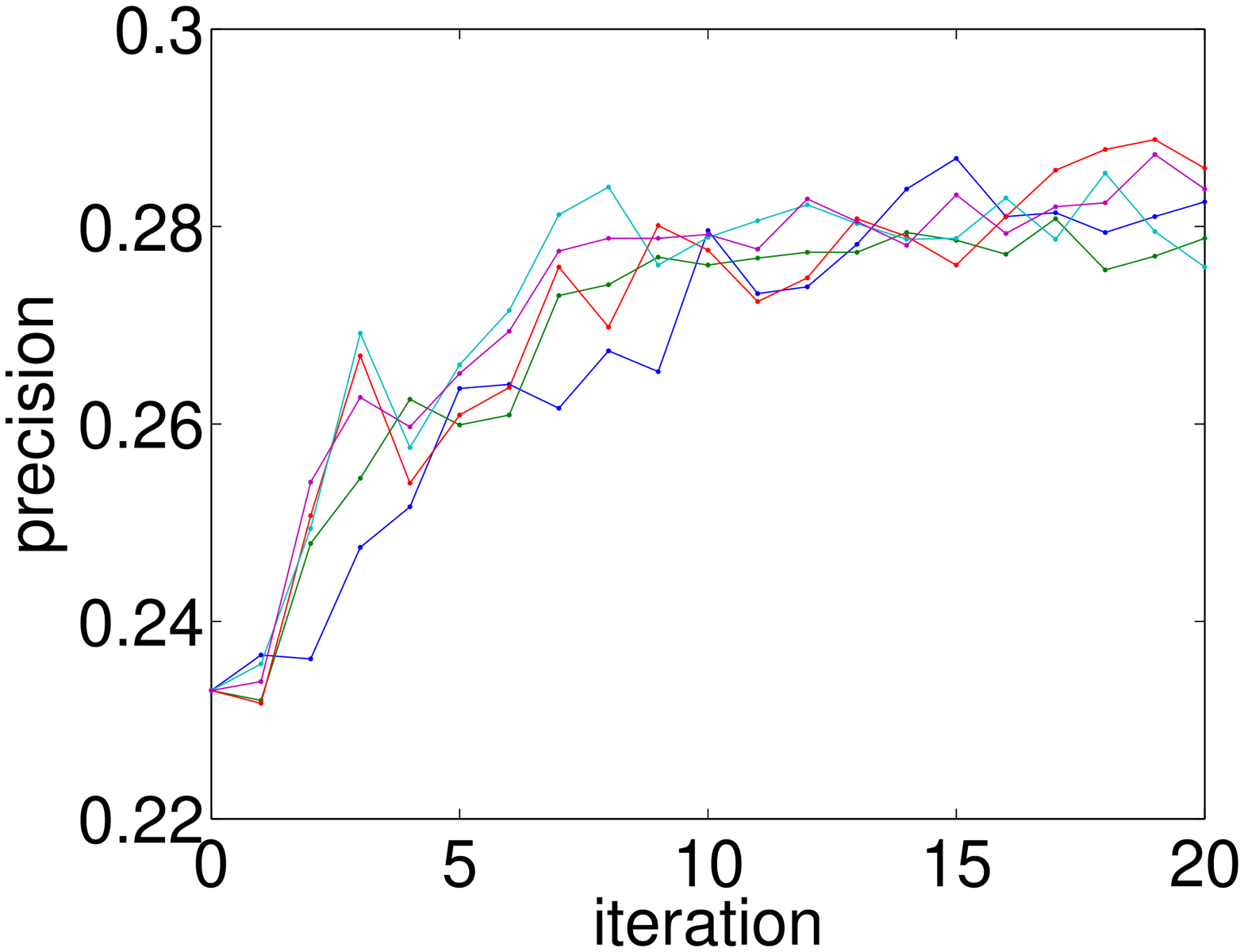} &
    \psfrag{time}[t][]{time}
    \includegraphics[width=0.24\linewidth]{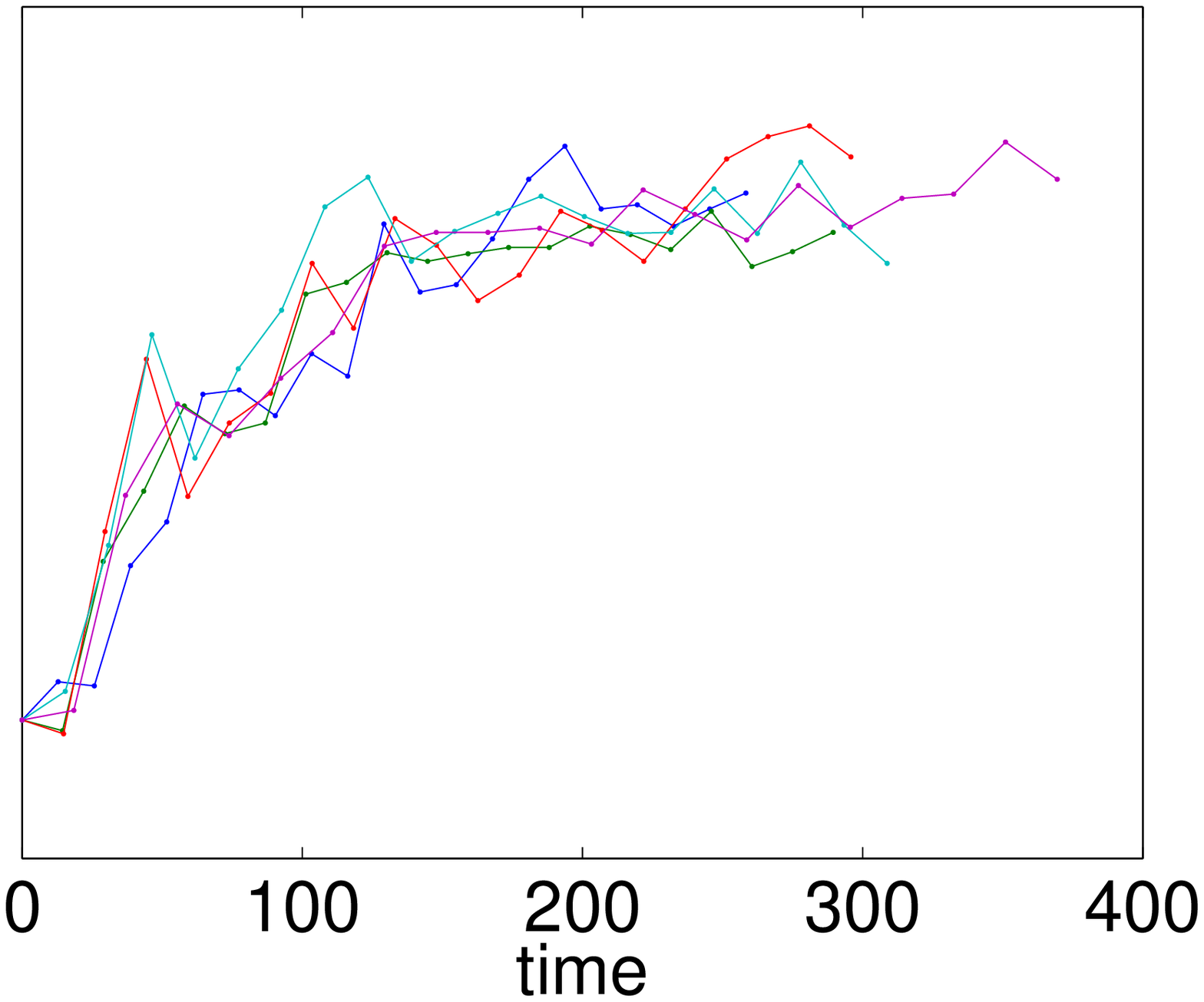} &
    \psfrag{iteration}[t][]{iteration}
    \includegraphics[width=0.235\linewidth]{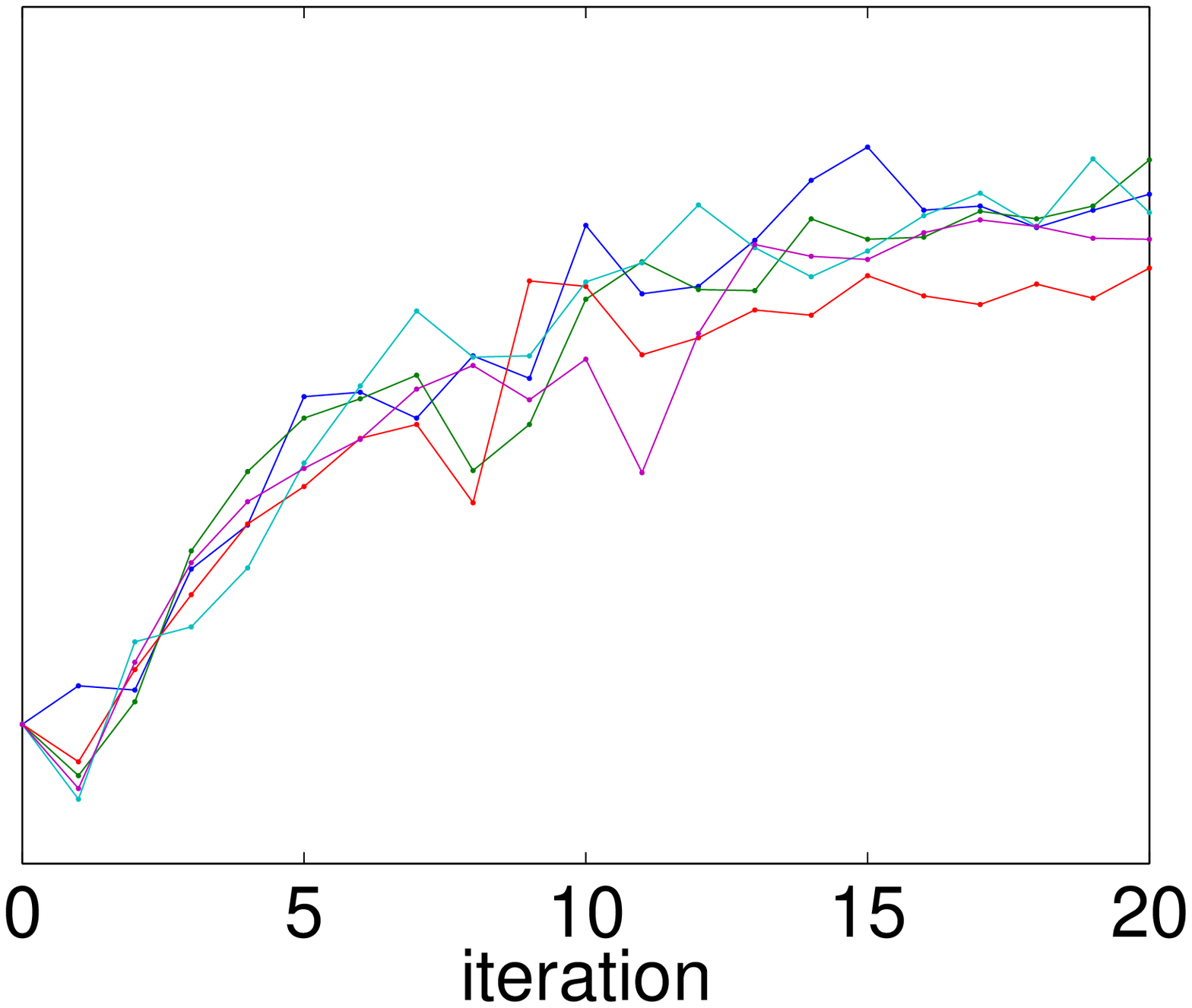} &
    \psfrag{iteration}[t][]{iteration}
    \includegraphics[width=0.235\linewidth]{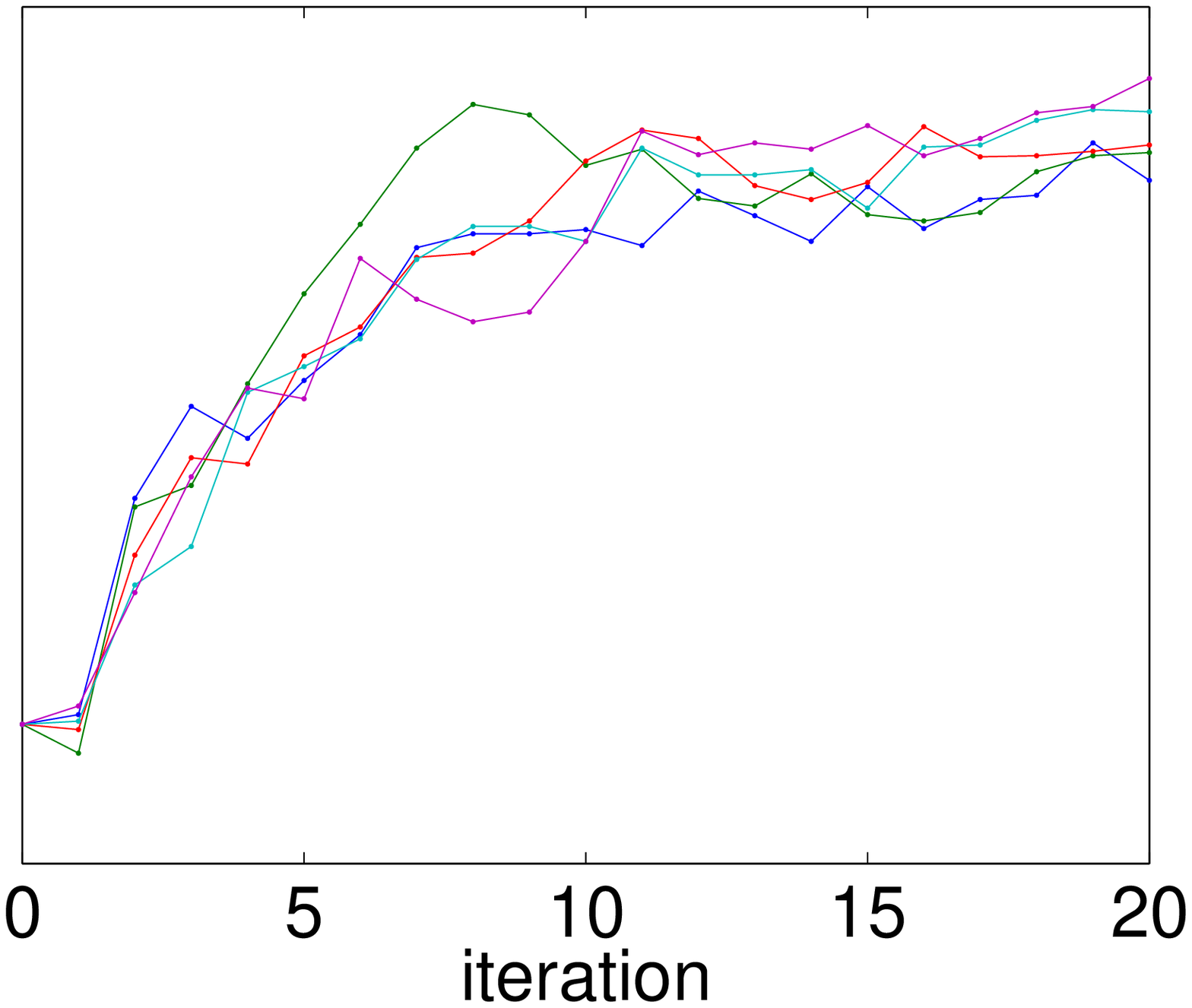}
  \end{tabular}
  \caption{SIFT-10K dataset. \emph{Left two columns}: single machine ($P = 1$) and different number of epochs $e$ in the \W\ step. \emph{Right two columns}: fixed number of epochs (either 1 or 8) but different number of machines $P$.}
  \label{f:sift10k-epochs}
\end{figure}

\begin{figure}[t]
  \psfrag{iteration}{}
  \psfrag{time}{}
  \psfrag{QPerror}[][t]{\footnotesize$E_Q$}
  \psfrag{BAerror}[][t]{\footnotesize$E_{\text{BA}}$}
  \psfrag{precision}[][t]{precision}
  \begin{tabular}{@{}c@{\hspace{0\linewidth}}c@{\hspace{0.015\linewidth}}c@{\hspace{0\linewidth}}c@{}}
    \multicolumn{2}{c}{\makebox[0.45\linewidth][c]{\dotfill Number of epochs in \W\ step\dotfill}} & \multicolumn{2}{c}{\makebox[0.45\linewidth][c]{\dotfill Number of machines $P$\dotfill}} \\
    error-iteration view & error-time view & 2 epochs in \W\ step & 8 epochs in \W\ step \\[-1ex]
    \includegraphics[width=0.274\linewidth]{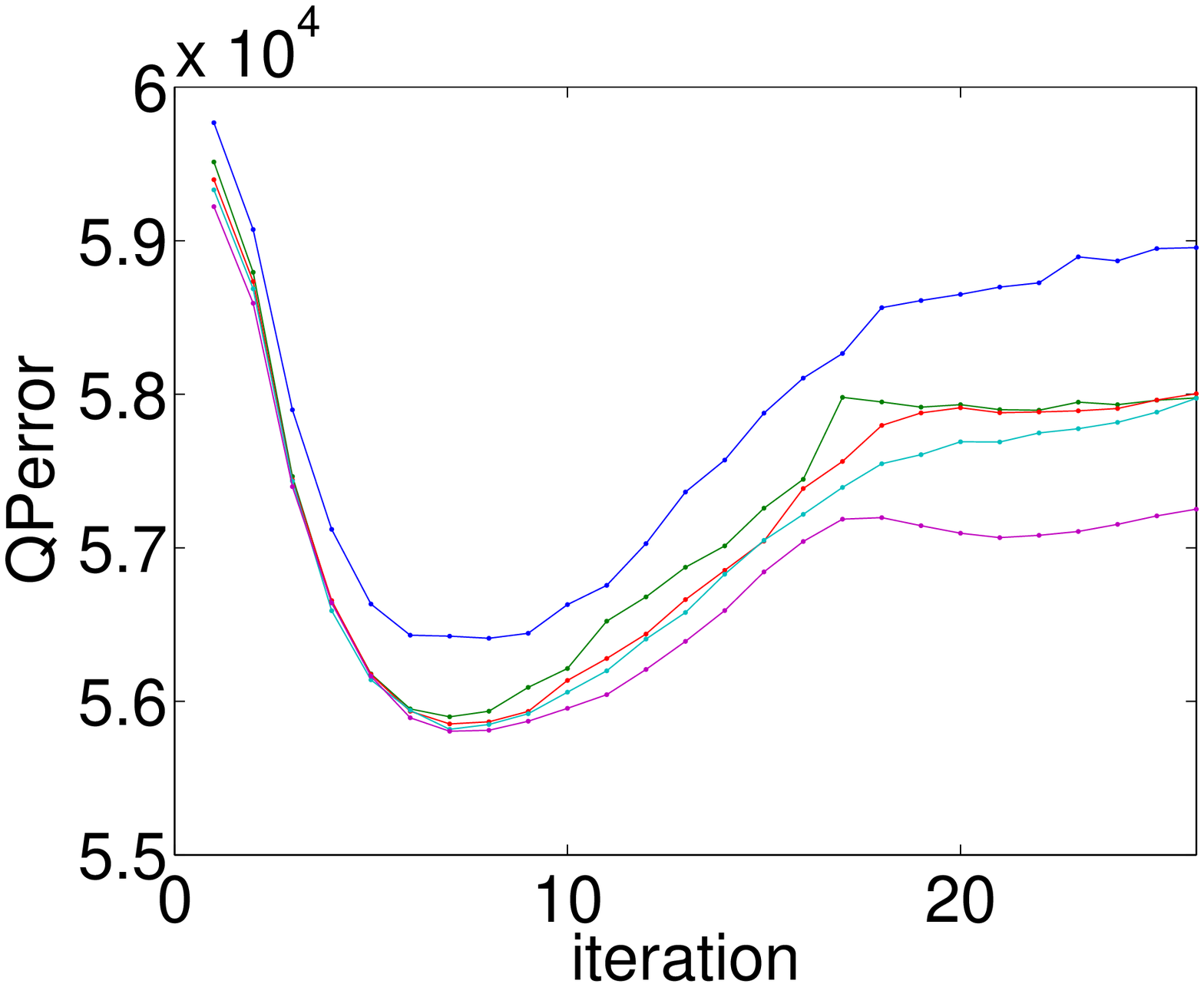} &
    \includegraphics[width=0.235\linewidth]{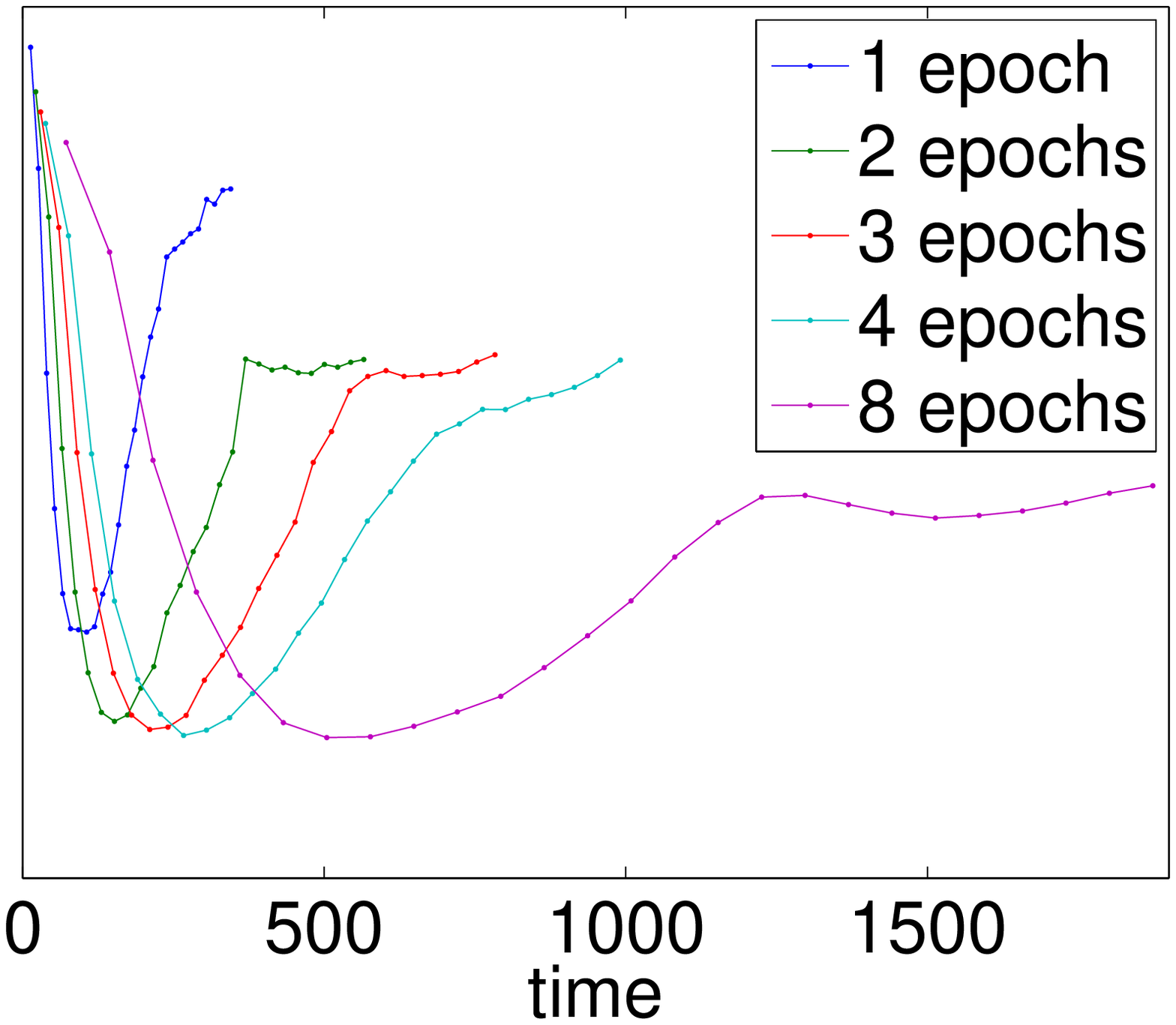} &
    \includegraphics[width=0.235\linewidth]{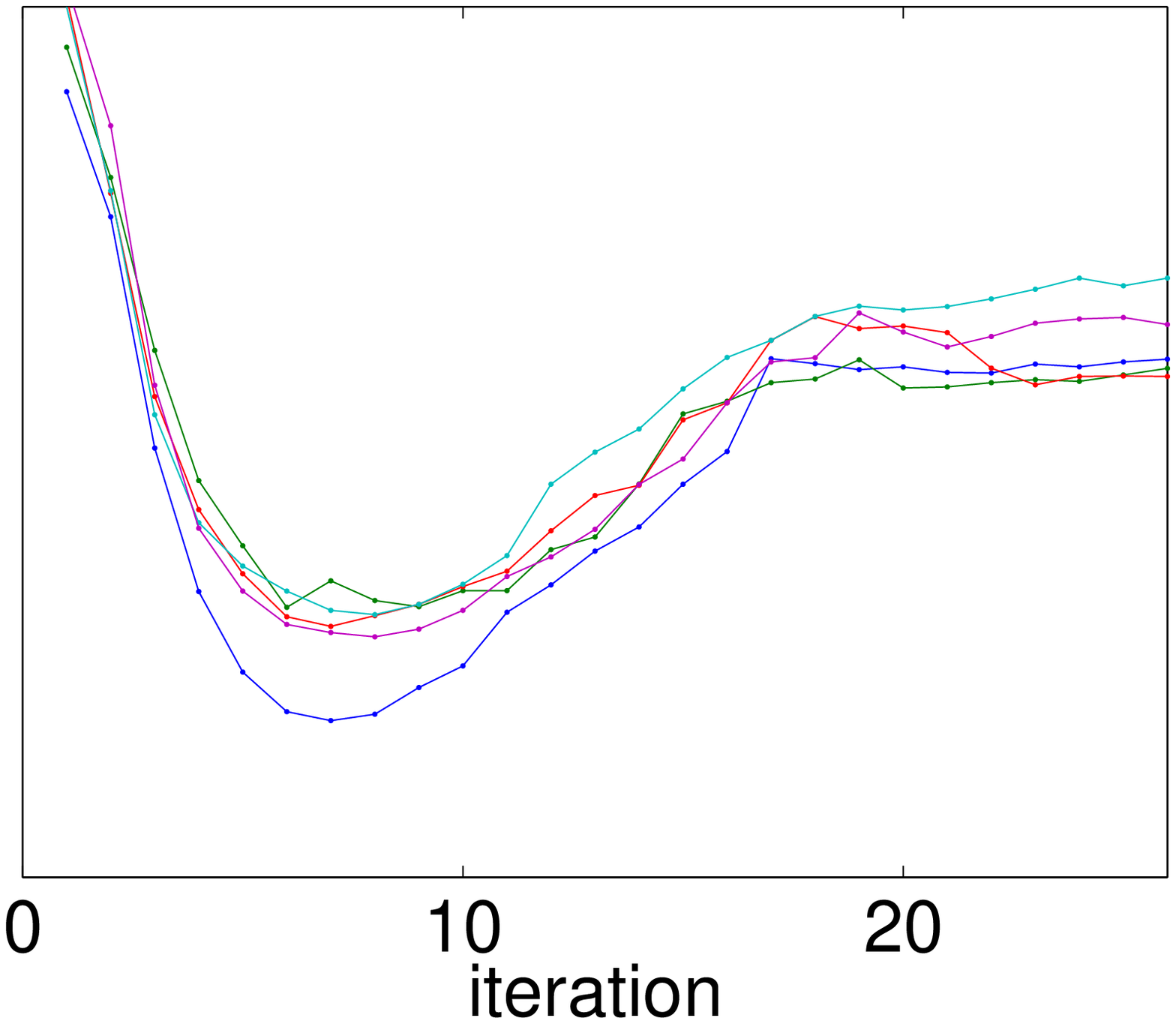} &
    \includegraphics[width=0.235\linewidth]{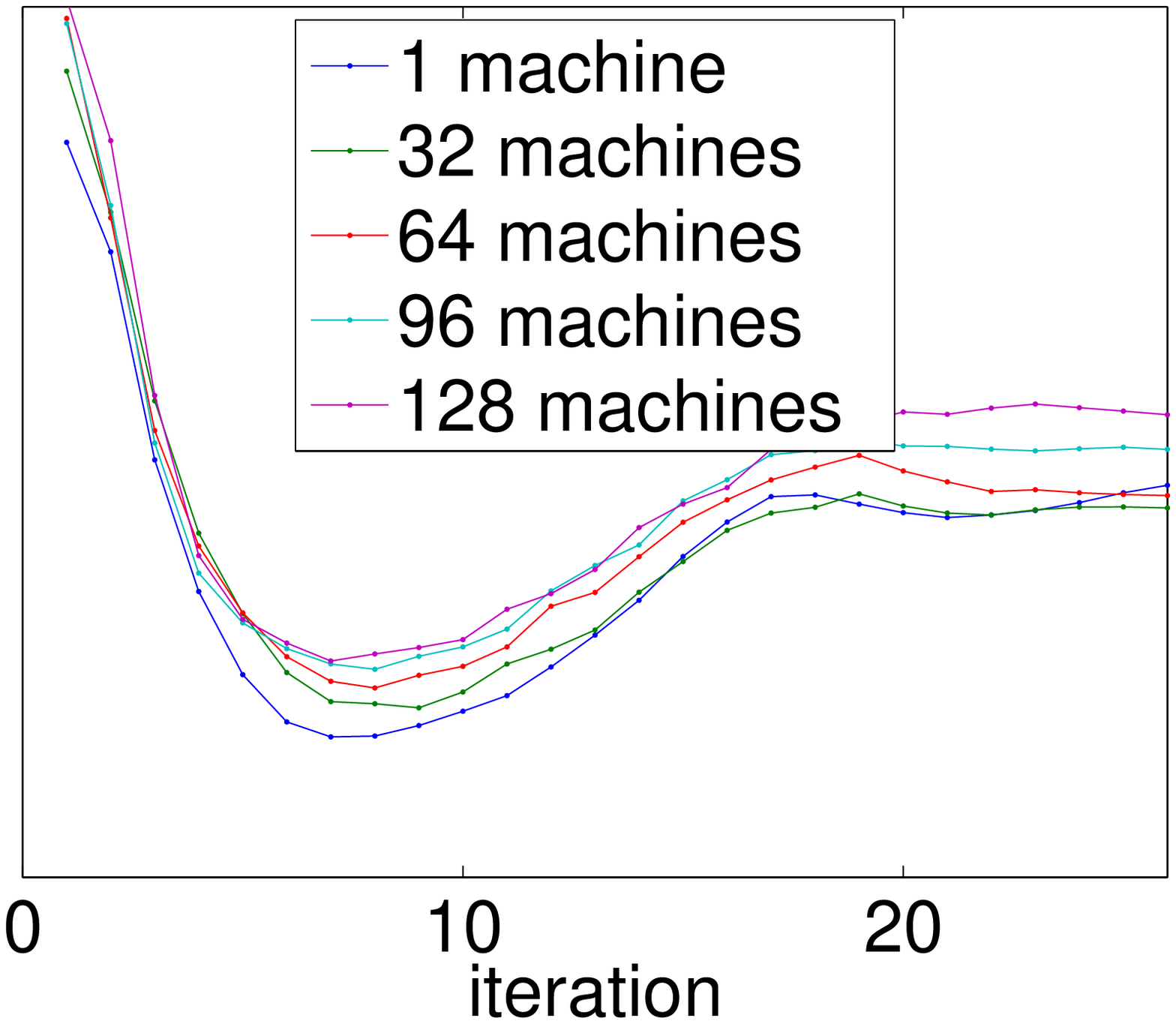} \\[-2.5ex]
    \includegraphics[width=0.274\linewidth]{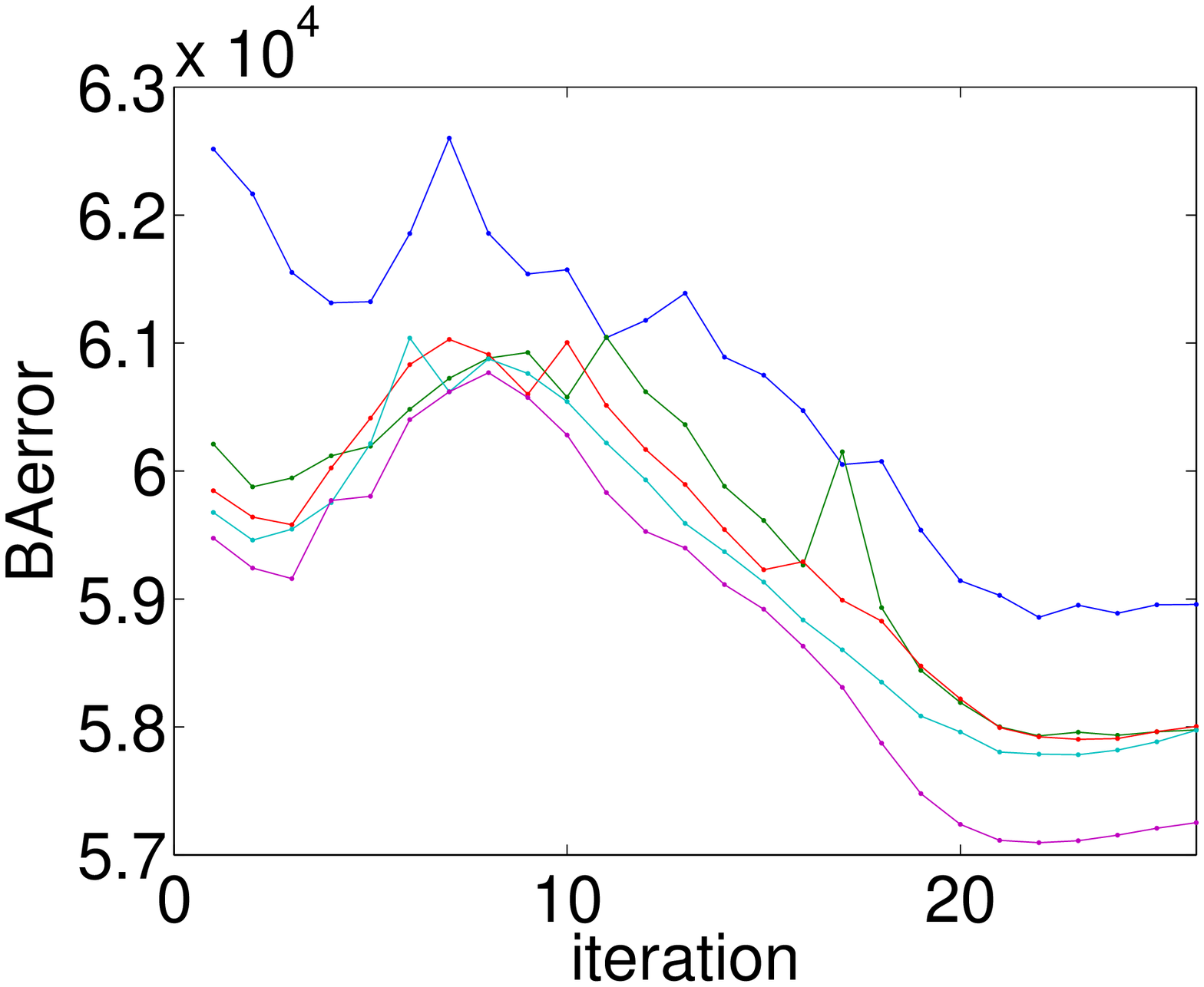} &
    \includegraphics[width=0.235\linewidth]{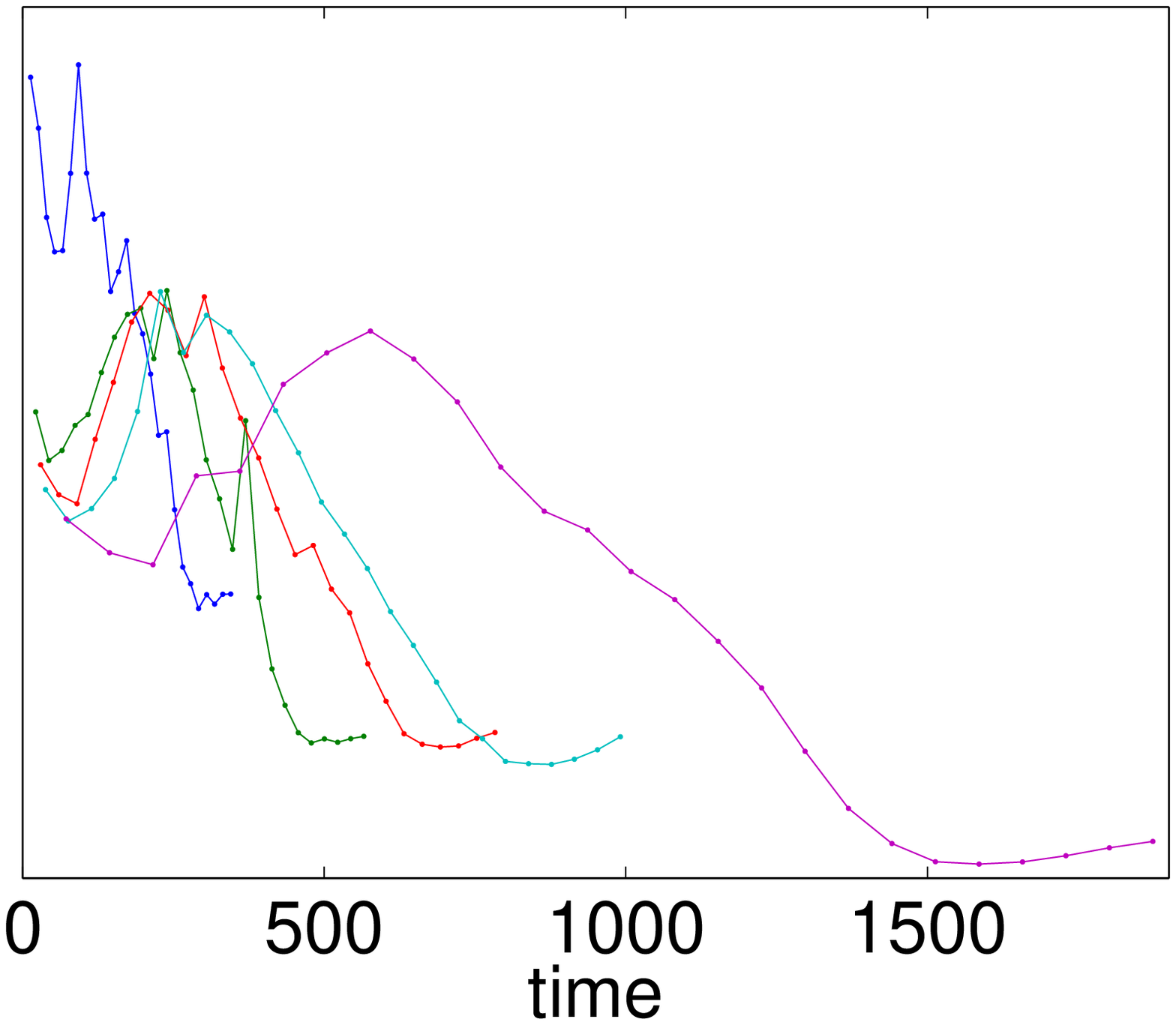} &
    \includegraphics[width=0.235\linewidth]{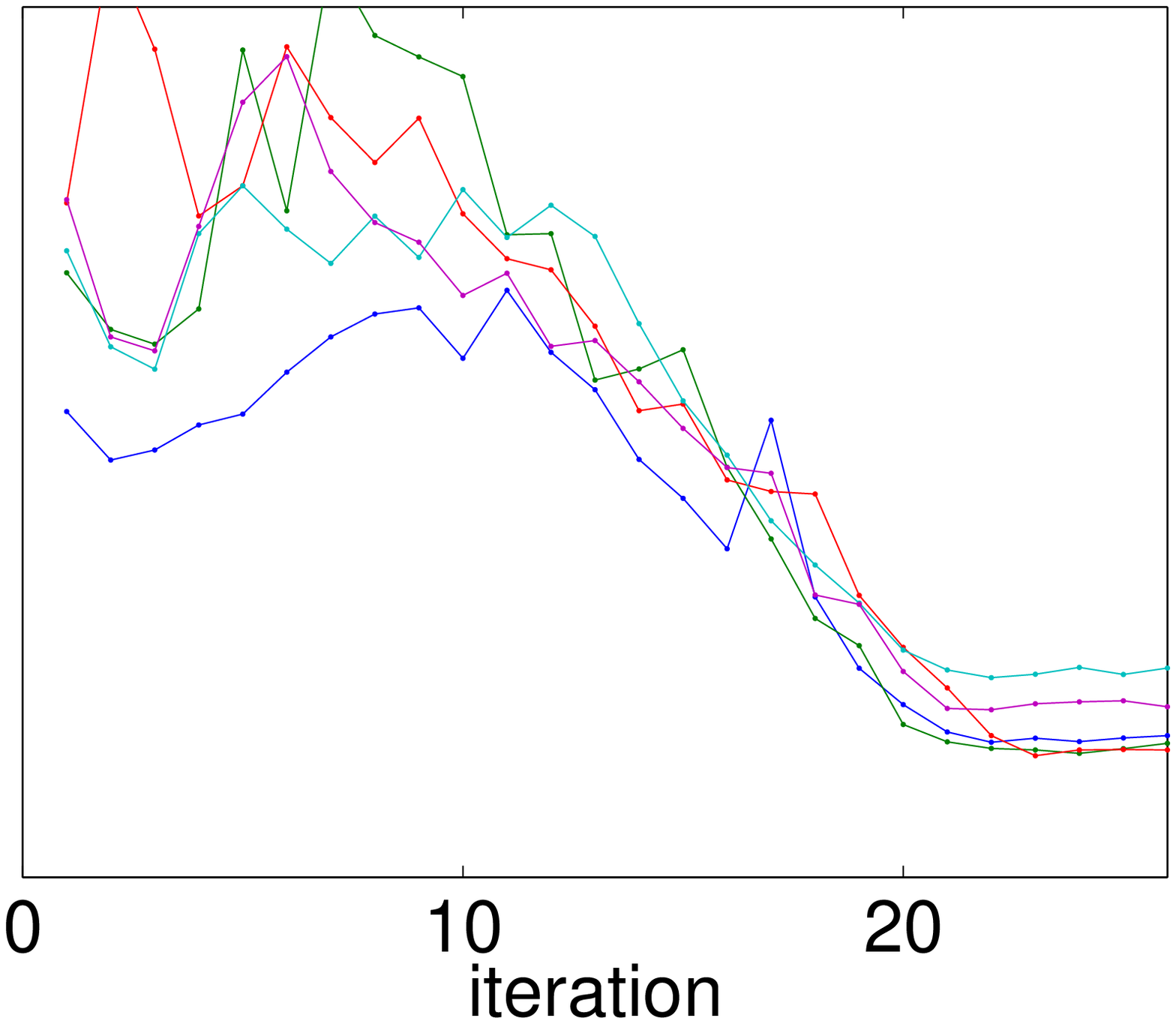} &
    \includegraphics[width=0.235\linewidth]{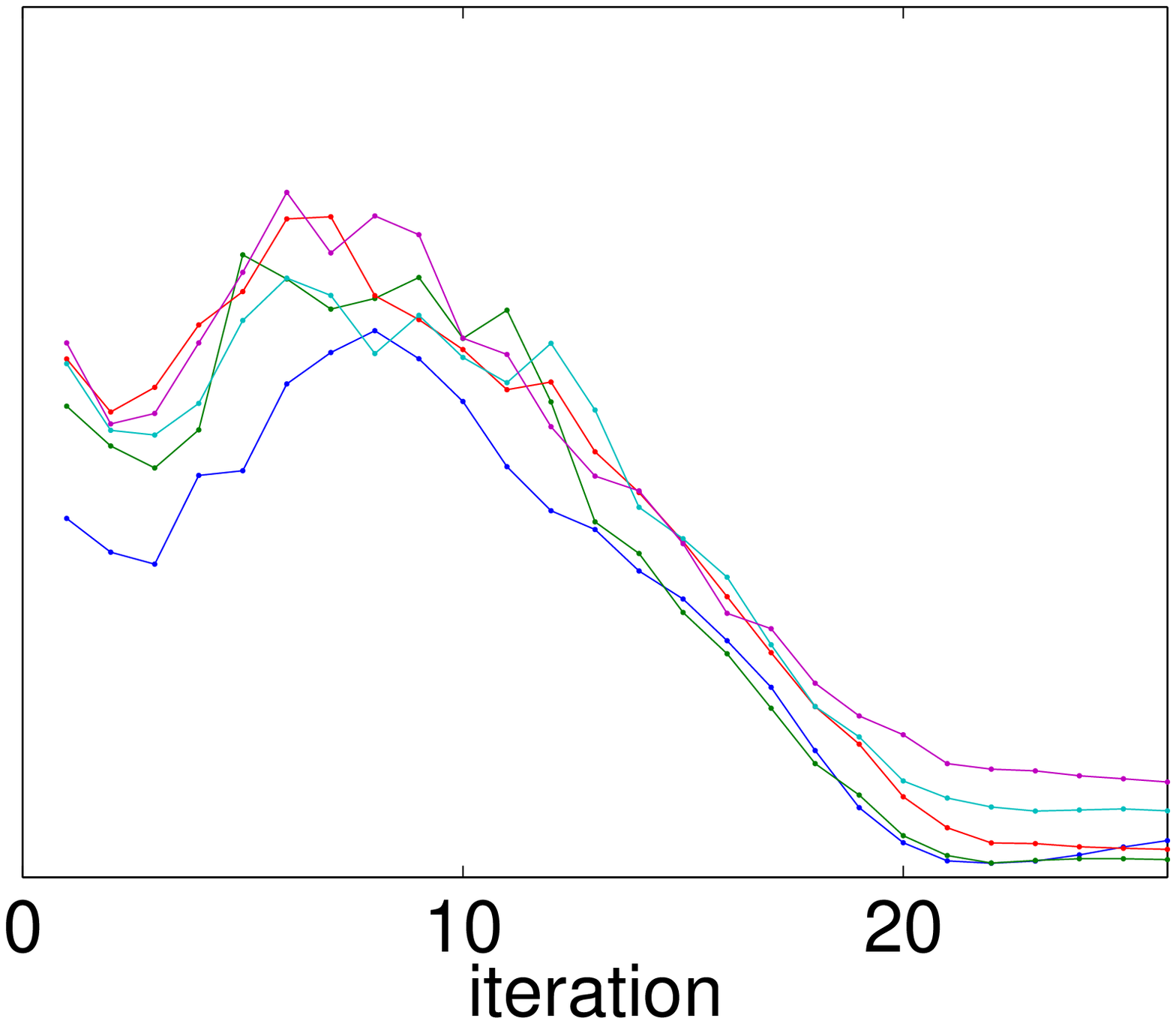} \\[-2.5ex]
    \psfrag{iteration}[t][]{iteration}
    \includegraphics[width=0.274\linewidth]{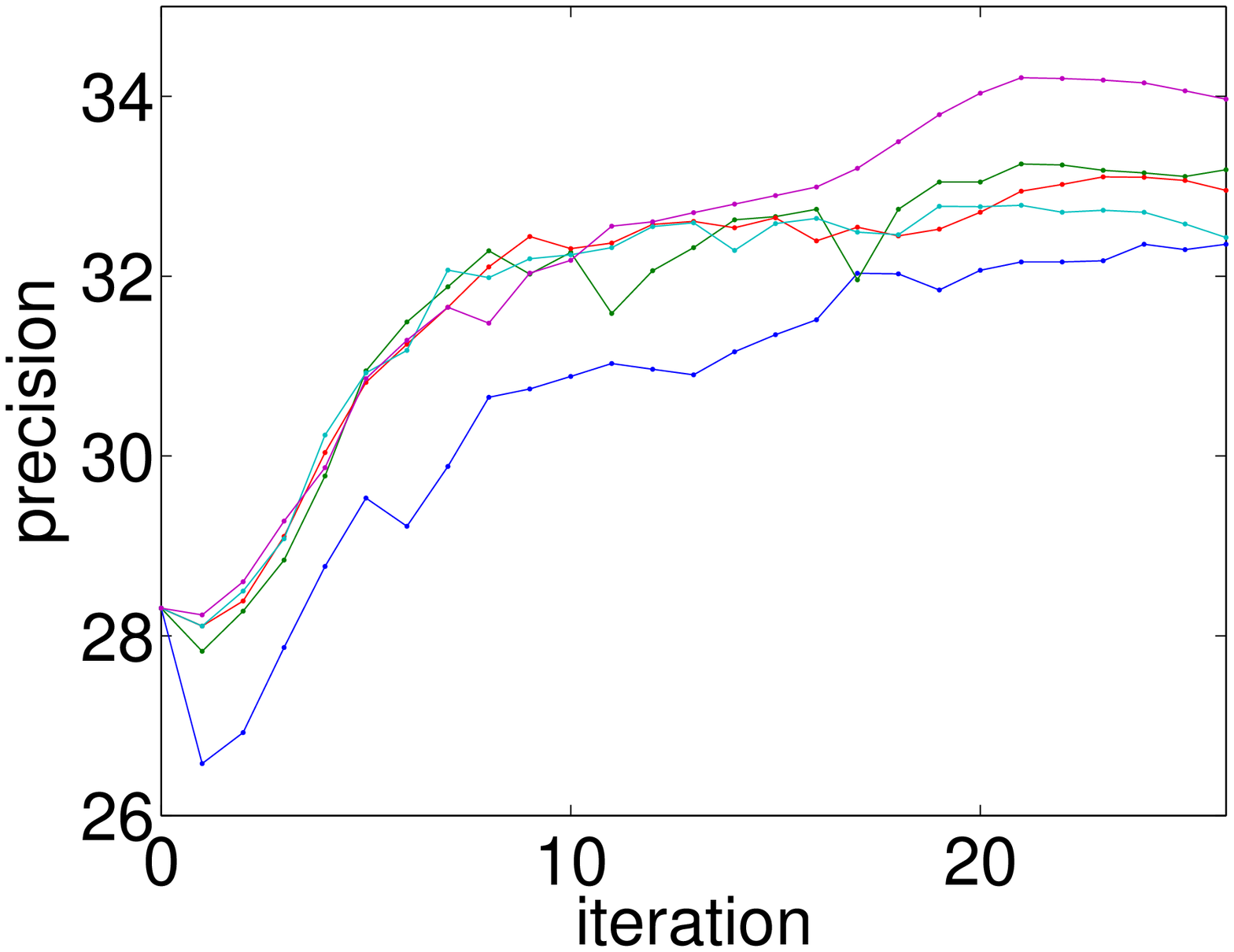} &
    \psfrag{time}[t][]{time}
    \includegraphics[width=0.235\linewidth]{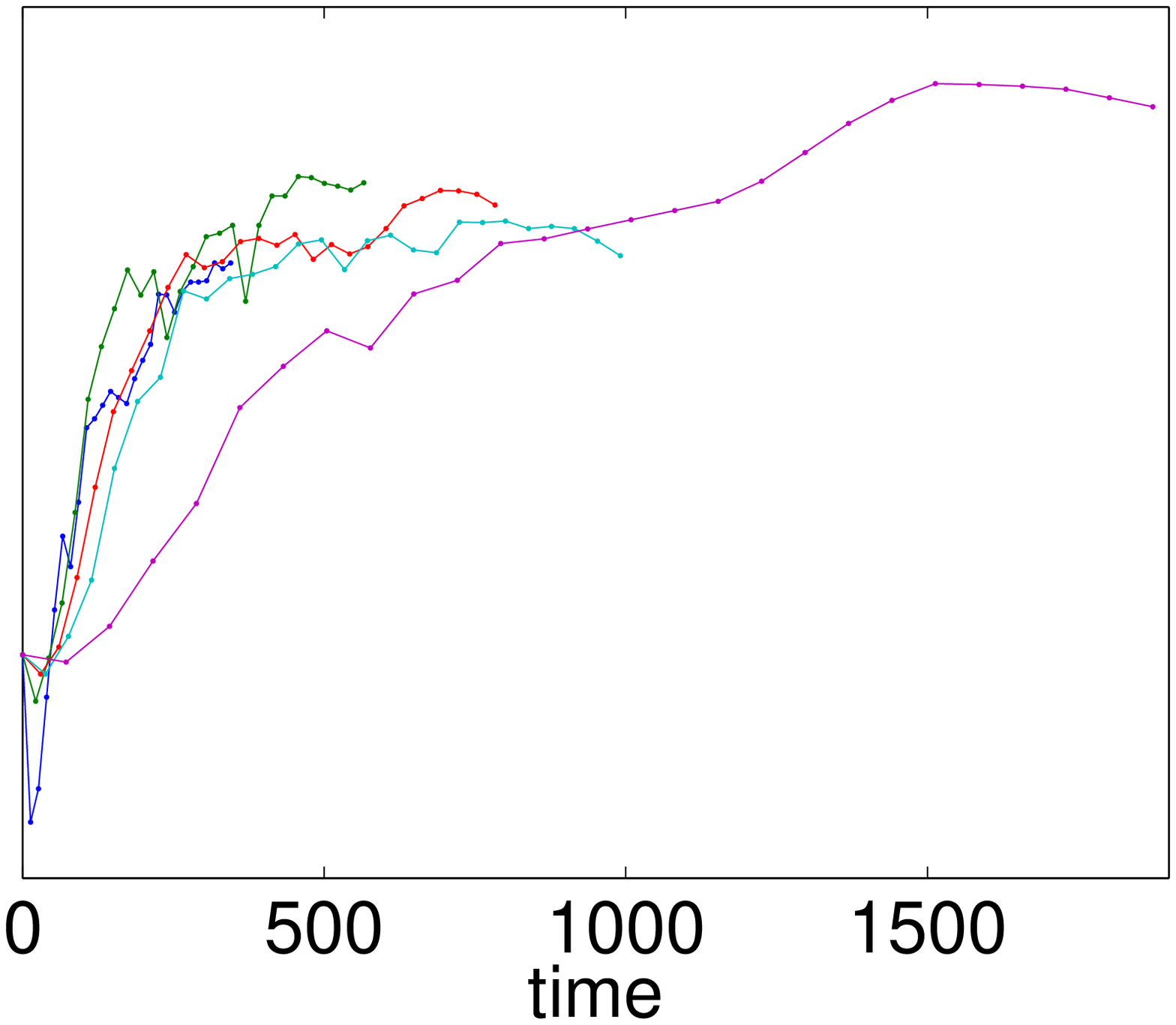} &
    \psfrag{iteration}[t][]{iteration}
    \includegraphics[width=0.235\linewidth]{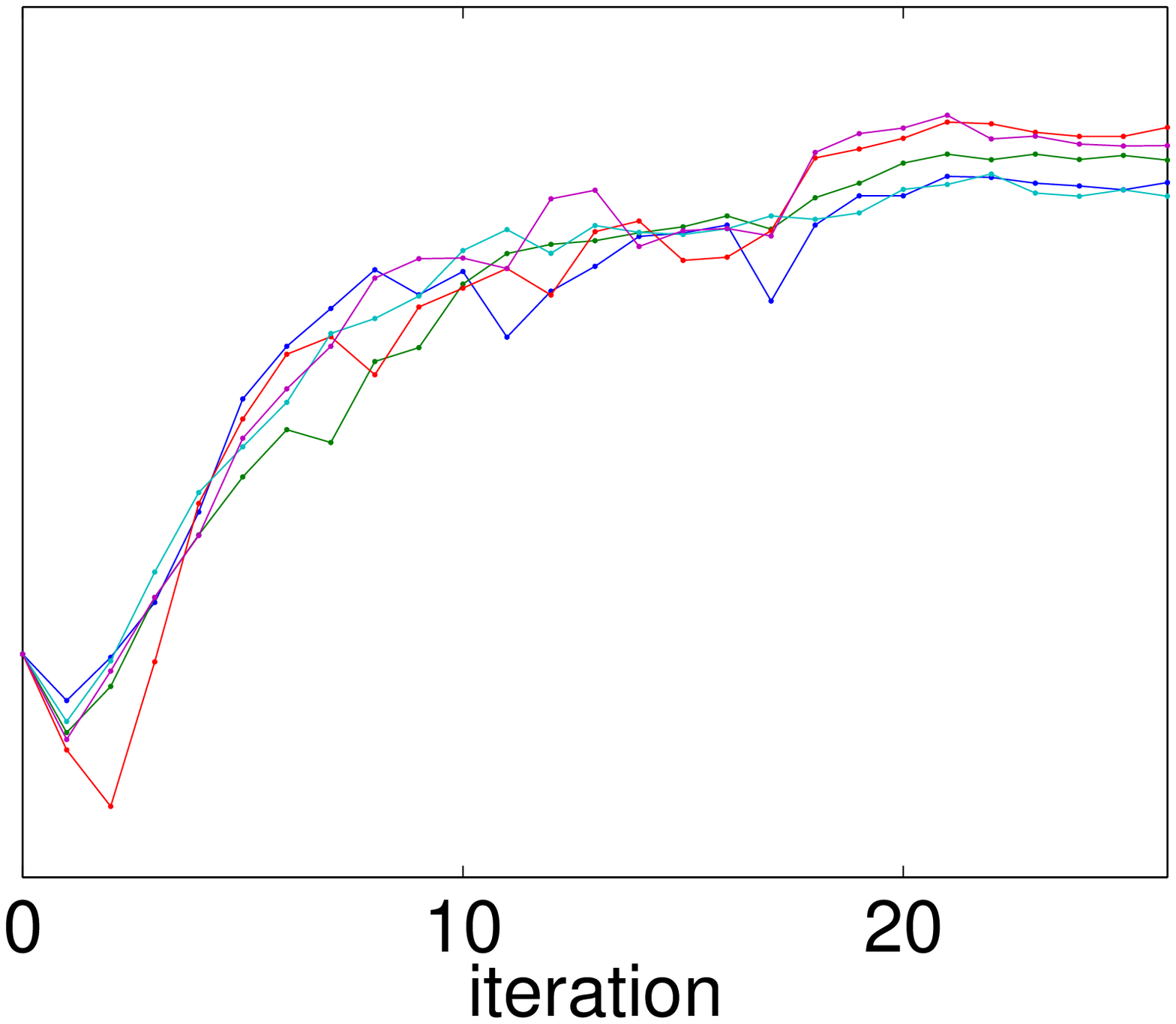} &
    \psfrag{iteration}[t][]{iteration}
    \includegraphics[width=0.235\linewidth]{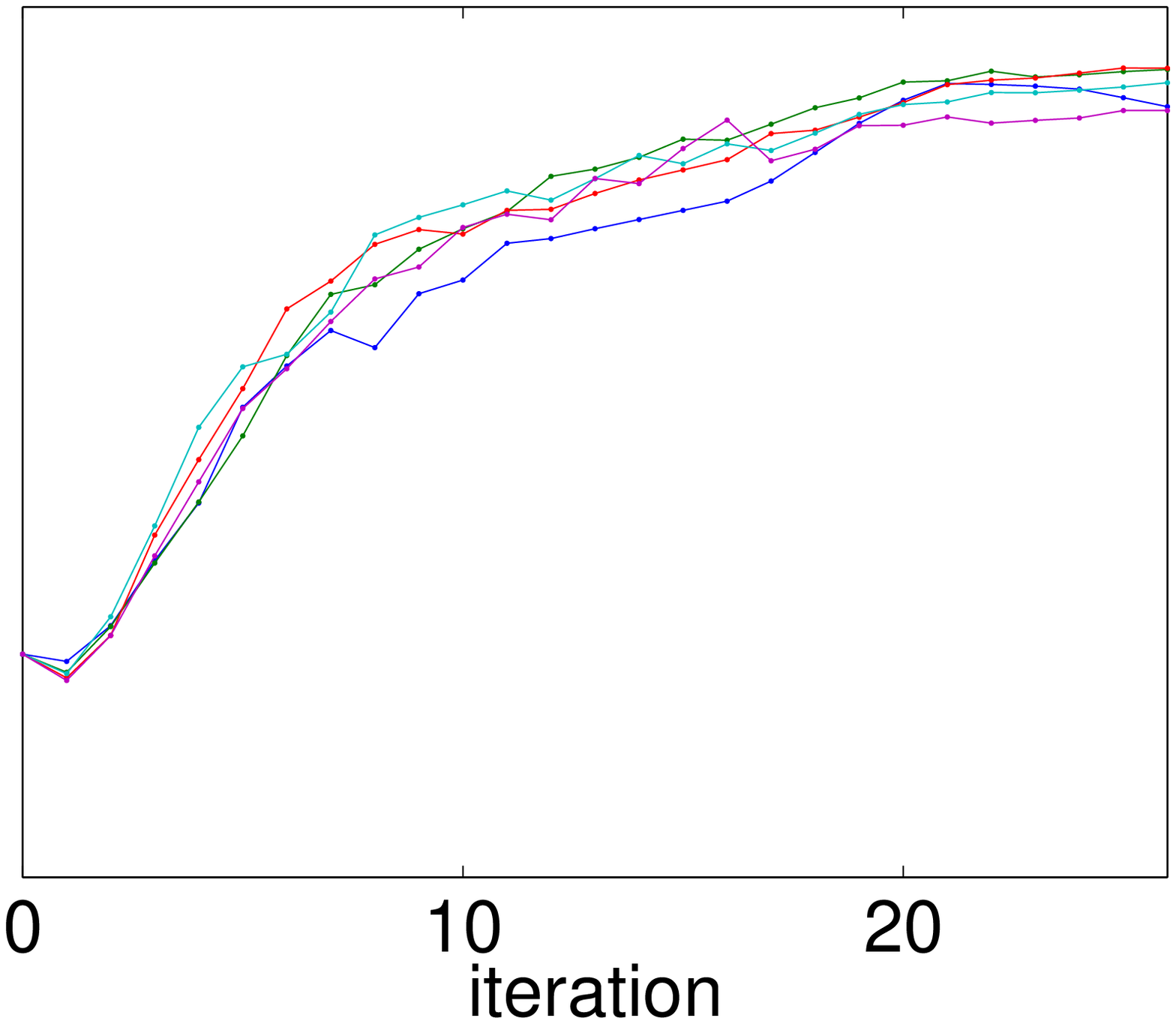}
  \end{tabular}
  \caption{CIFAR dataset. \emph{Left two columns}: single machine ($P = 1$) and different number of epochs $e$ in the \W\ step. \emph{Right two columns}: fixed number of epochs (either 2 or 8) but different number of machines $P$.}
  \label{f:cifar-epochs}
\end{figure}

\begin{figure}[t]
  \psfrag{iteration}{}
  \psfrag{time}{}
  \psfrag{QPerror}[][t]{\footnotesize$E_Q$}
  \psfrag{BAerror}[][t]{\footnotesize$E_{\text{BA}}$}
  \psfrag{precision}[][t]{precision}
  \begin{tabular}{@{}c@{\hspace{0\linewidth}}c@{\hspace{0.015\linewidth}}c@{\hspace{0\linewidth}}c@{}}
    \multicolumn{2}{c}{\makebox[0.45\linewidth][c]{\dotfill Number of epochs in \W\ step\dotfill}} & \multicolumn{2}{c}{\makebox[0.45\linewidth][c]{\dotfill Number of machines $P$\dotfill}} \\
    error-iteration view & error-time view & 2 epochs in \W\ step & 8 epochs in \W\ step \\[-1ex]
    \includegraphics[width=0.274\linewidth]{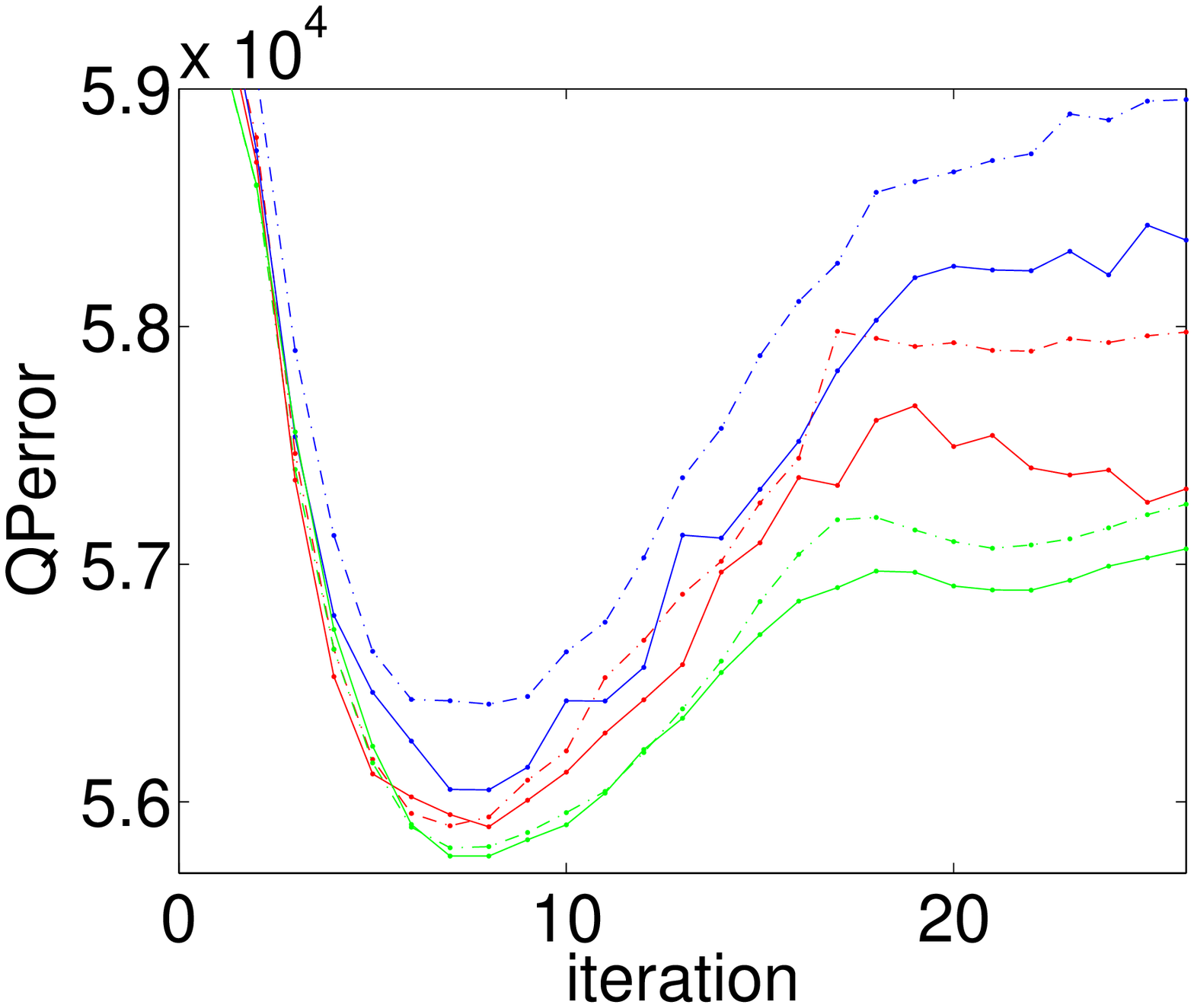} &
    \includegraphics[width=0.235\linewidth]{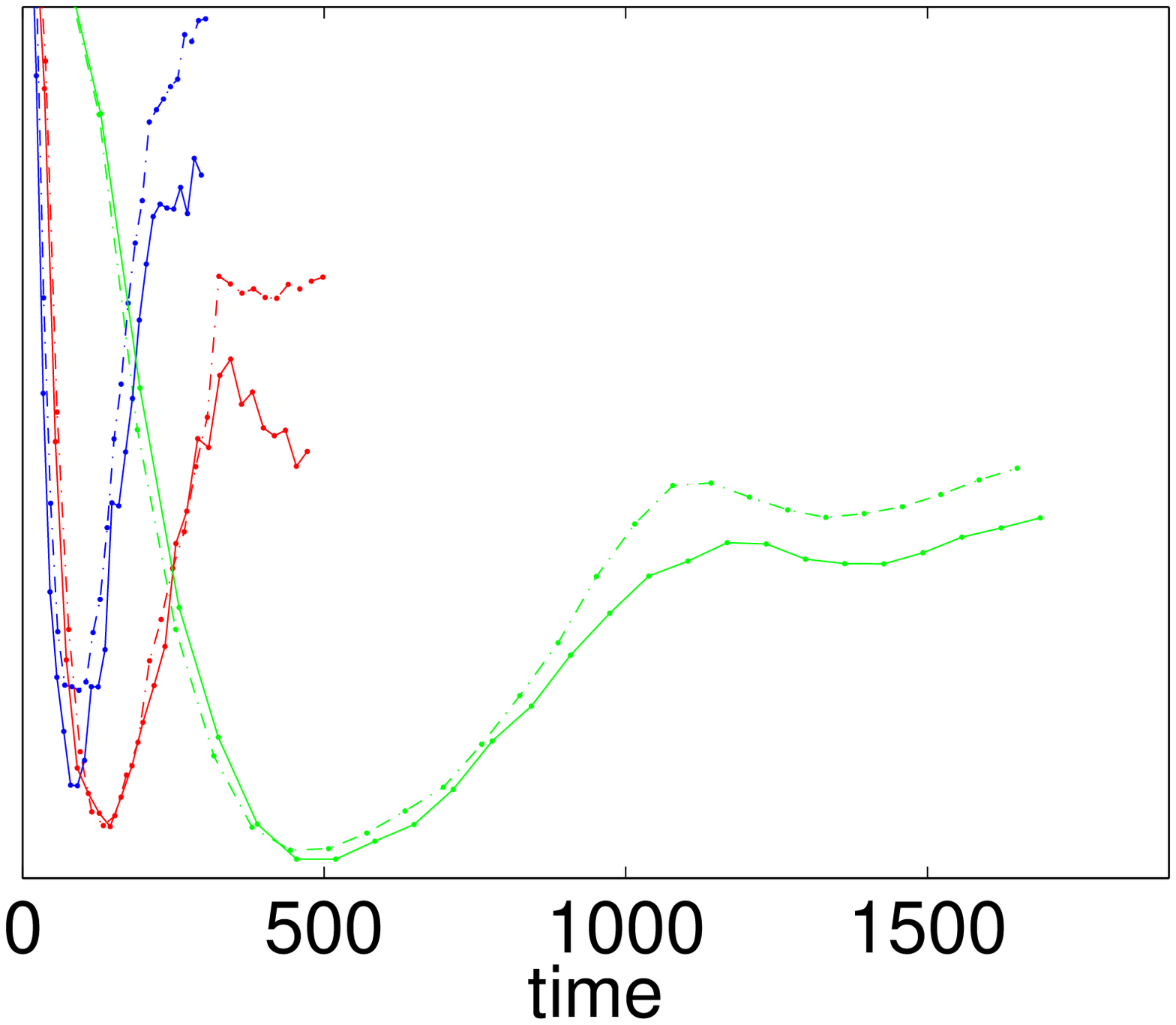} &
    \includegraphics[width=0.235\linewidth]{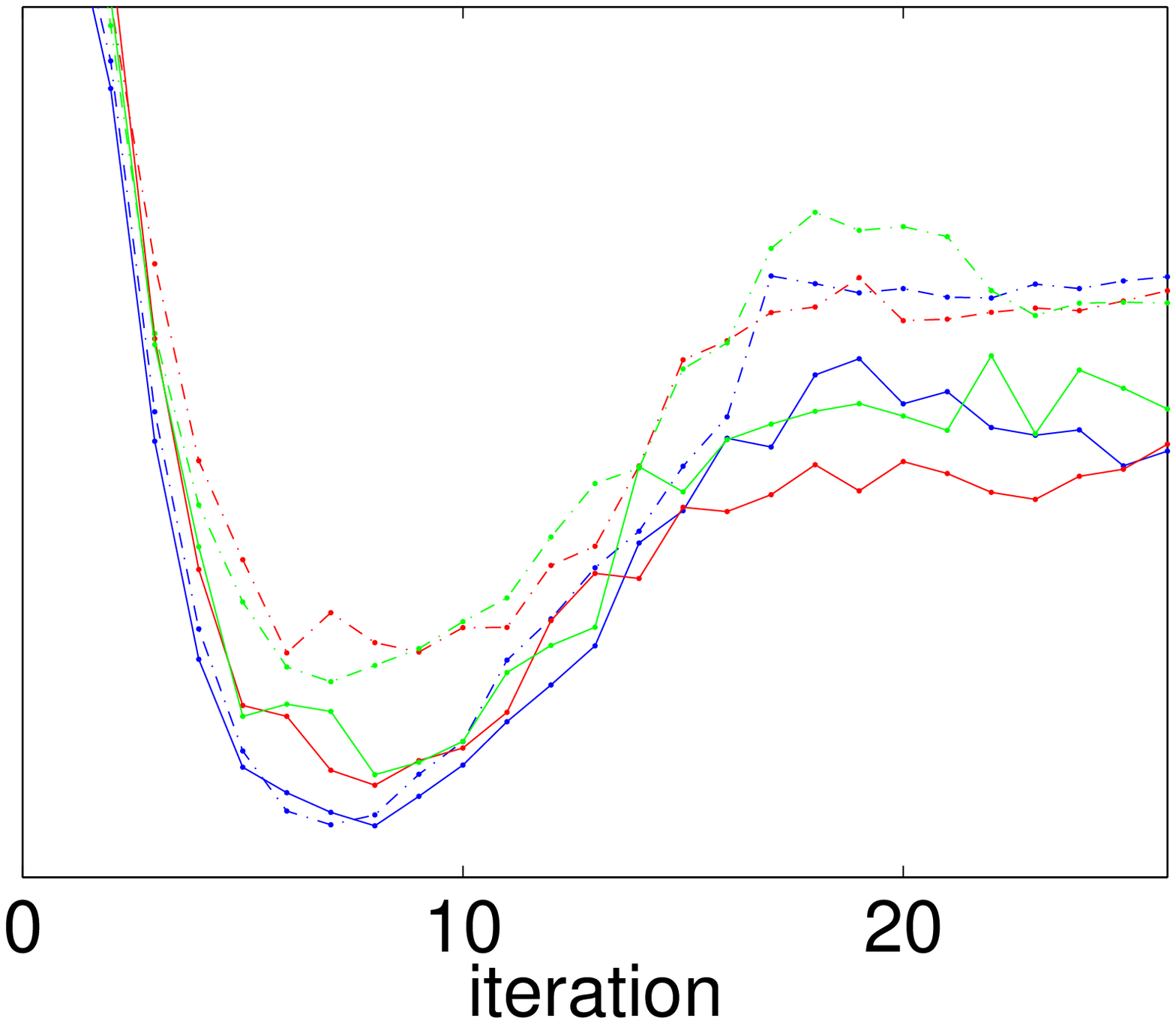} &
    \includegraphics[width=0.235\linewidth]{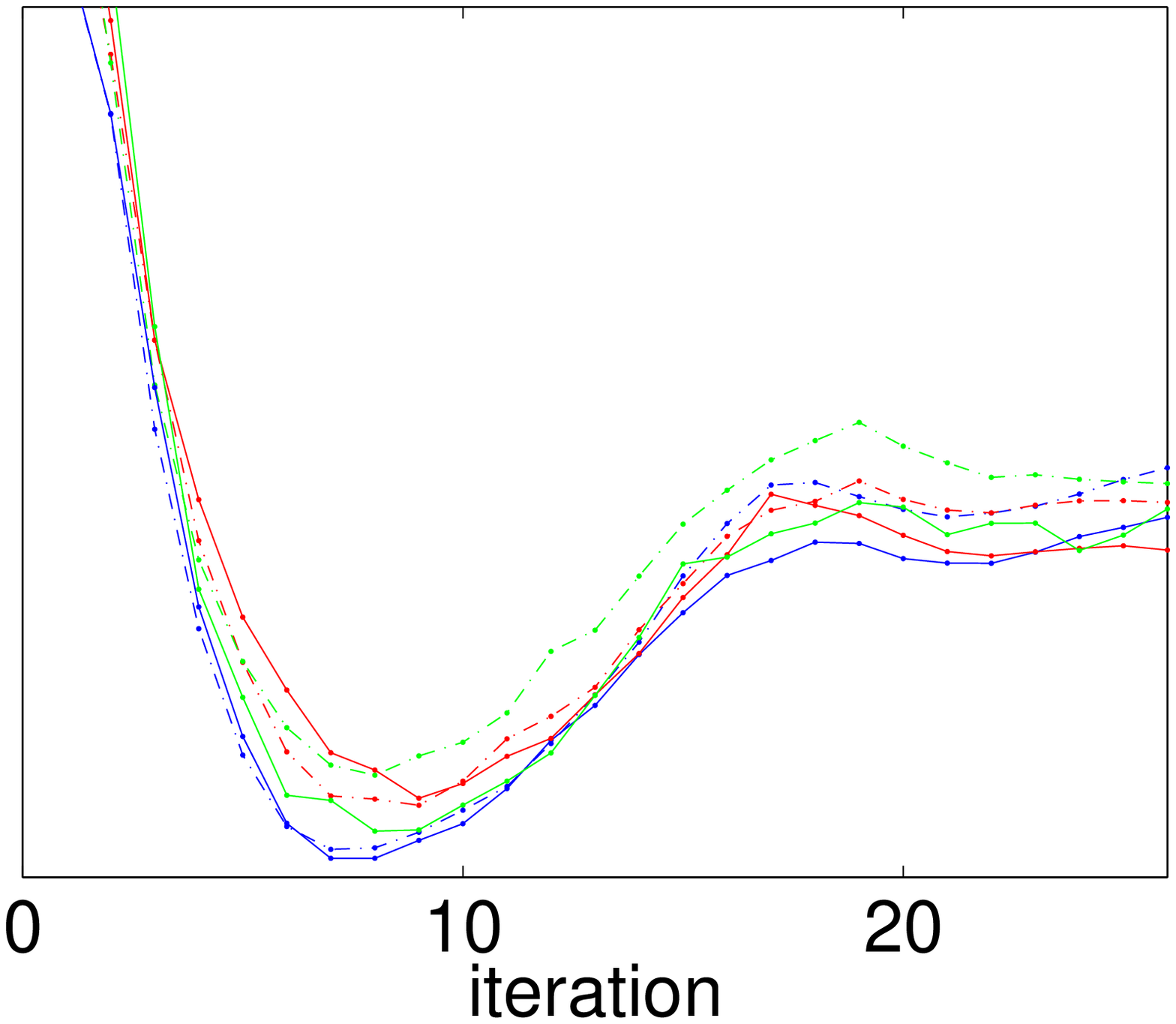} \\[-2.5ex]
    \includegraphics[width=0.274\linewidth]{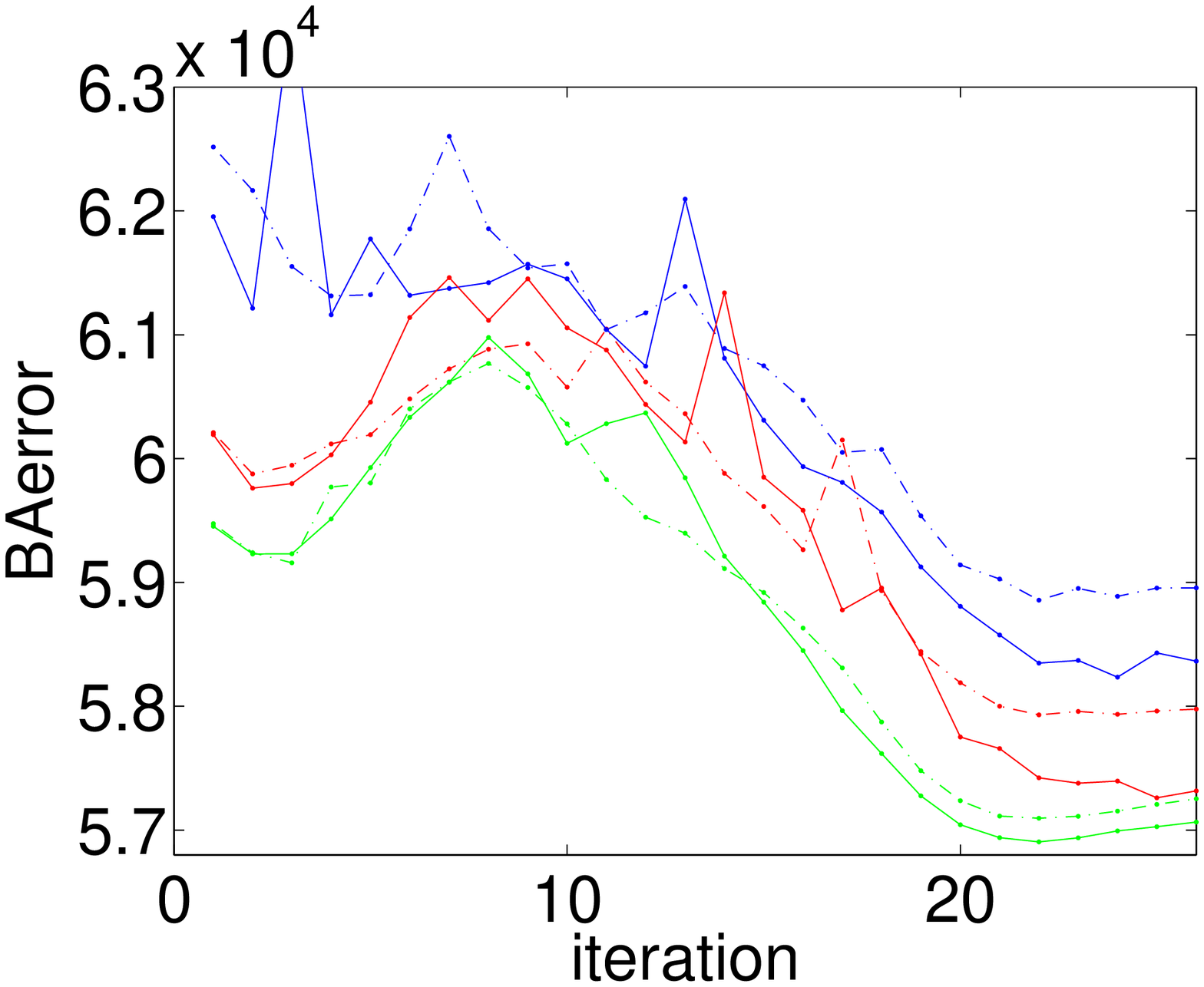} &
    \includegraphics[width=0.235\linewidth]{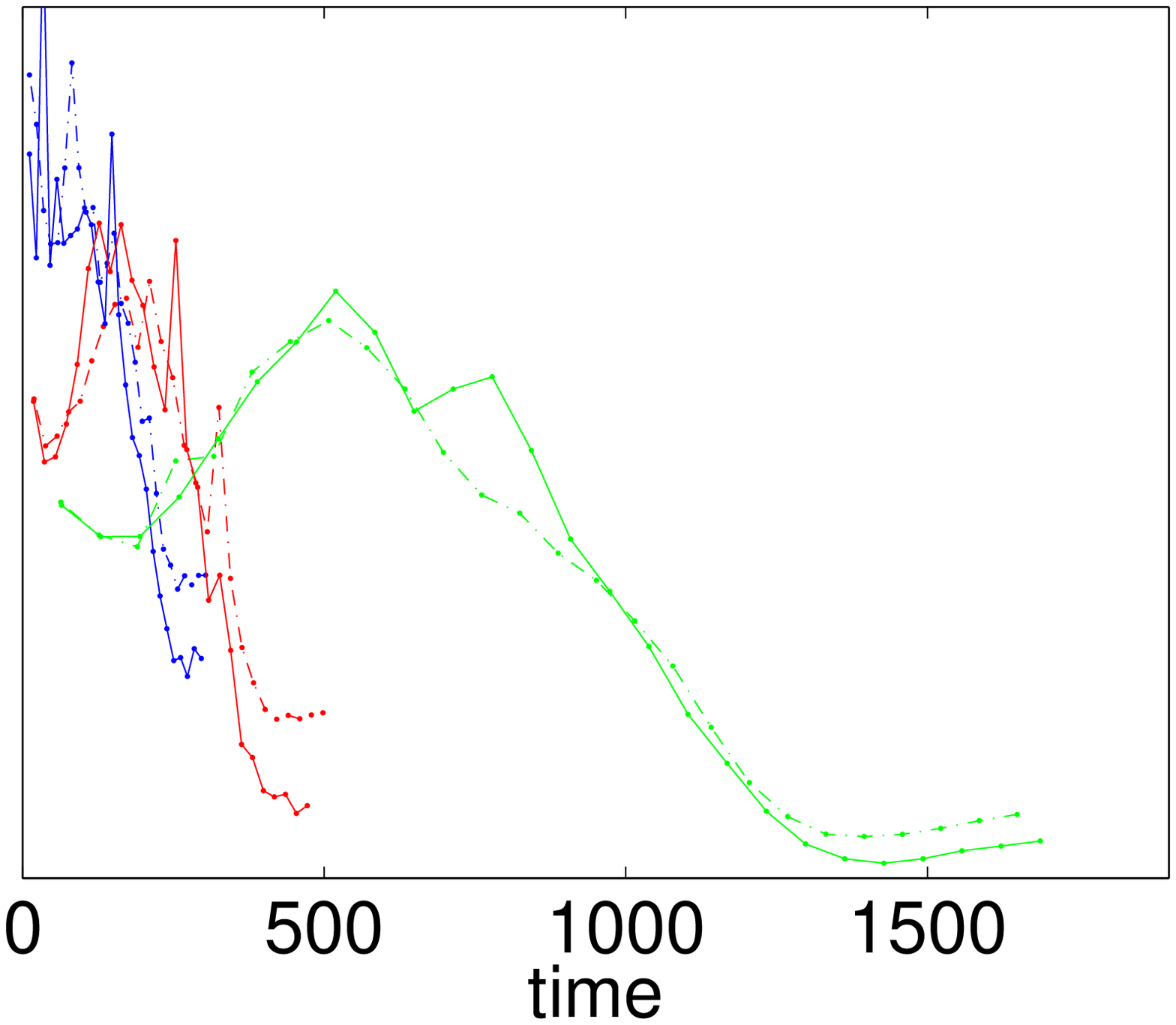} &
    \includegraphics[width=0.235\linewidth]{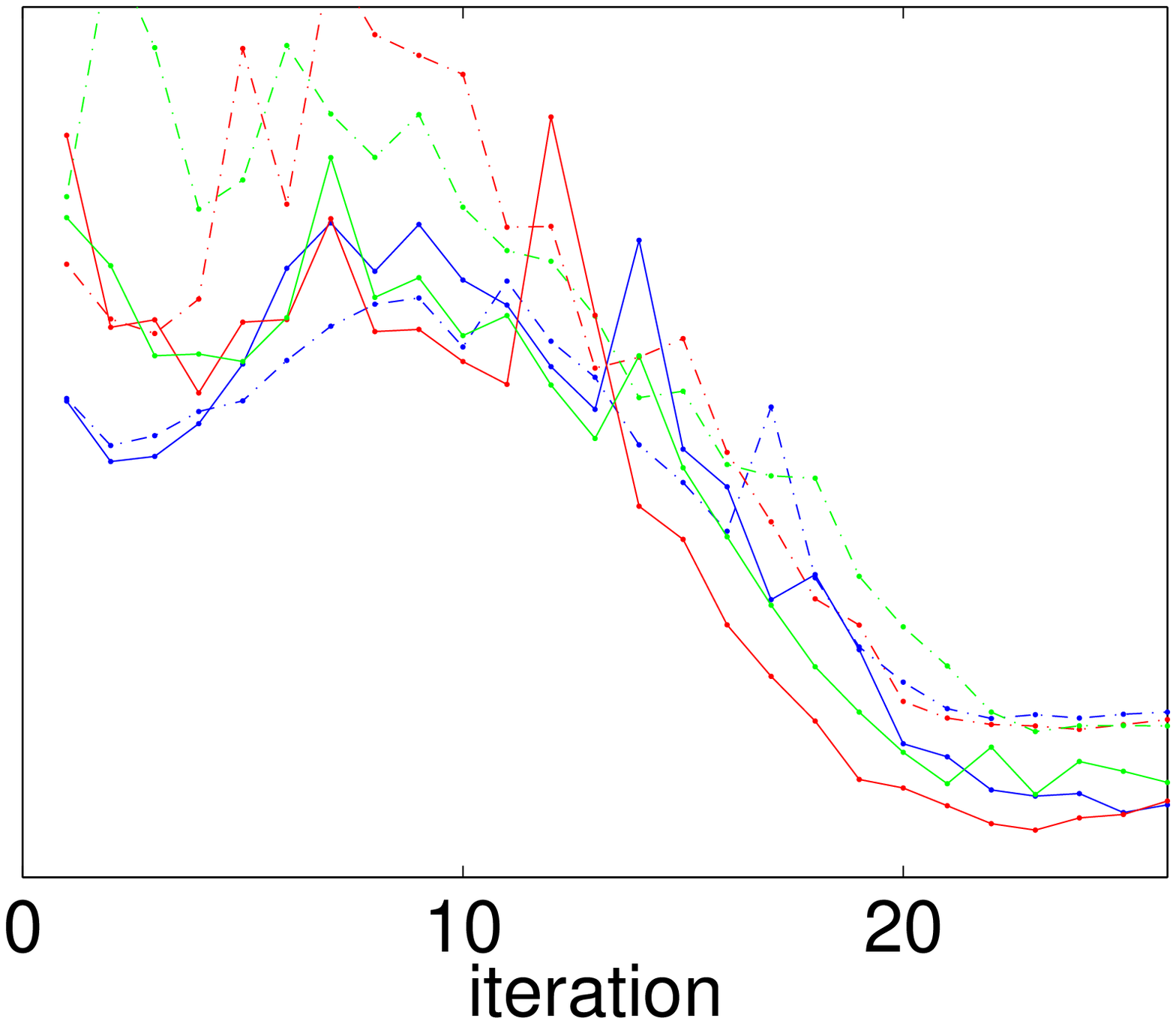} &
    \includegraphics[width=0.235\linewidth]{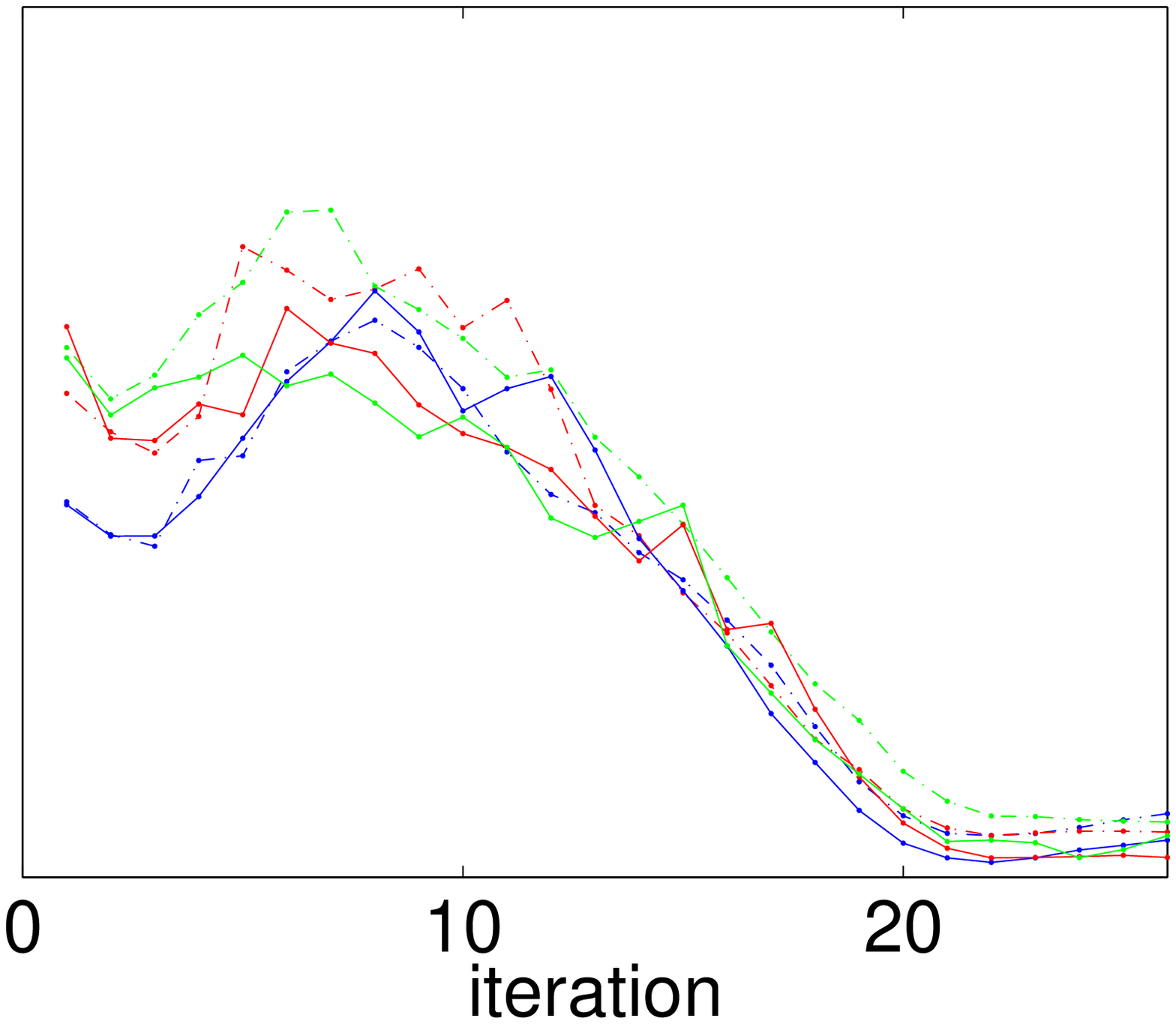} \\[-2.5ex]
    \psfrag{iteration}[t][]{iteration}
    \includegraphics[width=0.274\linewidth]{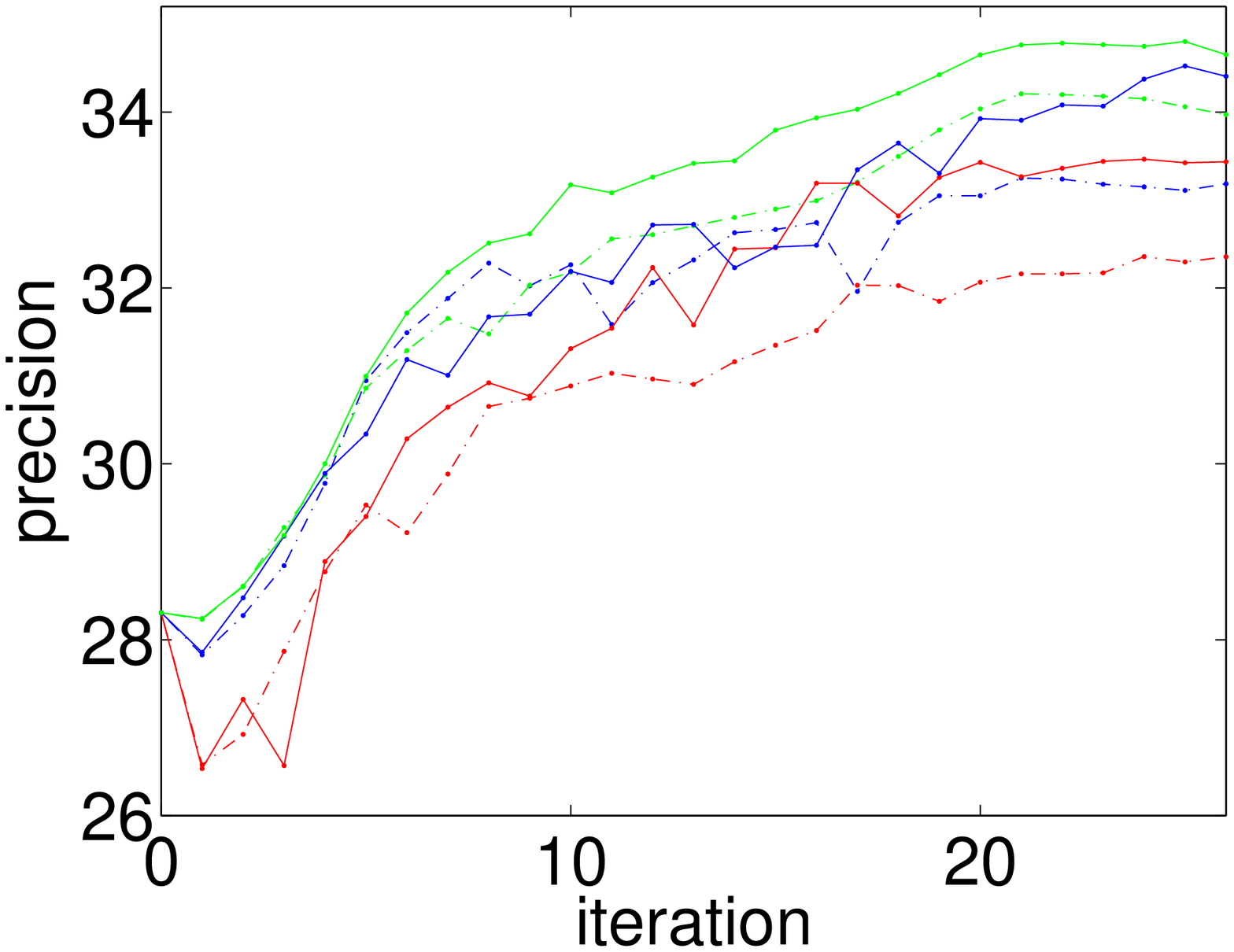} &
    \psfrag{time}[t][]{time}
    \includegraphics[width=0.235\linewidth]{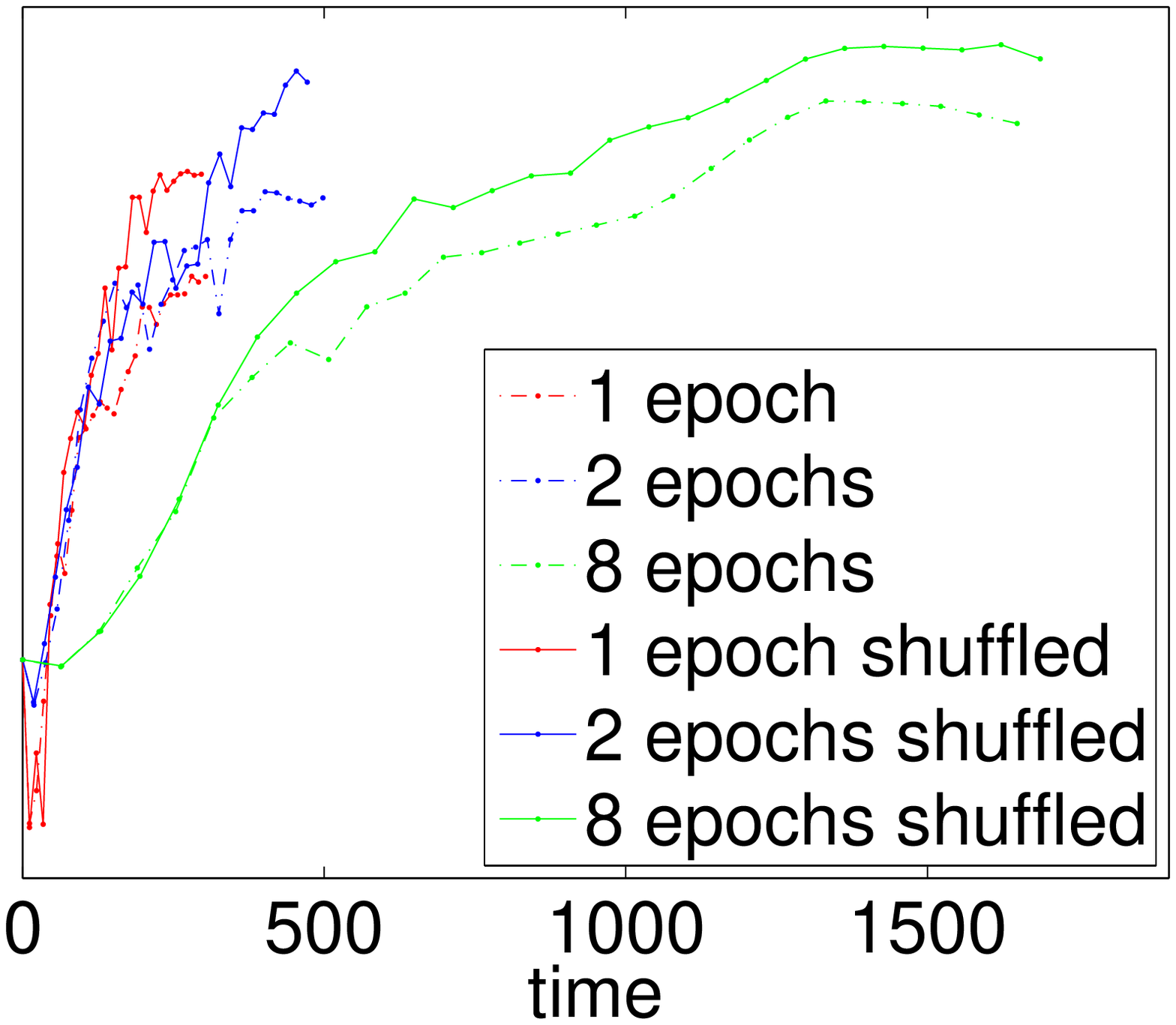} &
    \psfrag{iteration}[t][]{iteration}
    \includegraphics[width=0.235\linewidth]{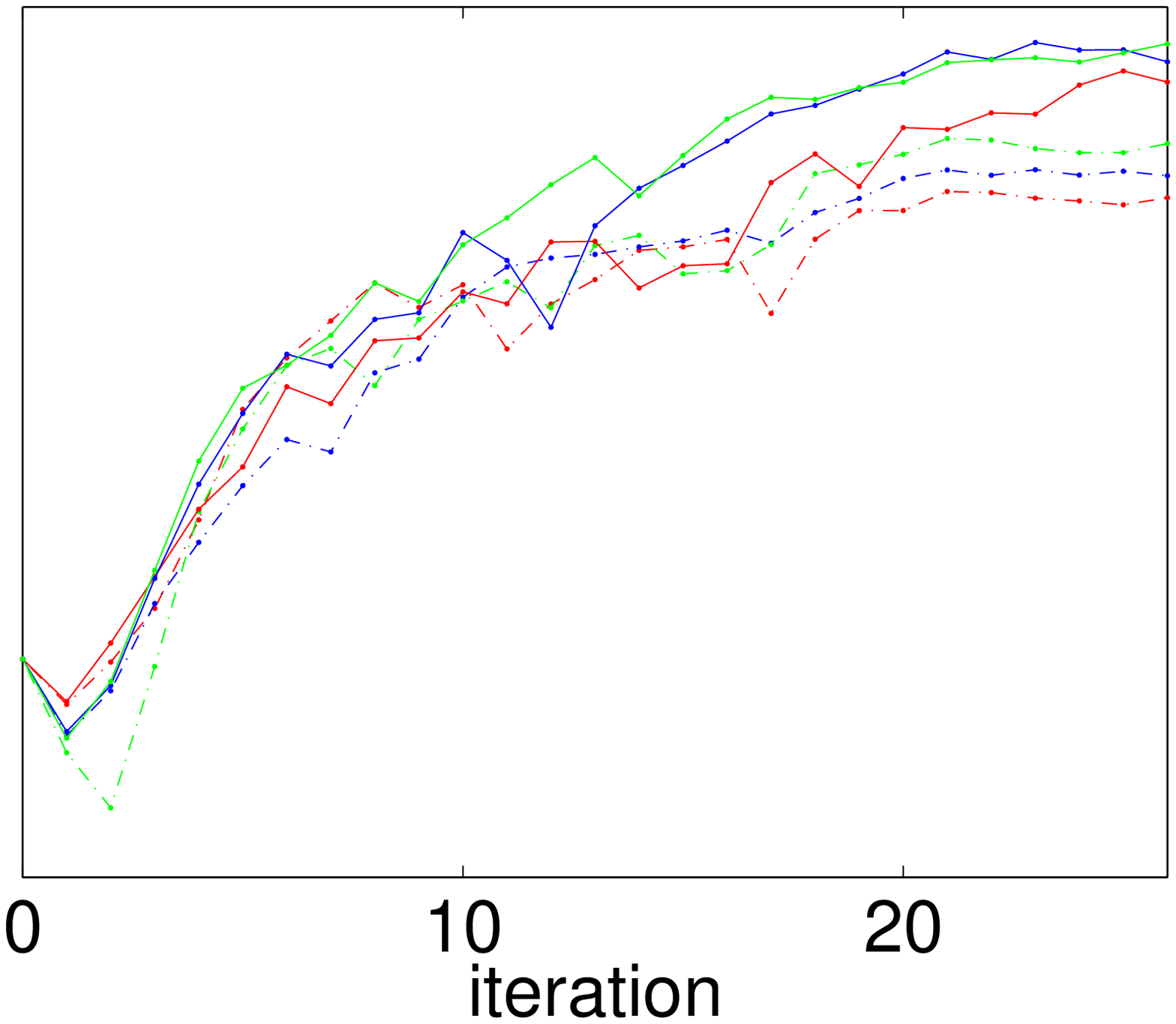} &
    \psfrag{iteration}[t][]{iteration}
    \includegraphics[width=0.235\linewidth]{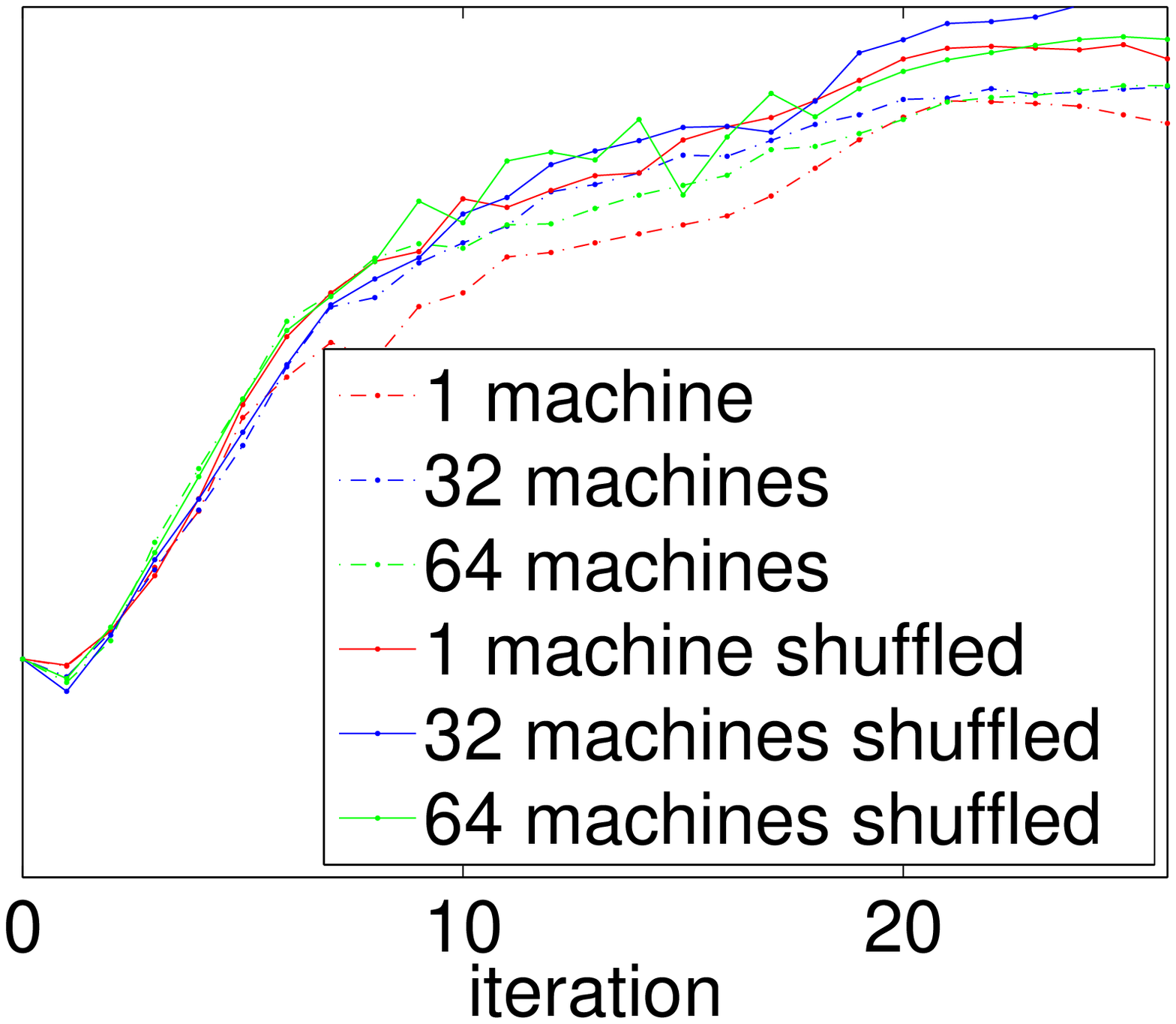}
  \end{tabular}
  \caption{Like fig.~\ref{f:cifar-epochs} but with and without minibatch shuffling in the \W\ step (solid and dashed lines, resp.).}
  \label{f:cifar-epochs-shuffle}
\end{figure}

\subsection{Speedup}
\label{s:expts:speedup}

The fundamental advantage of ParMAC and distributed optimisation in general is the ability to train on datasets that do not fit in a single machine, and the reduction in runtime because of parallel processing. Fig.~\ref{f:speedup} (top row) shows the strong scaling%
\footnote{In ``strong scaling'', the total problem size is fixed and the problem size on each machine is inversely proportional to the number of machines $P$. In ``weak scaling'', the problem size on each machine is fixed, so the total problem size is proportional to $P$. High speedups are easier to obtain in weak scaling than in strong scaling \citep{GoedecHoisie01a}.}
speedups achieved experimentally, as a function of the number of machines $P$ for fixed problem size (dataset and model), in CIFAR and SIFT-1M ($N =$ 50K and 1M training points, respectively). Even though these datasets and especially the number of independent submodels ($M = 2L = 32$ effective submodels of the same size, as discussed in section~\ref{s:speedup-th:practical}) are relatively small, the speedups we achieve are nearly perfect for $P \le M$ and hold very well for $P > M$ up to the maximum number of machines we used ($P = 128$ in the distributed system). The speedups flatten as the number of epochs (and consequently the amount of communication) increases, because for this experiment the bottleneck is the \W\ step, whose parallelisation ability (i.e., the number of concurrent processes) is limited by $M = 2L$ (the \Z\ step has $N$ independent processes and is never a bottleneck, since $N$ is very large). However, as noted earlier, using 1 to 2 epochs gives a good enough result, very close to doing an exact \W\ step. The runtime for SIFT-1M on $P = 128$ machines with 8 epochs was 12 minutes and its speedup 100$\times$. This is particularly remarkable given that the original, nested model did not have model parallelism. These speedups are vastly larger than those achieved by earlier large-scale nonconvex optimisation approaches such as Google's DistBelief system \citep{Le_12a,Dean_12a}, although admittedly the deep nets trained there were far larger than our BAs.

Fig.~\ref{f:speedup} (bottom) shows the speedups predicted by our theoretical model of section~\ref{s:speedup-th}. We set the parameters $e$ and $N$ to their known values, and $M = 2L = 32$ for CIFAR and SIFT-1M and $M = 2L = 128$ for SIFT-1B (effective number of independent equal-size submodels). For the time parameters, we set $t^{\W}_r = 1$ to fix the time units, and we set $t^{\W}_c$ and $t^{\Z}_r$ by trial and error to achieve a reasonably good fit to the experimental speedups. Specifically, we set $t^{\W}_c = 10^4$ for both datasets, and $t^{\Z}_r = 200$ for CIFAR and $40$ for SIFT-1M. Although these are fudge factors, they are in rough agreement with the fact that communicating a weight vector over the network is orders of magnitude slower than updating it with a gradient step, and that the \Z\ step is quite slower than the \W\ step because of the binary optimisation it involves.

Fig.~\ref{f:speedup} (bottom, right plot) also shows the theoretical prediction for the SIFT-1B dataset ($N = 10^8$, $M = 128$), using the same parameters as in SIFT-1M (again assuming the distributed memory system): $t^{\W}_c = 10^4$ and $t^{\Z}_r = 40$. Since here $M$ is quite larger and $N$ is much larger, the speedup is nearly perfect over a very wide range (note the plot goes up to $P = 1\,024$ machines, even though our experiments are limited to $P = 128$).

\begin{figure}[t]
  \centering
  \psfrag{processor}{}
  \begin{tabular}{@{}c@{\hspace{0\linewidth}}c@{\hspace{0\linewidth}}c@{}}
    CIFAR & SIFT-1M & SIFT-1B \\
    \psfrag{speedup}[][t]{speedup $S(P)$ (experiment)}
    \includegraphics[height=0.25\linewidth]{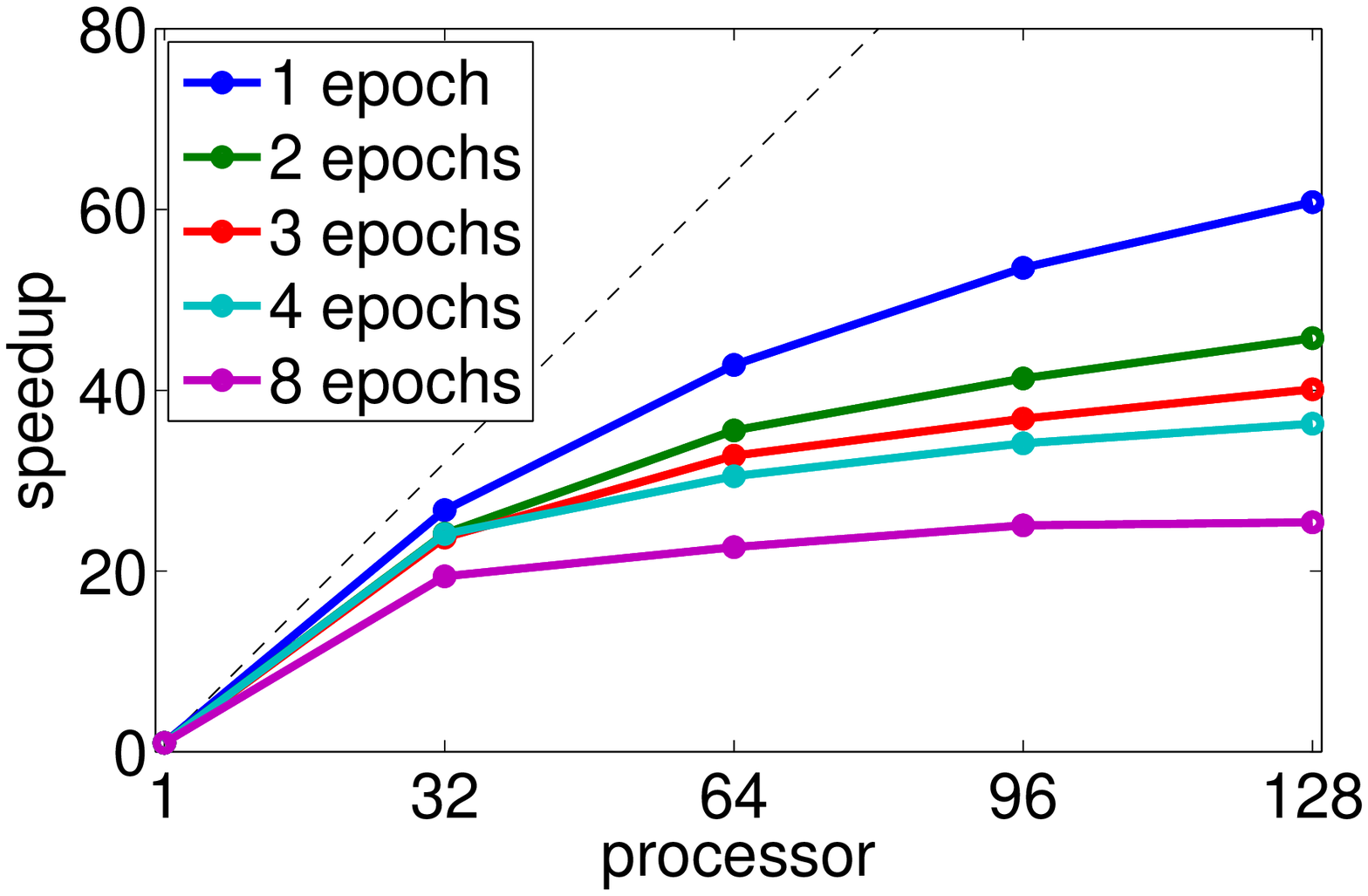} &
    \psfrag{speedup}{}
    \includegraphics[height=0.25\linewidth]{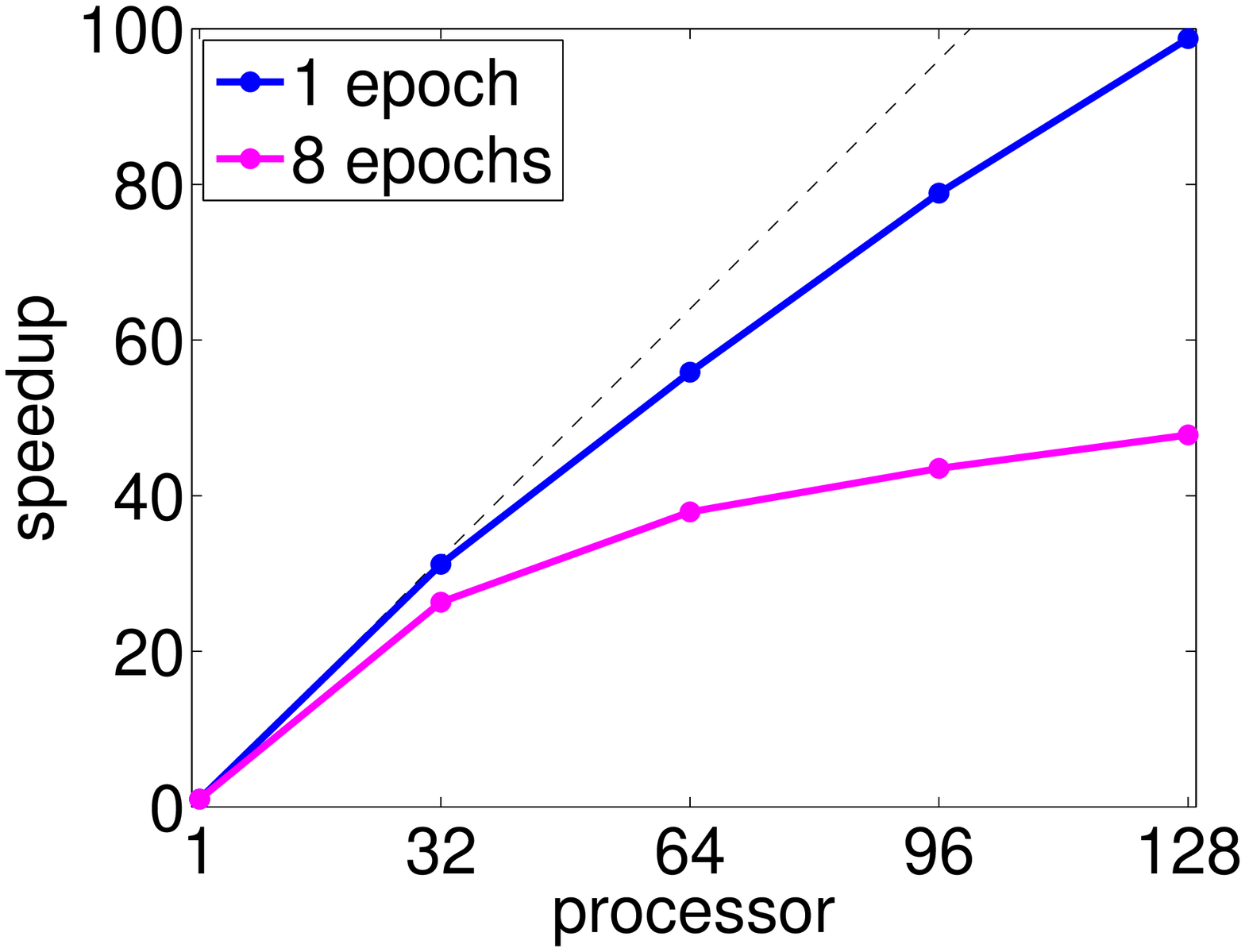} &
    \raisebox{0.125\linewidth}{\caja{c}{c}{too long to run}} \\[2ex]
    \psfrag{processor}[t][]{number of machines $P$}
    \psfrag{speedup}[][t]{speedup $S(P)$ (theory)}
    \includegraphics[height=0.25\linewidth]{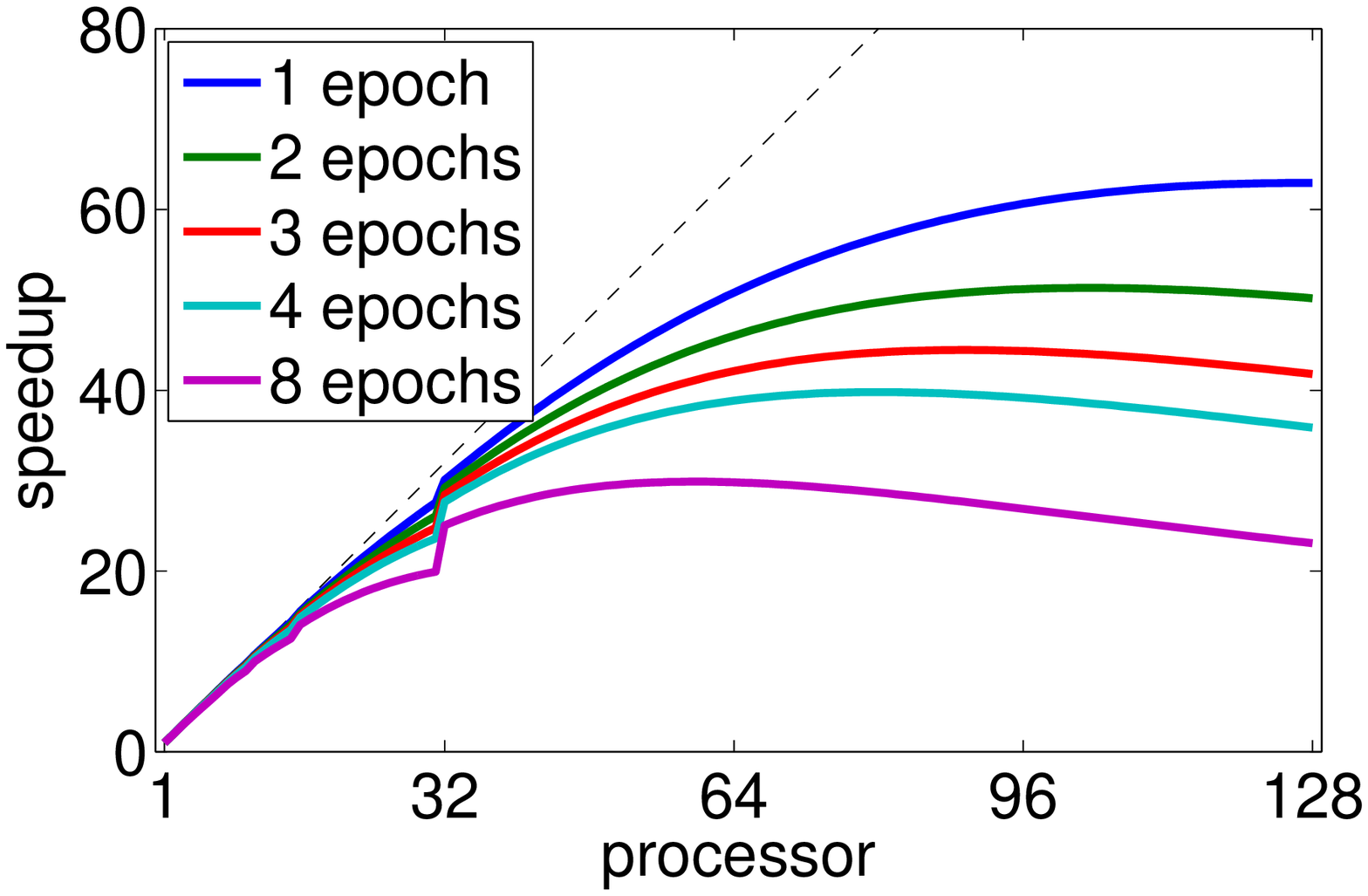} &
    \psfrag{processor}[t][]{number of machines $P$}
    \psfrag{speedup}{}
    \includegraphics[height=0.25\linewidth]{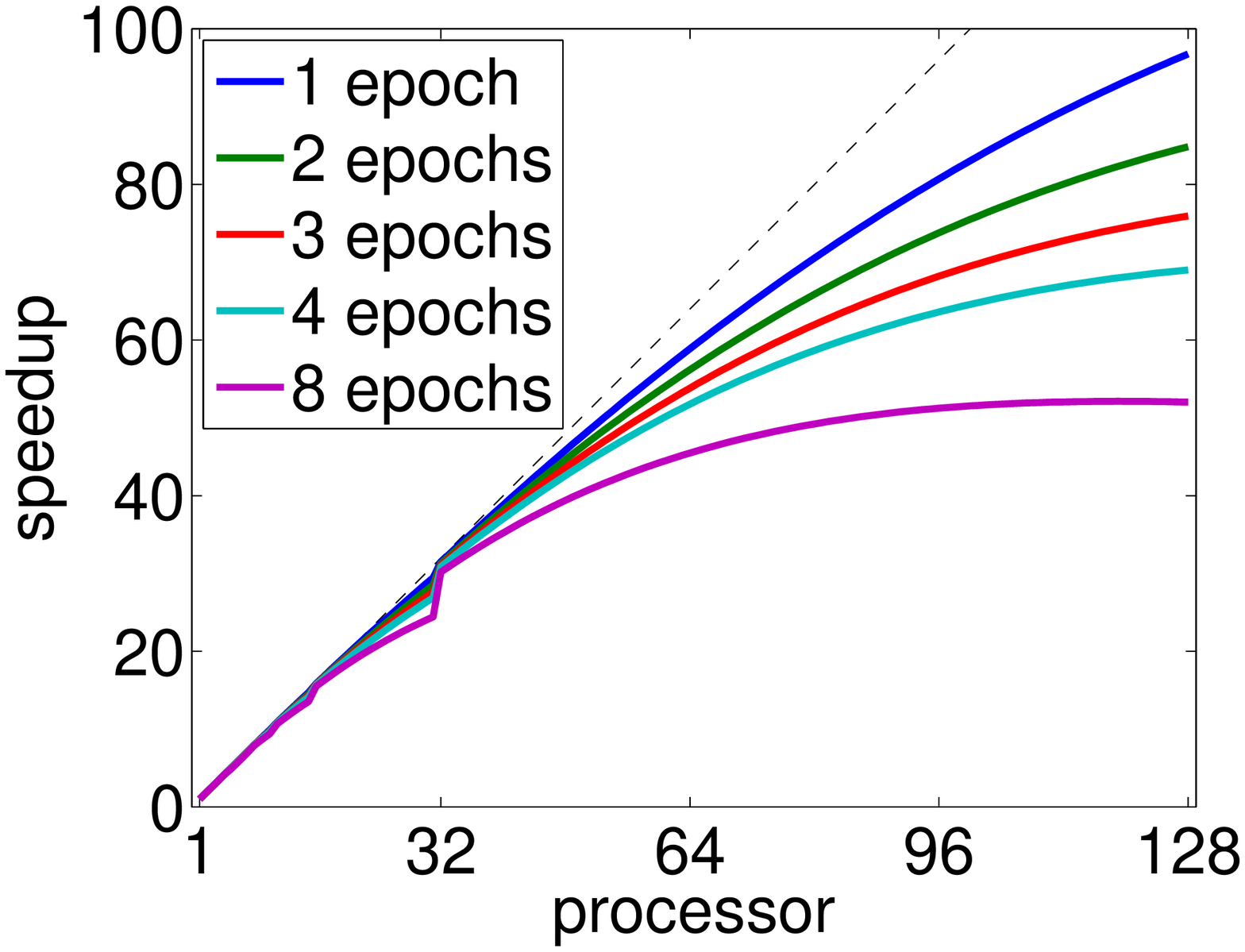} &
    \psfrag{processor}[t][]{number of machines $P$}
    \psfrag{speedup}{}
    \includegraphics[height=0.25\linewidth]{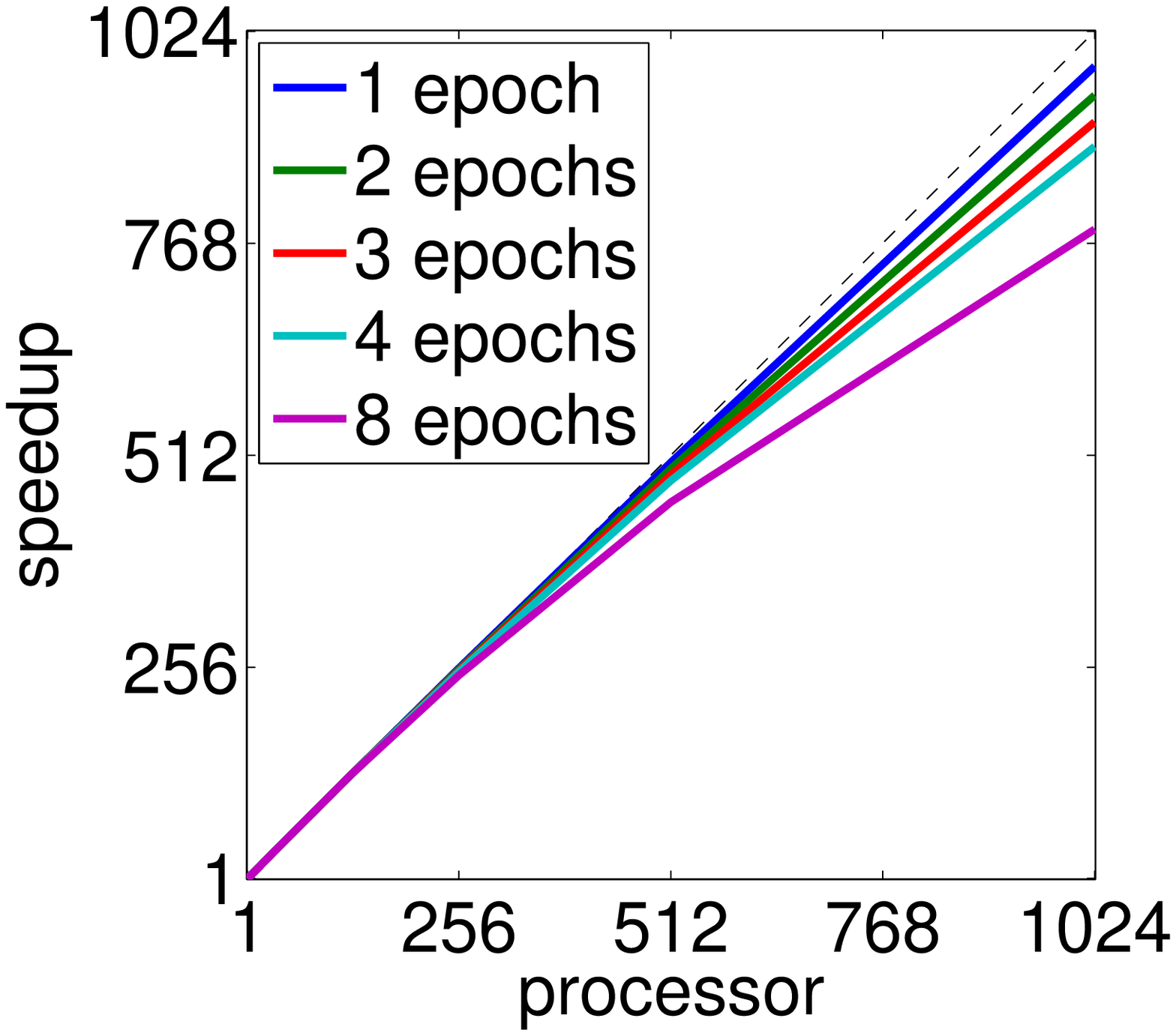}
  \end{tabular}
  \caption{Speedup $S(P)$ as a function of the number of machines $P$ for CIFAR, SIFT-1M and SIFT-1B. Top: experimental result in the distributed memory system. Bottom: theoretical result predicted by the model of section~\ref{s:speedup-th}. The dataset size and number of submodels $(N,M)$ is $(50\,000,32)$ for CIFAR, $(10^6,32)$ for SIFT-1M and $(10^8,128)$ for SIFT-1B.}
  \label{f:speedup}
\end{figure}

\subsection{Large-scale experiment: SIFT-1B dataset}
\label{s:expts:large-scale}

SIFT-1B is one of the largest datasets, if not the largest one, that are publicly available for comparing nearest-neighbour search algorithms with known ground-truth (i.e., precomputed exact Euclidean distances for each query to its $k$ nearest vectors in the base set). The training set contains $N = 100$M vectors, each consisting of 128 SIFT features.

Handling the SIFT-1B dataset required special care because of its size and the limited amount of memory (total 512GB for 128 processors in the distributed system and 256GB for 64 processors in the shared-memory one). Each vector has 128 SIFT features and each feature in the original dataset is stored in a single byte rather than as double-precision floats (8 bytes), as in our other experiments, totalling 12.8GB for the training set if using a linear hash function and 200GB if using an RBF one (see below). Rather than converting it to floats, which would exceed 1TB, we modified our code to convert each feature only as needed. In the \Z\ step each datapoint is processed independently and the conversion to double is done one point at a time. In the \W\ step it is done one minibatch at a time. It is of course possible to use hard disk as additional storage but this would slow down training. The auxiliary coordinates, which must be stored in MAC algorithms, take only 6.25\% the memory of the data (64 bits per datapoint compared to 128 bytes).

We used $L = 64$ bits (hash functions). As hash function, we trained a linear SVM as before, and a kernel SVM using $m$ Gaussian radial basis functions (RBF) with fixed bandwidth $\sigma$ and centres. This means the only trainable parameters are the weights, so the MAC algorithm does not change except that it operates on an $m$-dimensional input vector of kernel values (stored as one byte each), instead of the 128 SIFT features. The Gaussian kernel values are in $(0,1]$ but, as before, to save memory we store them as an unsigned one-byte integer value in $[0,255]$. We used $m = 2\,000$ centres, the maximum we could fit in memory, picked at random from the training set. In trials with a subset of the training set, we set $\sigma = 160$. This worked well and was wide enough to ensure that, with our limited one-byte precision, no data point would produce $m$ zeros as kernel values.

On trials on a subset of the training dataset, we set the number of epochs to $e = 2$ with shuffling (we observed no improvements by using more epochs, which is understandable given the size of the dataset). We initialised the binary codes from truncated PCA trained on a subset of size 1M (recall@R=100: 55.2\%), which gave results comparable to the baseline in \citet{Jegou_11b} if using 8 bytes per indexed vector (without postprocessing by reranking as done in that paper).

We ran ParMAC on the whole training set in the distributed system with 128 processors for 6 iterations and in the shared-memory one with 64 processors for 10 iterations. The results are given in the following table and figures~\ref{f:sift1m-iters}--\ref{f:sift1m-recall}:
\begin{center}
  \begin{tabular}{@{}lccc@{}}
    \toprule
    \raisebox{-0.7ex}[0pt][0pt]{\caja{t}{l}{Hash function \\ (encoder)}} & \raisebox{-0.7ex}[0pt][0pt]{\caja{t}{c}{Recall \\ @R=100}} & \multicolumn{2}{c}{Time (hours)} \\
    \cmidrule{3-4}
     &  & distrib. & shared \\ 
    \midrule
    linear SVM & 61.5\% & 29.30 & 11.04 \\
    kernel SVM & 66.1\% & 83.44 & 32.19 \\
    \bottomrule
  \end{tabular}
\end{center}
The learning curves (fig.~\ref{f:sift1m-iters}) are essentially identical over both systems. The nonlinear RBF hash function outperforms the linear one in recall, as one would expect. The improvement occurs across the whole range of $R$ recall values (fig.~\ref{f:sift1m-recall}). Note the error in the nested model, $E_{\text{BA}}$, does not decrease monotonically. This is because MAC optimises instead the penalised function $E_Q$, in an effort the minimise $E_{\text{BA}}$ as $\mu$ increases.

Based on our previous results, the small number of epochs and the larger number of submodels in the \W\ step, we expect nearly perfect speedups. We cannot compute the actual speedup because the single-machine runtime is enormous. Using a scaled-down model and training set, we estimated that training in one machine (with enough RAM to hold the data and parameters) would take months.

Although the speedups are comparable on both the distributed and the shared-memory system, the former is 3--4 times slower. The reason is the distributed system has both slower processors and slower interprocessor communication (across a network); see also fig.~\ref{f:shared_dist}.

\begin{figure}[p]
  \centering
  \psfrag{recallR}[][t]{recall@R=100}
  \psfrag{AEerr}[][t]{$E_{\text{BA}}$}
  \psfrag{iteration}[][]{iteration}
  \begin{tabular}{@{}c@{\hspace{0.05\linewidth}}c@{}}
    \includegraphics[height=0.33\linewidth]{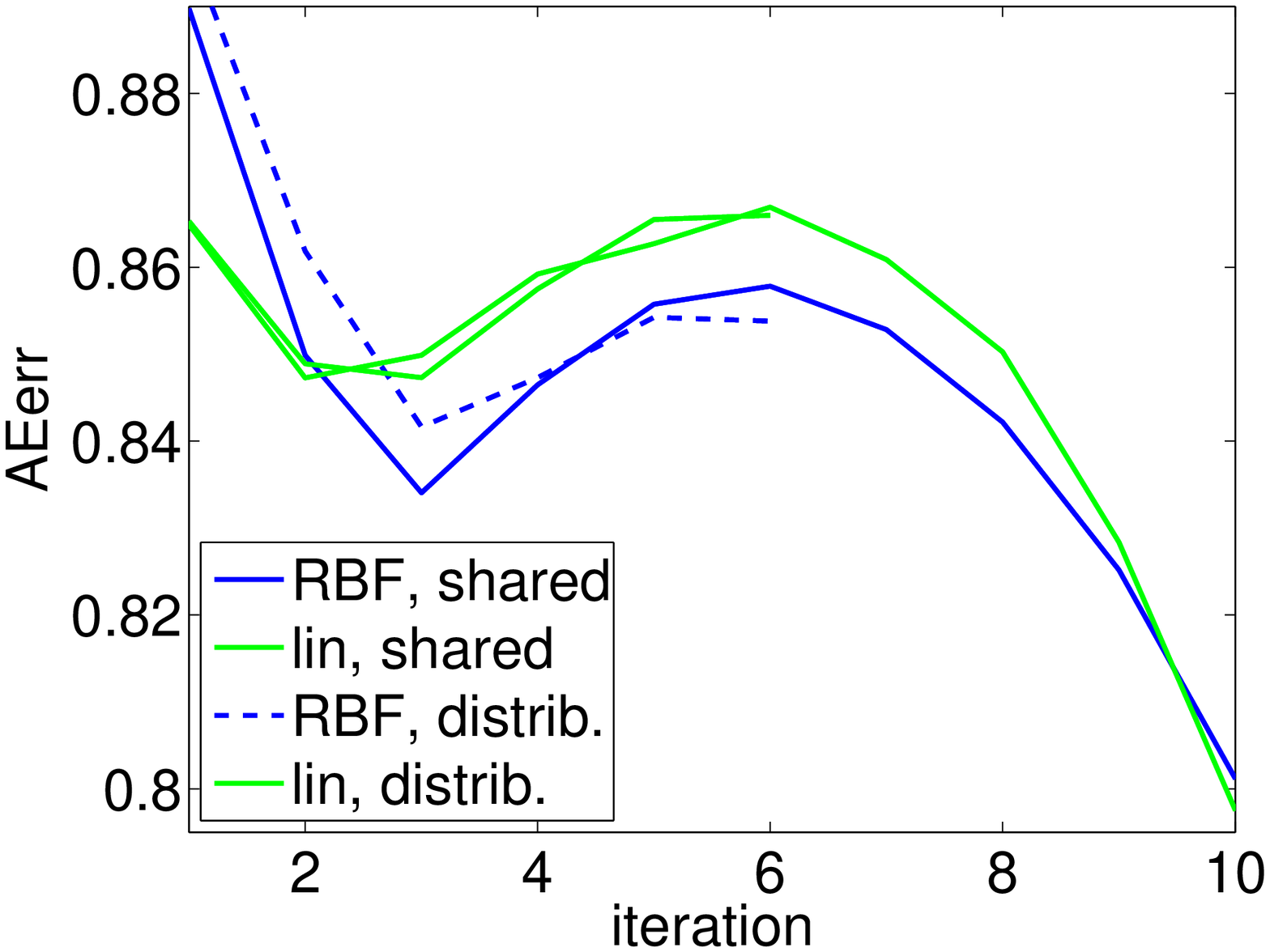} &
    \includegraphics[height=0.33\linewidth]{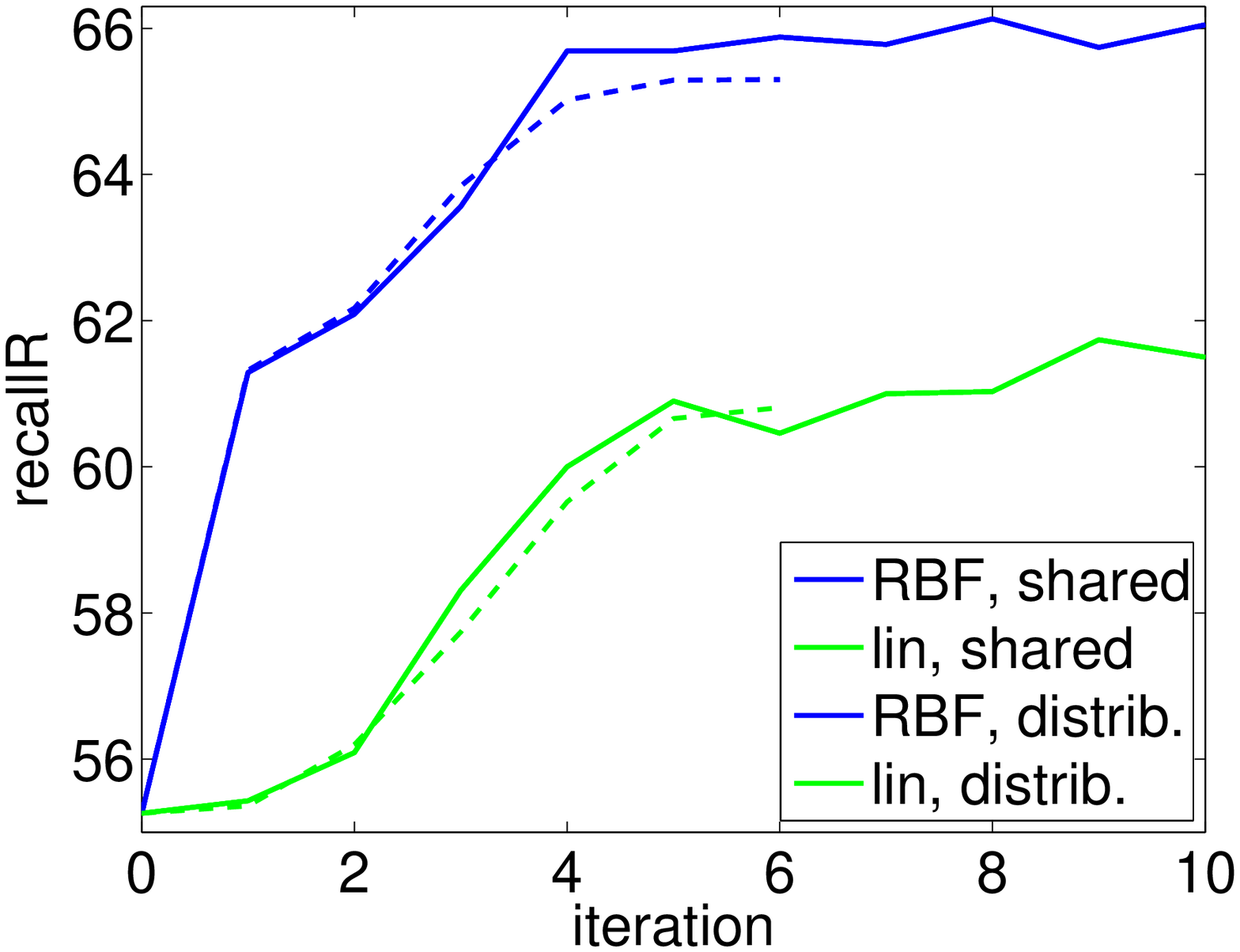}
  \end{tabular}
  \caption{SIFT-1B dataset, using the shared-memory and distributed system (solid and dashed lines, resp.).}
  \label{f:sift1m-iters}
\end{figure}

\begin{figure}[p]
  \centering
  \begin{tabular}{@{}c@{}c@{}c@{}}
    \psfrag{recallR}[][]{recall@R}
    \psfrag{R}[t][]{$R$}
    \includegraphics[height=0.26\linewidth]{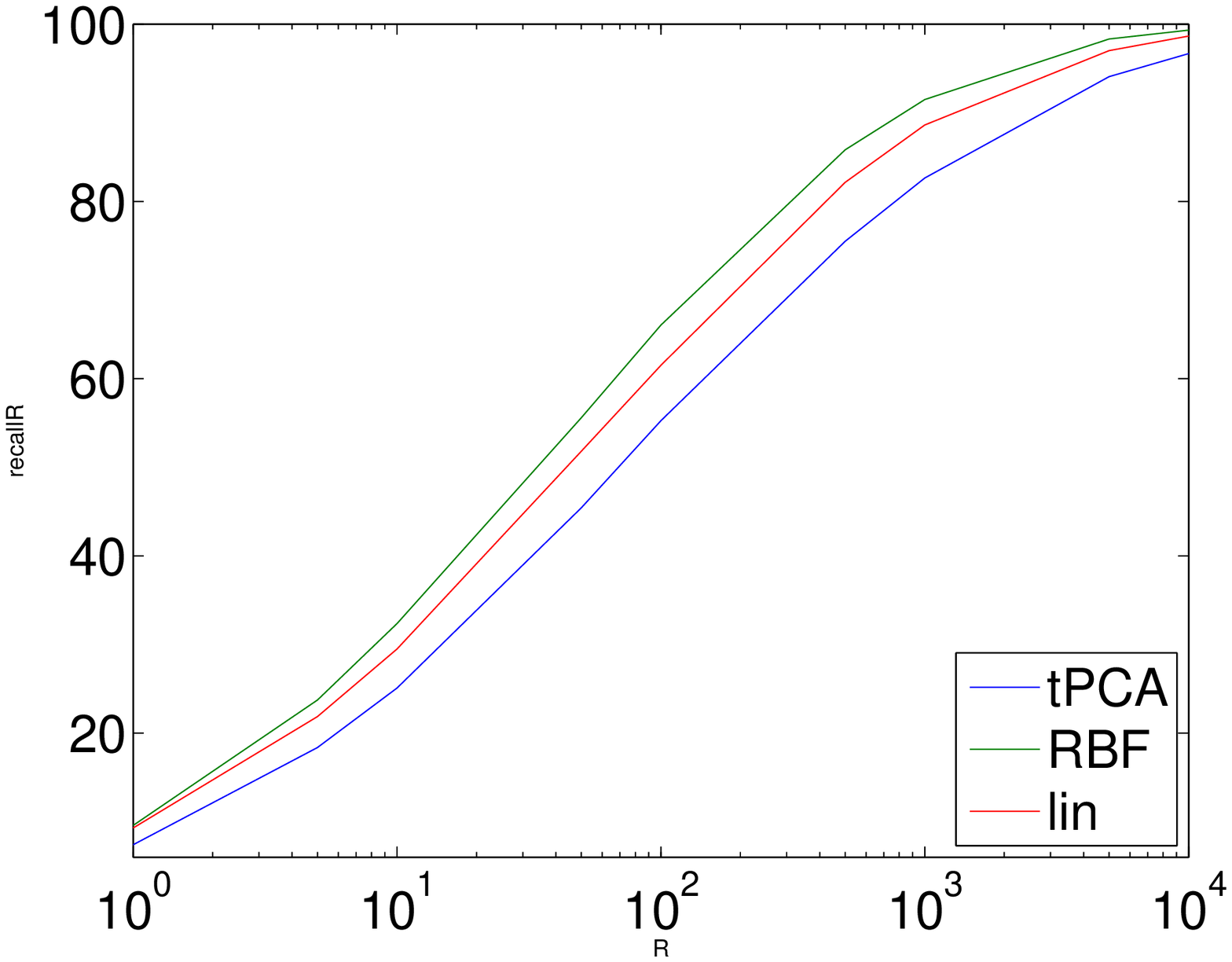} &
    \psfrag{recallR}{}
    \psfrag{R}[t][]{$R$ (linear hash function)}
    \includegraphics[height=0.26\linewidth,bb=5 158 588 616,clip]{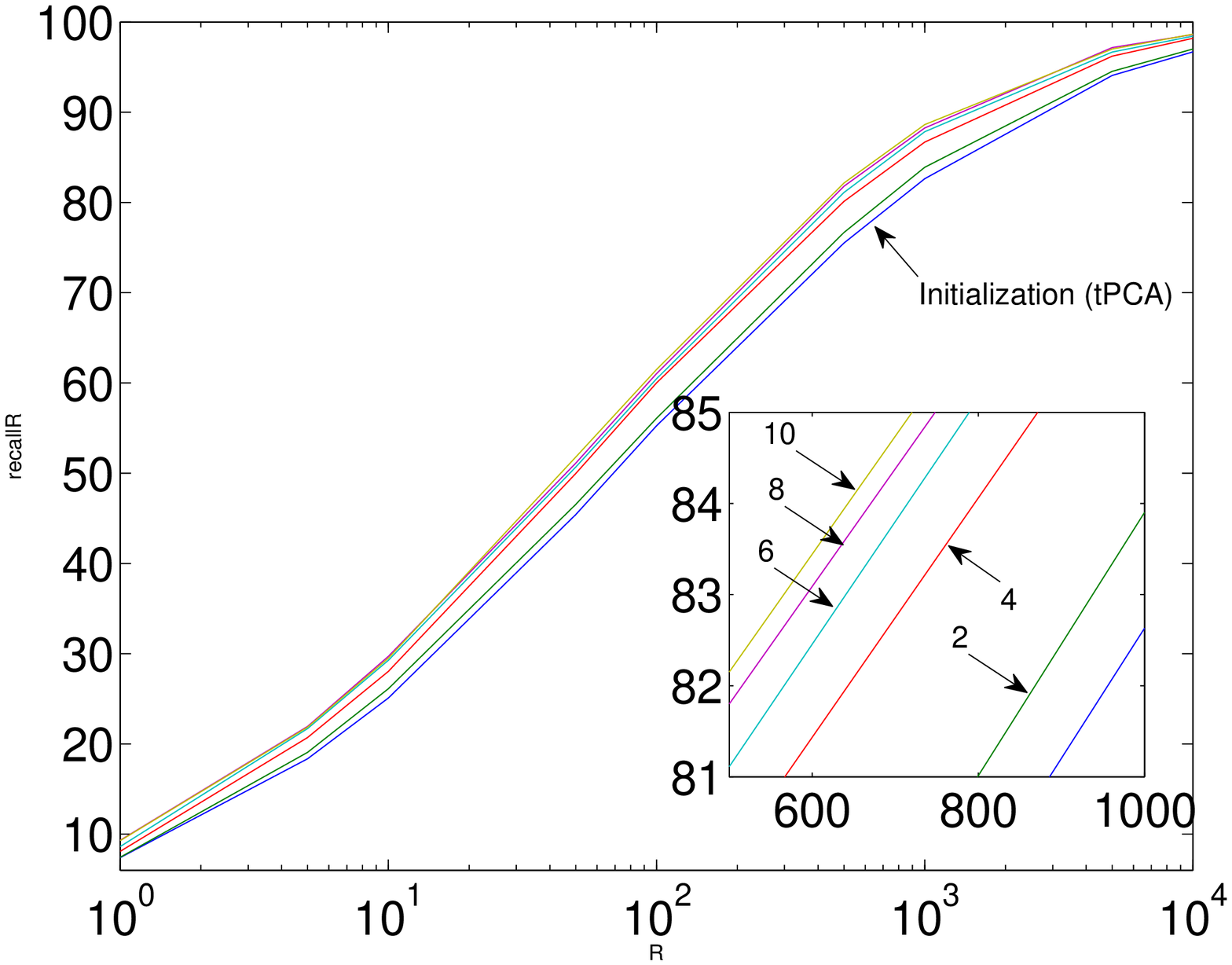} &
    \psfrag{recallR}{}
    \psfrag{R}[t][]{$R$ (RBF hash function)}
    \includegraphics[height=0.26\linewidth,bb=5 158 588 616,clip]{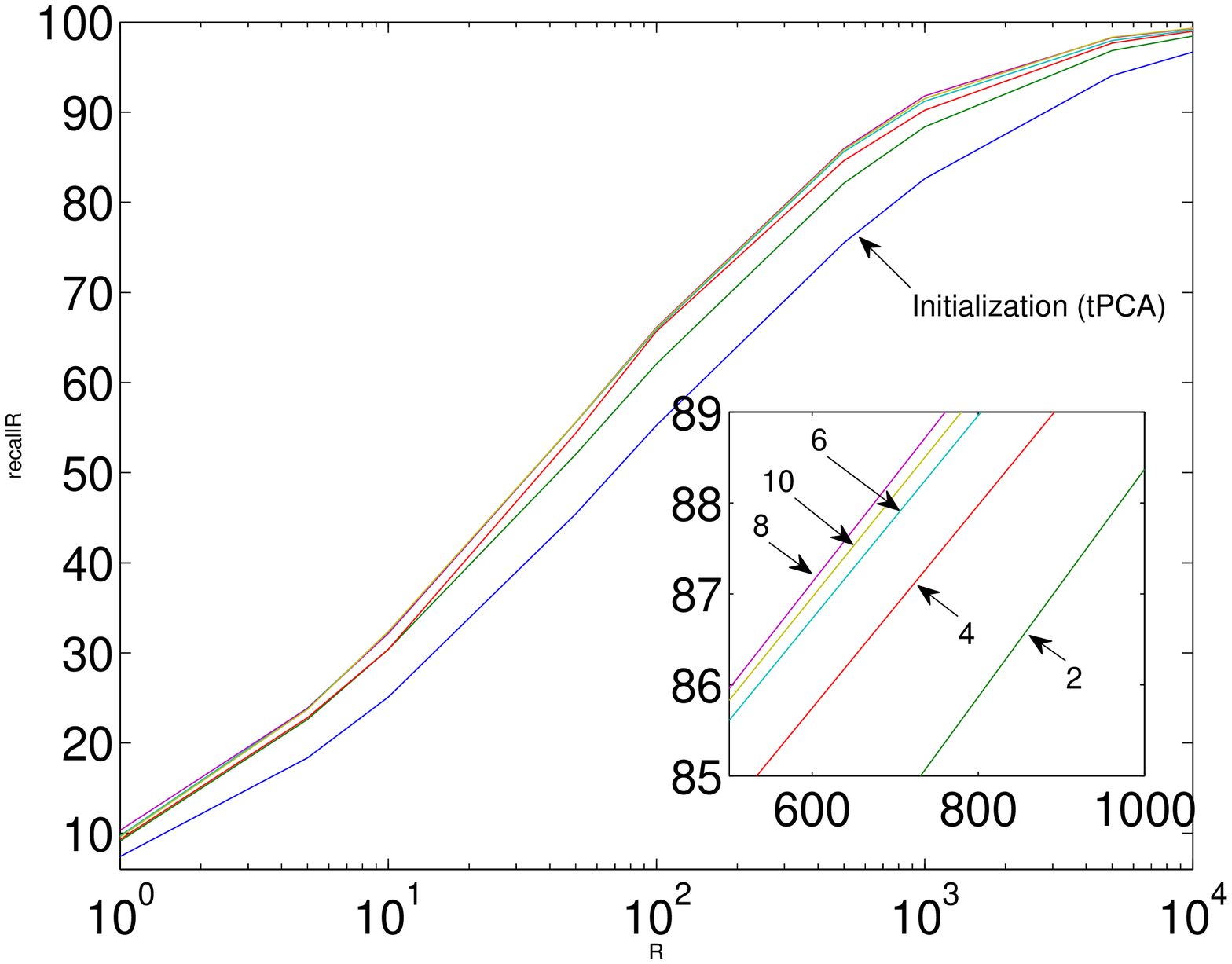}
  \end{tabular}
  \caption{Recall@R on the SIFT-1B dataset for truncated PCA (initialisation), linear and kernel hash functions (left plot: final result; right two plots: over iterations, as labelled).}
  \label{f:sift1m-recall}
\end{figure}

\begin{figure}[p]
  \centering
  \psfrag{settings}[t][]{\# nodes $\times$ \# processors}
  \psfrag{time (s)}[][t]{time (s)}
  \includegraphics[width=0.40\linewidth]{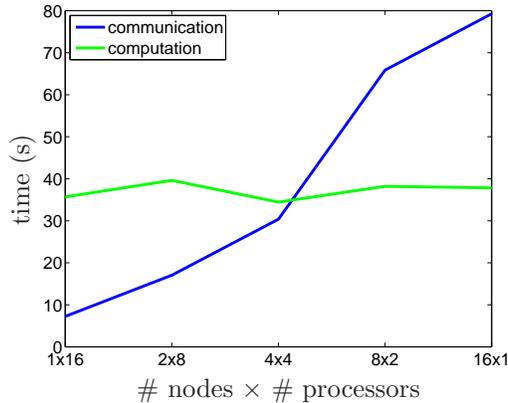}
  \caption{Time spent on communication and computation as a function of the number of processors per node in the TSCC cluster. The time for our shared-memory UC Merced cluster (also using 16 processors, i.e., corresponding to 1$\times$16) is 2.57 and 8.76 seconds for communication and computation, respectively.}
  \label{f:shared_dist}
\end{figure}

\subsection{Shared-memory vs distributed systems}

The TSCC distributed cluster consists of nodes containing 16 processors and 64GB RAM. These processors communicate through shared-memory within a node, and across a network otherwise (which is slower). In all our experiments up to now, we always allocated processors within the same node if possible. But, depending on whether a user requests processors within or across nodes, there is then a tradeoff between both communication modes. A full study of this issue is beyond our scope, which is to understand the ParMAC algorithm in general rather than for specific computer architectures. However, we ran a small experiment to evaluate the computation and communication time spent as a function of the number of processors per node. We set the total number of processors to $P = 16$ but allocated them in the following configurations: from a single node with all 16 processors (1$\times$16, pure shared-memory) to 16 nodes each with 1 processor (16$\times$1, pure distributed), and intermediate configurations such as 2 nodes each with 8 processors (2$\times$8). We used the RBF hash function from the SIFT-1B experiment with all settings as before and trained it on a subset of 20K points for a single iteration. Figure~\ref{f:shared_dist} shows the resulting times. While the computation time remains constant, the communication time increases as we move from shared-memory to distributed settings, as expected. Hence, the effect on the ParMAC algorithm would be to increase the \W\ step runtime correspondingly and lower the parallel speedup.

\section{Discussion}
\label{s:discussion}

Developing parallel, distributed optimisation algorithms for nonconvex problems in machine learning is challenging, as shown by recent efforts by large teams of researchers \citep{Le_12a,Dean_12a}. One important advantage of ParMAC is its simplicity. Data and model parallelism arise naturally thanks to the introduction of auxiliary coordinates. The corresponding optimisation subproblems can often be solved reusing existing code as a black box (as with the SGD training of SVMs and linear mappings in the BA). A circular topology is sufficient to achieve a low communication between machines. There is no close coupling between the model structure and the distributed system architecture. The development and implementation of ParMAC for binary autoencoders on large datasets in a distributed cluster was achieved in a few months by the PI and one junior PhD student (both without prior experience in MPI).

Rather than an algorithm, MAC is a meta-algorithm that can produce a specific optimisation algorithm for a given nested problem, depending on how the auxiliary coordinates are introduced and on how the resulting subproblems are solved (in this sense, MAC is similar to expectation-maximisation (EM) algorithms; \citealp{Dempst_77a}). For example, in the low-dimensional SVMs of \citet{WangCarreir14a}, the \Z\ step is a small quadratic program for each data point. However, regardless of these specifics, the resulting MAC algorithm typically exhibits a \W\ step with $M$ independent submodels and a \Z\ step with $N$ independent coordinates for the data points. Likewise, the specifics of a ParMAC algorithm (how the \W\ and \Z\ steps are optimised) will depend on its corresponding MAC algorithm. However, it will always split the data and auxiliary coordinates over machines and consist of a \Z\ step with no communication between machines, and a \W\ step where submodels visit machines in a circular topology, effectively training themselves by stochastic optimisation.

Further improvements can be made in specific problems. For example, it is possible to have further parallelisation or less dependencies (e.g.\ the weights of hidden units in layer $k$ of a neural net depend only on auxiliary coordinates in layers $k$ and $k+1$). This may reduce the communication in the \W\ step, by sending to a given machine only the model portion that it needs, or by allocating cores within a multicore machine accordingly. Also, the \W\ and \Z\ step optimisations can make use of further parallelisation by GPUs or by distributed convex optimisation algorithms. And, if the submodels are small in size, it may be better for each machine to operate on a set of submodels and then send them all together in a larger message (rather than sending each submodel as it is finished), since this will reduce latency overheads (i.e., the setup cost of a message). Many more refinements can (and should) be done in an industrial implementation. For example, one can store and communicate reduced-precision values for data and parameters with little effect of the accuracy, as has been done in neural nets (e.g.\ \citealp{Gupta_15a,Han_15a,Han_16a}). Various system-dependent optimisations may be possible (beyond those that a good compiler may be able to do, such as loop unrolling or code inlining), such as improving the spatial or temporal locality of the code given the type and size of the cache in each machine. In this paper, we have tried to keep our implementation as simple as possible, because our goal was to understand the parallelisation speedups of ParMAC in a setting as general as possible, rather than trying to achieve the very best performance for a particular dataset, model or distributed system.

ParMAC is very efficient in communication: no data or coordinates are ever sent, only the entire model $e+1$ times per iteration. Using one epoch (which is sufficient in large datasets), or using $e$ epochs but performing them within each machine, the entire model is sent twice per iteration. This is near optimal if we note the following. If the data cannot be communicated, then at every iteration each submodel must visit each machine (for it to be trained on the entire data). Hence, any correct algorithm will have to send the entire model at least once; ParMAC does so twice. Also, the circular topology is the minimal topology (considered as a directed graph on the $P$ machines) that is necessary to be able to optimise a global model on the entire dataset with $P$ machines, because each machine must be able to communicate with some other machine. It has $P$ edges and is truly distributed, with each machine having the same importance. 

A popular model for distributed optimisation (e.g.\ with parallel SGD) uses a worker-server connectivity: $W \gg 1$ workers that do actual parameter optimisation and $S \ge 1$ ``parameter servers''  that collect and broadcast parameters to worker machines. This is a bipartite graph with bidirectional edges between servers and workers, having $SW$ edges, which is quite larger than the number of machines $P = S+W$. The entire model must be sent twice per iteration, to and from the parameter server, but this creates a bottleneck when multiple workers send data to the same server, and $S \ll W$ in practice. No such bottleneck occurs in ParMAC.

Parallel and distributed computing systems have been around for decades. One important class are supercomputers, which are carefully designed in terms of the processors, memory system and connection network. They have been traditionally used to solve a wide variety of large-scale scientific computation problems, such as weather prediction, nuclear reactor modelling, or astrophysical or molecular simulations. Another important class are clusters of inexpensive, heterogenous workstations connected through an Ethernet network, with workstations differing in speed, memory/disk capacity, number of cores/GPUs, etc. This is used in data centres in Google, Amazon and other companies, and also in distributed computation models such as SETI@home that capitalise on the computation and Internet connectivity available to individuals, and their willingness to donate them to projects they find worthy. In these systems, the machine learning task may be one of other tasks running concurrently, such as web searches or email in a data centre (which may operate on the same data as the machine learning task), or personal applications in an individual's workstation. Supercomputers and clusters differ considerably across important factors: suitability for a particular problem, computation and communication speed, size of memory and disk, connection network, fault tolerance, load, cost, energy consumption, etc. At present it is unclear what the best choices will be for machine learning models (which exhibit a wide variety themselves), and we expect to see many different possibilities been researched in the immediate future. We suggest that ParMAC, by itself or in combination with other techniques, may play an important role with nested models because of the embarrassing parallelism it introduces and its loose demands on the underlying distributed system.

\section{Conclusion}
\label{s:concl}

We have proposed ParMAC, a distributed model for the method of auxiliary coordinates for training nested, nonconvex models in general, analysed its parallel speedup and convergence, and demonstrated it with an MPI-based implementation for a particular case, to train binary autoencoders. MAC creates parallelism by introducing auxiliary coordinates for each data point to decouple nested terms in the objective function. ParMAC is able to translate the parallelism inherent in MAC into a distributed system by 1) using data parallelism, so that each machine keeps a portion of the original data and its corresponding auxiliary coordinates; and 2) using model parallelism, so that independent submodels (weight vectors of a hash function or hidden unit) visit every machine in a circular topology, effectively executing epochs of a stochastic optimisation, without the need for a parameter server and therefore no communication bottlenecks. This keeps the communication between machines to a minimum within each iteration. In this sense, ParMAC can be seen as a strategy to be able to use existing, well-developed (convex) distributed optimisation techniques---applicable to simple functions---to a setting where simple functions are coupled by nesting into a nonconvex function whose training data is distributed over machines. The convergence properties of MAC (to a stationary point of the objective function) remain essentially unaltered in ParMAC. The parallel speedup can be theoretically predicted to be nearly perfect when the number of submodels is comparable or larger than the number of machines, and to eventually saturate as one continues to increase the number of machines, and indeed this was confirmed in our experiments. ParMAC also makes it easy to account for data shuffling, load balancing, streaming and fault tolerance. Hence, we expect that ParMAC could be a basic building block, in combination with other techniques, for the distributed optimisation of nested models in big data settings.

\subsubsection*{Acknowledgements}

Work supported by a Google Faculty Research Award and by NSF award IIS--1423515. We thank Dong Li (UC Merced) for useful discussions about MPI and performance evaluation on parallel systems, and Quoc Le (Google) for useful discussion about Google's DistBelief system.

\clearpage

\appendix

\section{Theoretical analysis of the speedup: proofs}
\label{s:speedup-app}

In section~\ref{s:speedup-th} we proposed the following theoretical estimate for the speedup $S(P)$:
\begin{equation}
  \label{e:speedup2}
  \tag{\ref{e:speedup}'}
  S(P) = \frac{T(1)}{T(P)} = \frac{\rho \frac{1}{\ceil{M/P}} M P}{\frac{1}{N} P^2 + \rho_2 P + \rho_1 \frac{1}{\ceil{M/P}} M}.
\end{equation}
Consider $S(P)$ as a real function of a real variable $P \ge 1$ (keeping in mind that only integer values of $P$ can occur in practice). The function $\ceil{M/P}$ is piecewise constant and takes the values $M,M-1,\dots,1$ as $P$ increases from $P=1$, with discontinuities where $M/P = k$ for $k=M,M-1,\dots,1$. Hence, $S(P)$ is piecewise continuous on $M$ intervals of the form $\big[1,\frac{M}{M-1}\big),\big[\frac{M}{M-1},\frac{M}{M-2}\big),\dots,\big[\frac{M}{2},M\big),[M,\infty)$. (Many of these intervals occur between integer values of $P$ so they are actually unobserved in practice; for example, for $M=16$ there are 8 intervals between $P=1$ and $P=2$.) Within each interval $P \in \big[\frac{M}{k},\frac{M}{k-1}\big)$ we have $\ceil{M/P} = k$, hence we can equivalently write the speedup of eq.~\eqref{e:speedup} as the following rational function of $P$:
\begin{equation}
  \label{e:speedup3}
  S(P) = \frac{\frac{1}{k} \rho M P}{\frac{1}{N} P^2 + \rho_2 P + \frac{1}{k} \rho_1 M} \quad \text{for } P \in \textstyle\big[\frac{M}{k},\frac{M}{k-1}\big), \qquad k=M,M-1,\dots,1.
\end{equation}

\subsection{Characterisation of the speedup function $S(P)$}

Our main theorem is theorem~\ref{th:speedup-charact} below. It characterises how the speedup grows as a function of $P$. Let us define
\begin{equation}
  \label{e:speedup-max2}
  \tag{\ref{e:speedup-max}'}
  P^*_k = \sqrt{\rho_1 M N / k} \qquad S^*_k = S(P^*_k) = \frac{\rho M / k}{\rho_2 + 2 \sqrt{\rho_1 M / N k}} \qquad k = 1,2,\dots,M.
\end{equation}
\begin{thm}
  \label{th:speedup-charact}
  Consider the function $S(P)$ of eq.~\eqref{e:speedup3} and $P^*_k$ and $S^*_k$ as in eq.~\eqref{e:speedup-max2}. Then:
  \begin{enumerate}
  \item $S^*_k < S^*_{k-1}$ for $k = 2,\dots,M$.
  \item Within interval \smash{$\big[\frac{M}{k},\frac{M}{k-1}\big)$} for $k = 1,2,\dots,M$, we have that $S(P)$ either is monotonically increasing, or is monotonically decreasing, or achieves a single maximum $S^*_k = S(P^*_k)$ at $P^*_k$.
  \item $S\big(\frac{M}{k}\big) > S(P)$ for $1 \le P < \frac{M}{k}$, for $k = 2,\dots,M$.
  \end{enumerate}
\end{thm}
\begin{proof}
  Part 1 is obvious by writing
  \begin{equation*}
    S^*_k = \frac{\rho M / k}{\rho_2 + 2 \sqrt{\rho_1 M / N k}} = \frac{\rho M}{\rho_2 k + 2 \sqrt{\rho_1 k M / N}}.
  \end{equation*}
  To prove part 2, we apply lemma~\ref{th:speedup-fcn} to $S(P)$ within interval \smash{$\big[\frac{M}{k},\frac{M}{k-1}\big)$} for $k = 1,2,\dots,M$. We obtain that $S(P)$ either is monotonically increasing, or is monotonically decreasing, or achieves a single maximum $S^*_k = S(P^*_k)$ at $P^*_k$.

  To prove part 3, we apply theorem~\ref{th:speedup-charact2} repeatedly for $k = 2,\dots,M$.
\end{proof}

\begin{rmk}
  As a particular case of theorem~\ref{th:speedup-charact} part 2 for $k=1$, we obtain that for $P \in [M,\infty)$
  \begin{equation}
    \label{e:speedup-largeP:max2}
    \tag{\ref{e:speedup-largeP:max}'}
    P^*_1 = \sqrt{\rho_1 M N} \qquad S^*_1 = S(P^*_1) = \frac{\rho M}{\rho_2 + 2 \sqrt{\rho_1 M / N}}
  \end{equation}
  and $S(P)$ is either monotonically decreasing with $P$ if $M \ge P^*_1$, or it increases from $P = M$ up to a single maximum at $P = P^*_1$ and then decreases monotonically. In both cases, if $t^{\W}_c > 0$ we have that $S(P) \rightarrow 0$ as $P \rightarrow \infty$, and $S(P) \approx \rho N M / P$ for large $P$.
\end{rmk}

\begin{thm}
  \label{th:speedup-charact2}
  Consider the function of eq.~\eqref{e:speedup3}, written more simply using $\rho'_1 = \rho_1 N > 0$, $\rho'_2 = \rho_2 N > 0$ and $\rho' = \rho N = \rho'_1 + \rho'_2 > 0$:
  \begin{equation}
    \label{e:speedup4}
    S(P) = \frac{\frac{1}{k} \rho' M P}{P^2 + \rho'_2 P + \frac{1}{k} \rho'_1 M} \quad \text{for } P \in \textstyle\big[\frac{M}{k},\frac{M}{k-1}\big), \qquad k=2,\dots,M.
  \end{equation}
  Then, for $k = 2,\dots,M$: $S\big(\frac{M}{k-1}\big) > S(P)$ $\forall P \in \big[\frac{M}{k},\frac{M}{k-1}\big)$.
\end{thm}
\begin{proof}
  From theorem~\ref{th:speedup-charact} part 2, we know that exactly one of the following three cases holds for $P \in \big[\frac{M}{k},\frac{M}{k-1}\big)$:
  \begin{enumerate}
  \item $S(P)$ is monotonically increasing. Then, it suffices to prove that $\lim_{P \rightarrow \frac{M}{k-1}}{ S(P) } < S\big(\frac{M}{k-1}\big)$. Indeed,
    \begin{equation*}
      \lim_{P \rightarrow \frac{M}{k-1}}{ S(P) } = \frac{\rho' M}{\rho' (k-1) + \rho'_2 + M k / (k-1)} < \frac{\rho' M}{\rho' (k-1) + M} = S\left(\frac{M}{k-1}\right).
    \end{equation*}
  \item $S(P)$ is monotonically decreasing. Then, it suffices to prove that $S\big(\frac{M}{k}\big) < S\big(\frac{M}{k-1}\big)$. Indeed,
    \begin{equation*}
      S\left(\frac{M}{k}\right) = \frac{\rho' M}{\rho' k + M} < \frac{\rho' M}{\rho' (k-1) + M} = S\left(\frac{M}{k-1}\right).
    \end{equation*}
  \item $S(P)$ achieves a single maximum $S^*_k = S(P^*_k)$ at $P^*_k$ in the interior of the interval. Then, it suffices to prove that $S^*_k < S\big(\frac{M}{k-1}\big)$. This last case is more complicated. In the sequel, we provide a proof that has not technical difficulties, although it is somewhat cumbersome.
  \end{enumerate}
  Let us then prove that $S^*_k < S\big(\frac{M}{k-1}\big)$. After a bit of algebra, from eqs.~\eqref{e:speedup4} and~\eqref{e:speedup-max2} we obtain the following:
  \begin{equation*}
    S^*_k < S\left(\frac{M}{k-1}\right) \Leftrightarrow M + \rho'_1 k - \rho' < 2 k P^*_k = 2 k \sqrt{\rho'_1 M / k}.
  \end{equation*}
  If $M + \rho'_1 k - \rho' \le 0$, then the condition holds and the proof is done. Otherwise, assume $M + \rho'_1 k - \rho' > 0$ and take squares in the previous equation:
  \begin{multline*}
    (M + \rho'_1 k - \rho')^2 < 4 \rho'_1 k M \Leftrightarrow M^2 + (\rho'_1 k - \rho')^2 - 2 (\rho'_1 k + \rho') M < 0 \\
    \Leftrightarrow \rho'_1 k + \rho' - 2 \sqrt{\rho'_1 \rho' k} < M < \rho'_1 k + \rho' + 2 \sqrt{\rho'_1 \rho' k}.
  \end{multline*}
  Hence, we need to prove that the following inequalities hold:
  \begin{equation}
    \label{e:ineqM}
    \rho'_1 k + \rho' - 2 \sqrt{\rho'_1 \rho' k} < M < \rho'_1 k + \rho' + 2 \sqrt{\rho'_1 \rho' k}
  \end{equation}
  under the following assumptions:
  \begin{itemize}
  \item $P^*_k = \sqrt{\rho'_1 M / k} \in \left(\frac{M}{k},\frac{M}{k-1}\right)$, since case 3 above means that $S(P)$ achieves a maximum $S^*_k$ at $P^*_k$ in the interior of the interval. Equivalently, $\rho'_1 k > M > \rho'_1 (k-1)^2/k = \rho'_1 \left(k-2+\frac{1}{k}\right)$.
  \item $M \ge k \ge 2$, since $k$ takes the values $2,3,\dots,M$.
  \item $M > \rho' - \rho'_1 k$, from above.
  \end{itemize}
  From $M < \rho'_1 k$ it follows that $M < \rho'_1 k + \rho' + 2 \sqrt{\rho'_1 \rho' k}$, and the RHS inequality in~\eqref{e:ineqM} is proven. Now let us prove the LHS inequality in~\eqref{e:ineqM}, which states $M > \rho'_1 k + \rho' - 2 \sqrt{\rho'_1 \rho' k}$. This can be derived from the assumption above that $M > \rho'_1 \left(k-2+\frac{1}{k}\right)$. Specifically, we will prove that $\rho'_1 \left(2 - \frac{1}{k}\right) < -\rho' + 2 \sqrt{\rho'_1 \rho' k}$, and this will complete the proof of the theorem.

  From assumptions $\rho'_1 k > M > \rho' - \rho'_1 k$ we get that $\rho' < 2 \rho'_1 k$, and since $\rho' \ge \rho'_1$, we have that $\rho' \in [\rho'_1,2 \rho'_1 k)$. Now, write $\rho' = a^2 \rho'_1$ with $a \in [1,\sqrt{2k})$. Then
  \begin{equation*}
    \rho'_1 \left(2 - \frac{1}{k}\right) < -\rho' + 2 \sqrt{\rho'_1 \rho' k} \Leftrightarrow a^2 + 2 - \frac{1}{k} < 2 a \sqrt{k}.
  \end{equation*}
  Since $k \ge 2 \Leftrightarrow \frac{3}{2} \ge 2 - \frac{1}{k}$, to prove $a^2 + 2 - \frac{1}{k} < 2 a \sqrt{k}$ it suffices to prove that $a^2 + \frac{3}{2} < 2 a \sqrt{k}$, or equivalently $a \in \left( \sqrt{k} - \sqrt{k - \frac{3}{2}},\sqrt{k} + \sqrt{k - \frac{3}{2}} \right)$. This interval indeed contains $[1,\sqrt{2k})$: a little algebra shows that
  \begin{equation*}
    \sqrt{k} - \sqrt{k - \frac{3}{2}} < 1 \Leftrightarrow k > \left(\frac{5}{4}\right)^2 \qquad \text{and} \qquad \sqrt{k} + \sqrt{k - \frac{3}{2}} > \sqrt{2k} \Leftrightarrow k > \frac{3}{4 (\sqrt{2}-1)}
  \end{equation*}
  both of which hold because $k \ge 2$.
\end{proof}

\begin{lemma}
  \label{th:speedup-fcn}
  Given constants $\alpha,\beta,\gamma,\delta > 0$, define the following real function for $P \ge 0$:
  \begin{equation}
    \label{e:speedup-fcn}
    \psi(P) = \frac{\delta P}{\alpha P^2 + \beta P + \gamma}
  \end{equation}
  and let $P^* = \sqrt{\gamma/\alpha}$ and $\psi^* = \psi(P^*) = \delta/(\beta + 2 \sqrt{\alpha \gamma})$. Then, in the interval $[a,b)$ with $1 \le a < b \le \infty$, $\psi
  \begin{cases}
    \text{is monotonically decreasing if } P^* \le a \\
    \text{achieves a single maximum at } P^* \text{ if } P^* \in (a,b) \\
    \text{is monotonically increasing if } P^* \ge b.
  \end{cases}$
\end{lemma}
\begin{proof}
  The derivatives of $\psi$ with respect to $P$ are:
  \begin{equation*}
    \psi'(P) = \frac{\delta (-\alpha P^2 + \gamma)}{(\alpha P^2 + \beta P + \gamma)^2} \qquad\psi''(P) = \frac{-2 \delta (-\alpha^2 P^3 + 3 \alpha \gamma P + \beta \gamma)}{(\alpha P^2 + \beta P + \gamma)^3}.
  \end{equation*}
  Hence $\psi'(P^*) = 0$ at $P^* = \sqrt{\gamma/\alpha}$ and $\psi''(P^*) < 0$, and the lemma follows.
\end{proof}

\subsection{Globally maximum speedup $S^* = \max_{P \ge 1}{S(P)}$}

The maximum speedup can be determined as follows. Both $S(P)$ in~\eqref{e:speedup-divisible} and $S^*_k$ in~\eqref{e:speedup-max} are monotonically increasing with $P$ (note $S^*_k$ is decreasing with $k$, and $k$ is decreasing with $P$). Hence, the global maximum of $S(P)$ occurs in the last interval $[M,\infty)$, either at the beginning ($P=M$, if $P^*_1 \le M$) or in its interior ($P=P^*_1$, if $P^*_1 > M$). This also follows from theorem~\ref{th:speedup-charact}. Specifically, the global maximum $S^*$ of $S(P)$ is:
\begin{itemize}
\item If $M \ge \rho_1 N$: $S^* = M / \left( 1 + \frac{M}{\rho N} \right) \le M$, achieved at $P = M$.
\item If $M < \rho_1 N$: $S^* = S^*_1 = \frac{\rho M}{\rho_2 + 2 \sqrt{\rho_1 M / N}} > M$, achieved at $P = P^*_1 = \sqrt{\rho_1 M N} > M$. \\
  Let us prove that $S^*_1 > M$. Assume $S^*_1 \le M$, then
  \begin{equation*}
    \frac{\rho M}{\rho_2 + 2 \sqrt{\rho_1 M / N}} \le M \Leftrightarrow \rho \le \rho_2 + 2 \sqrt{\rho_1 M / N} \Leftrightarrow \rho_1 \le 4 M / N \Leftrightarrow M \ge \rho_1 N / 4
  \end{equation*}
  which contradicts the condition that $M < \rho_1 N$, hence $S^*_1 > M$.
\end{itemize}
In practice, with large values of $N$, the more likely case is $S^* = S^*_1$ for $P = P^*_1 > M$.

\subsection{The ``large dataset'' case}

If we take $P \ll \rho_2 N$ (``large dataset'' case), the $P^2$ term in the speedup expression~\eqref{e:speedup2} becomes negligible, and the speedup becomes
\begin{equation*}
  S(P) = \frac{\rho \frac{1}{\ceil{M/P}} M P}{\frac{1}{N} P^2 + \rho_2 P + \rho_1 \frac{1}{\ceil{M/P}} M} \approx \rho / \left( \frac{\rho_1}{P} + k \frac{\rho_2}{M} \right)
\end{equation*}
where $k = \ceil{M/P} \in \{1,2,\dots,M\}$. Now, by taking $k=1$ ($M < P$) and $k = M/P$ ($M$ divisible by $P$) we obtain the following important cases:
\begin{equation}
  \tag{\ref{e:speedup-largeN}'}
  \text{if $M$ divisible by $P$:} \quad S(P) \approx P; \qquad \text{if $M > P$:} \quad S(P) \approx \rho / \left( \frac{\rho_1}{P} + \frac{\rho_2}{M} \right)
\end{equation}
so that the speedup is almost perfect up to $P = M$, and then it is approximately the weighted harmonic mean of $M$ and $P$ (hence, $S(P)$ is monotonically increasing and between $M$ and $P$). For $P \gg \rho_1$, we have $S(P) \approx \frac{\rho}{\rho_2} M > M$.

\clearpage

\section{Important MPI functions}
\label{s:MPI}

For reference, we briefly describe important MPI functions and their parameters \citep{Gropp_99a,Gropp_99b,MPI12a}.

\subsection{Environment Management Routines}

\begin{itemize}
\item \texttt{MPI\_Init(\&argc,\&argv)}: initialises the MPI execution environment. It must be called exactly once in every MPI program before calling any other MPI functions. For C programs, it may be used to pass the command-line arguments to all processes. Input: \texttt{argc}, pointer to the number of arguments; \texttt{argv}, pointer to the argument vector.
\item \texttt{MPI\_Comm\_size(comm,\&size)}: returns the total number of MPI processes in the specified communicator. If the communicator is \texttt{MPI\_COMM\_WORLD}, then it represents the number of MPI tasks available to your application. Input: \texttt{comm}, communicator (handle). Output: \texttt{size}, number of processes in the group of \texttt{comm} (integer).
\item \texttt{MPI\_Comm\_rank(comm,\&rank)}: returns the rank of the calling MPI process within the specified communicator. Initially, each process is assigned a unique integer rank between 0 and the number of tasks (1 within the communicator \texttt{MPI\_COMM\_WORLD}). This rank is often referred to as a task ID. If a process becomes associated with other communicators, it will have a unique rank within each of these as well. Input: \texttt{comm}, communicator (handle). Output: \texttt{rank}, rank of the calling process in the group of \texttt{comm} (integer).
\item \texttt{MPI\_Finalize()}: terminates the MPI execution environment. It should be the last MPI function called in any MPI program.
\end{itemize}

\subsection{Point to Point Communication Routines}

MPI point-to-point operations involve message passing between exactly two MPI tasks. One task performs a send operation and the other task performs a matching receive operation. There are different types of send and receive functions, used for different purposes, such as synchronous send; blocking send, blocking receive; non-blocking send, non-blocking receive; buffered send; combined send-receive. Their argument list generally takes one of the following formats:
\begin{itemize}
\item Blocking send: \texttt{MPI\_Send(buffer,count,type,dest,tag,comm)}.
\item Non-blocking send: \texttt{MPI\_Isend(buffer,count,type,dest,tag,comm,request)}.
\item Blocking receive: \texttt{MPI\_Recv(buffer,count,type,source,tag,comm,status)}.
\item Non-blocking receive: \texttt{MPI\_Irecv(buffer,count,type,source,tag,comm,request)}.
\end{itemize}
Here is a brief description of the parameters:
\begin{itemize}
\item \texttt{buffer}: program (application) address space that references the data that is to be sent or received. In most cases, this is simply the variable name that is to be sent or received. For C programs, \texttt{buffer} is passed by reference and must be prepended with an ampersand: \texttt{\&buffer}.
\item \texttt{count}: indicates the number of data elements of type \texttt{type} to be sent.
\item \texttt{type}: the data type that is sent or received. For reasons of portability, MPI predefines its elementary data types.
\item \texttt{dest}: for send routines, it indicates the process to which a message should be delivered. Specified as the rank of the receiving process.
\item \texttt{source}: for receive routines, it indicates the originating process of the message. Specified as the rank of the sending process. It may be set to the wild card \texttt{MPI\_ANY\_SOURCE} to receive a message from any task.
\item \texttt{tag}: arbitrary non-negative integer assigned by the programmer to uniquely identify a message. Send and receive operations should match message tags. For a receive operation, the wild card \texttt{MPI\_ANY\_TAG} can be used to receive any message regardless of its tag.
\item \texttt{comm}: it indicates the communication context, or set of processes for which the source or destination fields are valid. Unless the programmer is explicitly creating new communicators, the predefined communicator \texttt{MPI\_COMM\_WORLD} is usually used.
\item \texttt{status}: for receive routines, it indicates the source of the message and the tag of the message. In C, \texttt{status} is a pointer to a predefined structure \texttt{MPI\_Status}. Additionally, the actual number of bytes received is obtainable from \texttt{status} via the \texttt{MPI\_Get\_count} routine.
\item \texttt{request}: used by non-blocking send/receive. Since non-blocking operations may return before the requested system buffer space is obtained, the system issues a unique ``request number''. The programmer uses this system-assigned ``handle'' later (in a \texttt{Wait}-type routine) to determine completion of the non-blocking operation. In C, \texttt{request} is a pointer to a predefined structure \texttt{MPI\_Request}.
\end{itemize}
These are the \emph{blocking message passing routines}:
\begin{itemize}
\item \texttt{MPI\_Send(\&buf,count,datatype,dest,tag,comm)}: basic blocking send operation. It returns only after the application buffer in the sending task is free for reuse. 
\item \texttt{MPI\_Recv(\&buf,count,datatype,source,tag,comm,\&status)}: receive a message and block until the requested data is available in the application buffer in the receiving task.
\item \texttt{MPI\_Ssend(\&buf,count,datatype,dest,tag,comm)}: synchronous blocking send. It sends a message and blocks until the application buffer in the sending task is free for reuse and the destination process has started to receive the message.
\item \texttt{MPI\_Bsend(\&buf,count,datatype,dest,tag,comm)}: buffered blocking send. It allows the programmer to allocate the required amount of buffer space into which data can be copied until it is delivered. It alleviates problems associated with insufficient system buffer space. It returns after the data has been copied from the application buffer space to the allocated send buffer. It must be used with the \texttt{MPI\_Buffer\_attach} routine.
\item \texttt{MPI\_Buffer\_attach(\&buffer,size), MPI\_Buffer\_detach(\&buffer,size)}: used by the programmer to allocate or deallocate message buffer space to be used by \texttt{MPI\_Bsend}. The \texttt{size} argument is specified in actual data bytes (not a count of data elements). Only one buffer can be attached to a process at a time.
\item \texttt{MPI\_Wait(\&request,\&status)}: blocks until a specified non-blocking send or receive operation has completed.
\end{itemize}
These are the \emph{non-blocking message passing routines}:
\begin{itemize}
\item \texttt{MPI\_Isend(\&buf,count,datatype,dest,tag,comm,\&request)}: identifies an area in memory to serve as a send buffer. Processing continues immediately without waiting for the message to be copied out from the application buffer. A communication request handle is returned for handling the pending message status. The program should not modify the application buffer until subsequent calls to \texttt{MPI\_Wait} or \texttt{MPI\_Test} indicate that the non-blocking send has completed.
\item \texttt{MPI\_Irecv(\&buf,count,datatype,source,tag,comm,\&request)}: identifies an area in memory to serve as a receive buffer. Processing continues immediately without actually waiting for the message to be received and copied into the the application buffer. A communication request handle is returned for handling the pending message status. The program must use calls to \texttt{MPI\_Wait} or \texttt{MPI\_Test} to determine when the non-blocking receive operation completes and the requested message is available in the application buffer.
\item \texttt{MPI\_Test(\&request,\&flag,\&status)}: checks the status of a specified non-blocking send or receive operation. It returns in \texttt{flag} logical true (1) if the operation completed and logical false (0) otherwise.
\end{itemize}

\subsection{Collective Communication Routines}

\begin{itemize}
\item \texttt{MPI\_Bcast(\&buffer,count,datatype,root,comm)}: data movement operation. IT broadcasts (sends) a message from the process with rank \texttt{root} to all other processes in the group. 
\item \texttt{MPI\_Gather(\&sendbuf,sendcnt,sendtype,\&recvbuf,recvcount,recvtype,root,comm)}: data movement operation. It gathers distinct messages from each task in the group to a single destination task. Its reverse operation is \texttt{MPI\_Scatter}. 
\item \texttt{MPI\_Allgather(\&sendbuf,sendcount,sendtype,\&recvbuf,recvcount,recvtype,comm)}: data movement operation. It concatenates data to all tasks in a group. Each task in the group, in effect, performs a one-to-all broadcasting operation within the group. 
\item \texttt{MPI\_Reduce(\&sendbuf,\&recvbuf,count,datatype,op,root,comm)}: collective computation operation. It applies a reduction operation on all tasks in the group and places the result in one task. 
\item \texttt{MPI\_Allreduce(\&sendbuf,\&recvbuf,count,datatype,op,comm)}: collective computation and data movement operation. It applies a reduction operation and places the result in all tasks in the group. It is equivalent to \texttt{MPI\_Reduce} followed by \texttt{MPI\_Bcast}. 
\end{itemize}


\end{document}